\documentclass{article}
\usepackage[utf8]{inputenc} %
\usepackage[T1]{fontenc}    %
\usepackage{fullpage}

\usepackage[round]{natbib}

\usepackage[margin = 1in]{geometry}

\usepackage{url}            %
\usepackage{booktabs}       %
\usepackage{amsfonts}       %
\usepackage{wrapfig}
\usepackage[mathscr]{euscript}
\usepackage{float}
\usepackage[nodisplayskipstretch]{setspace}

\usepackage{times}
\usepackage{adjustbox}
\usepackage{dsfont}
\usepackage{amsmath, amssymb}
\usepackage{amsthm, thmtools, thm-restate}
\usepackage{bbm}
\usepackage{graphicx}
\usepackage{latexsym}
\usepackage{mathtools}
\usepackage{multirow}
\usepackage{paralist}
\usepackage{minitoc} %
\usepackage{xspace}
\usepackage{enumitem}
\usepackage{mdframed}

\theoremstyle{plain}
\newtheorem{theorem}{Theorem}
\newtheorem{proposition}{Proposition}
\newtheorem{lemma}{Lemma}
\newtheorem{corollary}{Corollary}

\theoremstyle{definition}
\newtheorem{definition}{Definition}
\newtheorem{assumption}[definition]{Assumption}

\theoremstyle{remark}

\newcommand{\myparagraph}[1]{\paragraph{#1.}\hspace{-0.8em}}  %

\usepackage[dvipsnames]{xcolor} %
\usepackage{tikz}
\usepackage{pgfplots}
\usetikzlibrary{shapes}
\usetikzlibrary{positioning}
\usetikzlibrary{plotmarks}
\usetikzlibrary{patterns}
\usetikzlibrary{intersections,shapes.arrows}
\usetikzlibrary{pgfplots.fillbetween}
\definecolor{darkpink}{rgb}{0.91, 0.33, 0.5}

\definecolor{puorange}{rgb}{0.80,0.40,0}
\definecolor{bluegray}{rgb}{0.04,0,0.7}
\definecolor{greengray}{rgb}{0.05,0.50,0.15}
\definecolor{darkbrown}{rgb}{0.40,0.2,0.05}
\definecolor{darkcyan}{rgb}{0,0.4,1}
\definecolor{black}{rgb}{0,0,0}
\definecolor{grey}{rgb}{0.93,0.93,0.93}
\definecolor{royalazure}{rgb}{0.0, 0.22, 0.66}
\definecolor{myyellow}{rgb}{0.70, 0.52, 0.01}
\definecolor{mypurple}{rgb}{0.38, 0.01, 0.70}

\newcommand{\red}[1]{{\color{purple}#1}}

\newcommand{\blue}[1]{{\color{royalazure}#1}}

\usepackage[colorlinks,citecolor=bluegray,linkcolor=darkbrown,urlcolor=blue,breaklinks]{hyperref}

\usepackage[capitalize,noabbrev]{cleveref}

\crefname{section}{Sec.}{Sections}

\crefname{theorem}{Thm.}{Thms.}
\crefname{lemma}{Lem.}{Lems.}
\crefname{corollary}{Cor.}{Cors.}
\crefname{proposition}{Prop.}{Props.}
\crefname{assumption}{Asm.}{Asms.}
\crefname{property}{Propt.}{Propts.}
\crefname{algorithm}{Alg.}{Algs.}
\crefname{appendix}{Appx.}{Appxs.}

\crefname{figure}{Fig.}{Figs.}
\crefname{table}{Tab.}{Tabs.}

\newcommand{\abs}[1]{\left| #1 \right|}

\newcommand{\ones}{\mathbf 1}
\newcommand{\zeros}{\operatorname{\mathbf 0}}

\newcommand{\Null}{\operatorname{null}}
\newcommand{\Range}{\operatorname{range}}
\newcommand{\Dom}{\operatorname{dom}}

\newcommand{\Tr}{\operatorname{\bf Tr}}
\newcommand{\Cl}{\operatorname{cl}}

\newcommand{\Span}{\operatorname{span}}

\newcommand{\ip}[1]{{\left\langle #1 \right\rangle}}
\newcommand{\ipsmall}[1]{{\langle #1 \rangle}}
\newcommand{\lone}{\mathbf{L}^1}
\newcommand{\ltwo}{\mathbf{L}^2}
\newcommand{\linf}{\mathbf{L}^\infty}

\newcommand{\Var}{\operatorname{\mathbb{V}ar}}

\newcommand{\Cov}{\operatorname{\mathbb{C}ov}}

\newcommand{\tv}{\operatorname{TV}}

\DeclareMathOperator*{\argmax}{arg\,max}
\DeclareMathOperator*{\argmin}{arg\,min}

\newcommand{\calB}{\mathcal{B}}

\newcommand{\calF}{\mathcal{F}}
\newcommand{\calG}{\mathcal{G}}
\newcommand{\calH}{\mathcal{H}}

\newcommand{\calS}{\mathcal{S}}

\newcommand{\ind}{\mathds{1}}

\newcommand{\ceil}[1]{\ensuremath{\left\lceil#1\right\rceil}}
\newcommand{\floor}[1]{\ensuremath{\left\lfloor#1\right\rfloor}}
\newcommand{\br}[1]{\ensuremath{\left\{#1\right\}}}
\renewcommand{\P}[2]{\ensuremath{{\mathbb P}_{#1}\left[#2\right]}}

\newcommand{\msc}[1]{\ensuremath{\mathscr{#1}}}
\newcommand{\prob}{\mathbb{P}}
\newcommand{\mc}[1]{\mathcal{#1}}
\newcommand{\sse}{\subseteq}

\newcommand{\R}{\mathbb{R}}
\newcommand{\Ex}{\mathbb{E}}
\newcommand{\E}[2]{\ensuremath{{\mathbb E}_{#1}\left[#2\right]}}
\newcommand{\sbr}[1]{\ensuremath{\left[#1\right]}}
\newcommand{\p}[1]{\ensuremath{\left(#1\right)}}
\renewcommand{\d}{\mathrm{d}}
\newcommand{\indep}{\perp\!\!\!\!\perp}

\newcommand{\risk}{\msc{R}}

\newcommand{\X}{\msc{X}}
\newcommand{\Y}{\msc{Y}}
\newcommand{\Z}{\msc{Z}}

\newcommand{\x}{\boldsymbol{x}}
\newcommand{\y}{\boldsymbol{y}}
\newcommand{\z}{\boldsymbol{z}}

\renewcommand{\u}{\boldsymbol{u}}
\renewcommand{\v}{\boldsymbol{v}}
\newcommand{\A}{\mathbf{A}}
\newcommand{\B}{\mathbf{B}}
\newcommand{\J}{\mathbf{J}}

\newcommand{\I}{\mathbf{I}}
\newcommand{\C}{\mathbf{C}}
\newcommand{\bmu}{\boldsymbol{\mu}}

\newcommand{\hC}{\widehat{\mathbf{C}}}
\newcommand{\hS}{\hat{\mathbf{\Sigma}}}
\newcommand{\bS}{\bar{\mathbf{\Sigma}}}

\newcommand{\balpha}{\boldsymbol{\alpha}}

\newcommand{\bbeta}{\boldsymbol{\beta}}

\newcommand{\heta}{\hat{\eta}}

\newcommand{\pow}[1]{^{\scriptscriptstyle(#1)}}

\newcommand{\hg}{\hat{g}}
\newcommand{\hrho}{\hat{\rho}}
\newcommand{\iidsim}{\overset{\text{\tiny i.i.d}}{\sim}}
\newcommand{\kmax}{k_{\max{}}}
\newcommand{\lmax}{l_{\max{}}}

\newcommand{\Lclip}{\mathcal{L}_{\mathrm{CLIP}}}

\newcommand{\hLBT}{\hat{\mathcal{L}}_{\mathrm{BT}}}
\newcommand{\hLSC}{\hat{\mathcal{L}}_{\mathrm{SC}}}
\newcommand{\hLclip}{\hat{\mathcal{L}}_{\mathrm{CLIP}}}
\newcommand{\hLvic}{\hat{\mathcal{L}}_{\mathrm{VICReg}}}

\newcommand{\HS}{\mathrm{HS}}

\newcommand{\df}{\operatorname{df}}
\newcommand{\norm}[1]{\lVert #1 \rVert}
\newcommand{\bignorm}[1]{\big{\lVert} #1 \big{\rVert}}

\newcommand{\V}{\mathbf{V}}
\newcommand{\hV}{\widehat{\mathbf{V}}}

\newcommand{\Rsans}{\mathsf{R}}
\newcommand{\M}{\mathbf{M}}
\newcommand{\hM}{\widehat{\mathbf{M}}}
\newcommand{\Ssans}{\mathsf{S}}

\newcommand{\Rhat}{\widehat{\mathsf{R}}}
\newcommand{\Fhat}{\widehat{F}}
\newcommand{\Lsans}{\mathsf{L}^2}

\newcommand{\defeq}{:=}

\newcommand{\dec}{\operatorname{dec}}

\newcommand{\Iop}{\mathbf{I}}
\newcommand{\Sop}{\mathbf{S}}
\newcommand{\Cop}{\mathbf{C}}
\newcommand{\Top}{\mathbf{T}}

\newcommand{\Np}{N_{\mathrm{p}}}
\newcommand{\Nu}{N_{\mathrm{u}}}
\newcommand{\hCopp}{\widehat{\mathbf{C}}_{\mathrm{p}}}
\newcommand{\hCopu}{\widehat{\mathbf{C}}_{\mathrm{u}}}

\newcommand{\op}{\operatorname{op}}
\newcommand{\polylog}{\operatorname{plog}}
\newcommand{\fstar}{r}
\newcommand{\etastar}{\eta_\star}
\newcommand{\fstarbound}{B_r}

\newif\ifcomments

\newcommand{\ZH}[1]{
\ifcomments
{\color{orange} [{\bf ZH}: #1]}
\fi
}

\doparttoc
\faketableofcontents

\title{A Generalization Theory for Zero-Shot Prediction}
\author{Ronak Mehta \qquad Zaid Harchaoui \vspace{0.3cm} \\
{
\small University of Washington, Seattle
}
}
\date{\today}

\begin{document}

\maketitle

\begin{abstract}
A modern paradigm for generalization in machine learning and AI consists of pre-training a task-agnostic foundation model, generally obtained using self-supervised and multimodal contrastive learning. The resulting representations can be used for prediction on a downstream task for which no labeled data is available. We present a theoretical framework to better understand this approach, called zero-shot prediction. We identify the target quantities that zero-shot prediction aims to learn, or learns in passing, and the key conditional independence relationships that enable its generalization ability.

\end{abstract}

\section{Introduction}\label{sec:intro}
In 2021, OpenAI shocked the world by improving the zero-shot classification accuracy on ImageNet from 11.5\%  to 76.2\% via the CLIP series of models \citep{radford2021learning}. This event redefined the goal of zero-shot prediction from producing models that generalized to \emph{unseen classes} to those that generalized to \emph{unseen tasks} entirely. Two fundamental drivers of CLIP's success were 1) the use of natural language as a medium for representing arbitrary classes (as in the previous state-of-the-art Visual N-grams \citep{li2017learning}), and 2) a massive, yet carefully designed pre-training set which significantly impacted downstream performance \citep{radford2021learning,fang2023data,xu2024demystifying}. Despite the remarkable success of these foundation model-based pipelines~\citep{bommasani2021opportunities}, there are unique components of zero-shot prediction that warrant investigation from a theoretical point of view.

To clarify these gaps, we contrast zero-shot prediction (ZSP) with the related setting of few-shot learning (FSL). Let $\x \in \X$ denote an input (often an image) that accompanies a discrete value $\y \in \Y$ (often a class label).
Common to both ZSP and FSL is a pre-training procedure in which a large unlabeled dataset $\x_1, \ldots, \x_N \in \X$ is used to produce an \emph{encoder} $\balpha: \X \rightarrow \R^d$. The \emph{embedding} $\balpha(\x)$ is thought to contain information that is relevant for predicting $\y$ from $\x$. Pre-training typically occurs through the process of self-supervised learning (SSL), using a \emph{pretext task} that can be solved with only instances of $\x$ (e.g.~filling in a blank image patch).
In FSL, the user may then access a labeled dataset $(\x^{\mathrm{lab}}_1, \y^{\mathrm{lab}}_1), \ldots, (\x^{\mathrm{lab}}_n, \y^{\mathrm{lab}}_n)$ from which a predictor can be trained inexpensively. This often takes the form of a linear classifier $\x \mapsto \mathbf{W} \balpha(\x) + \boldsymbol{b}$ for $\mathbf{W} \in \R^{\abs{\Y} \times d}$ and $\boldsymbol{b} \in \R^{\abs{\Y}}$. Where ZSP departs from FSL is the additional challenge of being given \emph{no directly labeled training data}. 

At first glance, ZSP seems impossible. Yet, the ingenuity of practitioners has yielded the following solution; if 1) each pre-training example $\x_i$ is paired with another ``view'' $\z_i \in \Z$ (e.g.~a caption in natural language) and 2) if each label $\y \in \Y$ can intelligently be embedded into $\Z$, then the relationship between each $\x_i$ and $\z_i$ could provide the means to perform prediction. Concretely, one learns a complementary encoder $\bbeta: \Z \rightarrow \R^d$ during pre-training and designs \emph{prompts} $\z_k^y$ for $\y \in \Y$ and $k = 1, \ldots, m$. Then,
\begin{align}
    \x \mapsto \argmax_{\y \in \Y} \frac{1}{m} \sum_{k=1}^m \langle \balpha(\x), \bbeta(\z_k^y)\rangle \label{eq:proxy}
\end{align}
is employed for prediction. An example of a prompt is the template text ``photo of a \underline{\hspace{0.5cm}}.'', where the blank can be filled by the textual representation of the class (e.g.~``cat'' or ``dog''). The ZSP pipeline, from pre-training to prompt selection, is clearly a wild departure from what is explained by statistical learning theory. Moreover, while some components of these systems have been studied in the context of FSL (such as the reasons why various pre-training objectives result in encoders that provably accelerate learning), unique aspects of ZSP, such as the role of prompting and the cost of ``translating'' modalities, have not yet received theoretical treatment. Herein lies our question.

\begin{mdframed}[userdefinedwidth=0.95\columnwidth, align=center, innerleftmargin=6pt, innerrightmargin=6pt]
{\it Through what decomposition of downstream task performance can we compare zero-shot prediction to the direct supervised learner, with a transparent dependence on the 1) pre-training distribution, 2) evaluation distribution, and 3) prompting strategy?}
\end{mdframed}

\myparagraph{Contributions}
In \Cref{sec:framework}, we present a learning theoretic framework for the pre-training/evaluation/prompting data and propose two expressions for the population counterpart of~\eqref{eq:proxy}. 
These expressions, while equivalent at the population level, reflect two classes of learning methods which we call the ``conditional mean'' and the ``information density'' approaches.
In \Cref{sec:theory}, we prove a generic decomposition of the prediction error on the downstream task, which furnishes three components: \emph{prompt bias} measures the compatibility of the prompt strategy with the pre-training and evaluation distributions, \emph{residual dependence} measures the information-theoretic cost of using one modality to make predictions on another, and \emph{estimation error} quantifies the effect of the finite number of pre-training examples and prompts. The estimation error decomposes further depending on whether the conditional mean or information density approaches are taken. To provide insight and demonstrate the usefulness of the decomposition, we analyze the performance of nonparametric regression methods for each approach by way of finite-sample bounds in high probability. Our framework arms practitioners with a means to imbue existing SSL-to-ZSP pipelines with theoretical guarantees, depending on the approach with which they best align.
In \Cref{sec:experiments}, we illustrate our theoretical claims by empirically evaluating prompt bias and residual dependence on zero-shot prediction tasks with simulated and image data.

\myparagraph{Related Work}
One can argue that precursors to both FSL and ZSP in 
machine learning can be found in the literature of meta-learning, or ``learning to learn'' \citep{thrun1998learning, andrychowicz2016learning, finn2017model}. There, the downstream evaluation tasks are given to the user upfront, so that pre-training an encoder and training a predictor for all of the evaluation tasks can be performed in one step.
On the other hand, FSL and ZSP both involve fully task-agnostic pre-training phases. 
Seminal work in computer vision on matching words and pictures is also worth mentioning \citep{barnard2003matching,forsyth2009words}. 

Two complementary bodies of work studied phenomena common to FSL and ZSP. The first considers which properties of learned encoders can provably improve downstream performance \citep{wang2020understanding, haochen2021provable, atzmon2020acausal, wang2024desiderata, du2024lowrank}. The other is dedicated to explaining how otherwise mysterious SSL objectives achieve these properties \citep{wen2021toward, li2021selfsupervised, pokle2022constrasting, kiani2022jointembeddings, johnson2023contrastive, schwarz-ziv2023aninformation}. In particular, \citet{balestriero2022constrastive} and \citet{tan2024contrastive} relate various SSL objectives to spectral clustering.
One FSL-specific line of work studies when linear mappings of pre-trained encoders can achieve optimal downstream performance \citep{saunshi2019atheoretical, haochen2021provable, tosh2021contrastive, lee2021predicting}.

While informative representations are essential, the core of ZSP is the remarkable ability of models to make predictions without \emph{any} task-specific data, a challenge even for the perfect encoder. For context, we avoid the historical term ``zero-shot learning'' \citep{larochelle2008zerodata, akata2015label}, which refers to a setting in which pre-training data is not only \emph{labeled}, but contains metadata-based features associated with each class. In general, this only handles unseen classes, and only if the same features are observed at inference time. 
To our knowledge, the only work studying ZSP based on self-supervised pre-training is \citet{chen2024understanding}. In particular, \citet[Theorem 4.2 and Corollary 5.1]{chen2024understanding} provides bounds on the top-$k$ accuracy of ZSP for CLIP-based encoders. However, the bound \emph{increases} with the batch size, may not decay to zero even if the pre-training loss is fully optimized and upstream and downstream data distributions are the same, and does not seem to explicitly depend on the prompt quality. 
The independent and concurrent work of \citet{oko2025statistical} develops a statistical analysis based on sufficiency notions, with the aim of capturing the predictive performance in the downstream task. 
Their work is complementary to ours, in that they determine the distributional parameter learned by the CLIP objective, but also assume that the prompting strategy and downstream data distribution are ``idealized'', in that the prompt bias and residual dependence quantities alluded to in the contributions are zero.

On the applied side, we are inspired by the number of works that use diverse, class-specific prompts generated using large language models (LLMs) for enhancing ZSP performance \citep{pratt2023what, yang2023language, maniparambil2023enhancing} and interpretability \citep{menon2023visual, esfandiarpoor2024ifclip}. While these empirical methods, such as the customized prompts via language models method~(CuPL, \citet{pratt2023what}), are often designed using intuition from human understanding of natural language, we aim to offer a theoretical explanation for their success from a statistical learning theory and probabilistic graphical modeling perspective. Despite this particular application of LLMs, we also acknowledge that ``prompting'' in ZSP has a different meaning than in the growing field of prompt engineering, in which inputs are designed for large language models \citep{pryzant2023automatic, wang2024promptagent, guo2024connecting, sclar2024quantifying}.

\section{Theoretical Framework}\label{sec:framework}
We introduce the mathematical objects that connect the empirically-motivated predictor~\eqref{eq:proxy} to its theoretical counterpart analyzed in \Cref{sec:theory}. For the reader's convenience, a global notation table is provided in \Cref{sec:a:notation}.

\begin{figure*}[t!]
    \centering
    \includegraphics[width=\linewidth]{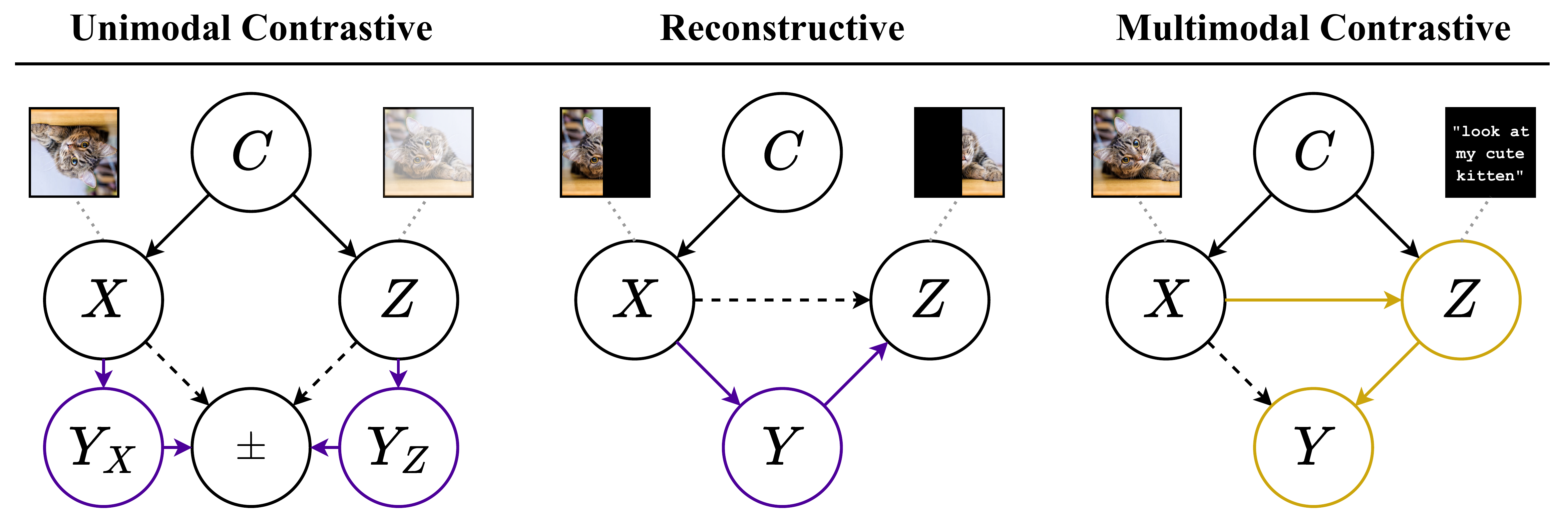}
    \caption{{\bf Graphical Models of Prediction Paths.} Each directed graphical model corresponds to the data types and dependence structures for various SSL pre-training approaches. The variable $C$ represents an unobserved context that determines the observed data-generating distribution. Dotted lines indicate the possibility of presence or absence of the arrow. 
    Methods {\color{myyellow}{\bf compatible with ZSP}} may learn the relationship between $X$ and $Z$ directly, whereas the relationship between $Z$ and $Y$ is learned via prompting.
    Methods that are {\color{mypurple}{\bf compatible with FSL}} learn the label $Y$ as a latent variable in the process of solving the pretext task. 
    }
    \label{fig:dag}
\end{figure*}

\myparagraph{Prediction Setups}
Consider random variables $X$, $Y$, and $Z$ observed in $\X$, $\Y$, and $\Z$, respectively. We interpret $\X$ as the space of images, $\Y$ as the (not necessarily discrete) space of labels, and $\Z$ as the space of text captions. 

Consider a probability measure $P_{X, Y}$ on $\X \times \Y$, called the \emph{evaluation distribution}. We specify a collection of downstream tasks, with which we may evaluate predictors on data drawn from $P_{X, Y}$ (e.g.~CIFAR-10). Consider a function $\fstar: \Y \rightarrow \R$, and the least squares prediction problem
\begin{align}
    \min_{\eta: \X \rightarrow \R} \E{P_{X, Y}}{(\eta(X) - \fstar(Y))^2}\label{eq:mmse}
\end{align}
The function $\fstar$ serves only to handle multiple task formats such as regression ($\fstar(\y) = \y$) or binary classification ($\fstar(\y) = \ind\br{\y=1}$) in a unified manner. We discuss formulations of multi-class classification and structured prediction in \Cref{sec:theory}. The optimizer of~\eqref{eq:mmse} over $\eta \in \ltwo(P_X)$\footnote{In the appendix, we carefully construct $\ltwo$-spaces as sets of equivalence classes of functions (see \Cref{sec:a:background:l2}) for explicitness and rigor. We do not belabor this distinction in the main text.}, or all measurable, square-integrable functions on $\X$, is
\begin{align}
    \etastar(\x) := \E{P_{Y, X}}{\fstar(Y)|X}(\x).\label{eq:est:bayes}
\end{align}
We will call this the \emph{direct predictor} throughout this paper, which will contrast our viewpoint of ZSP as an indirect, multi-stage prediction procedure. Indeed, the prompting step in~\eqref{eq:proxy} resembles an empirical average of draws from a probability distribution on $\Z$ based on the class label $Y = \y$ (especially when considering the LLM-based generation methods mentioned in \Cref{sec:intro}), whereas the encoders capture a dependence relation between $X$ and $Z$. Accordingly, we introduce a probability measure  $Q_{X, Z}$ on $\X \times \Z$, called the \emph{pre-training distribution}, and the \emph{prompt distribution} $\rho_{Y, Z}$ on $\Y \times \Z$ which represents the user-defined strategy for generating prompts. As a theoretical model for ZSP, we propose the function
\begin{align}
    \eta_\rho(\x) = \E{Q_{X, Z}}{g_\rho(Z)|X}(\x),\label{eq:est:ts}
\end{align}
called the \emph{indirect predictor}, where
\begin{align}
    g_\rho(\z) = \E{\rho_{Y, Z}}{\fstar(Y)|Z}(\z).\label{eq:g_rho}
\end{align}
Notice that $\etastar$ relies only on $P_{X, Y}$ while $\eta_\rho$ is a two-stage predictor relying only on the pair $(Q_{X, Z}, \rho_{Y, Z})$. The pre-training, evaluation, and prompt distributions represent pairwise dependencies between the random variables $X$, $Y$, and $Z$, as well as the observable data of the problem. Intuitively, our analysis of the performance gap between the direct and indirect predictors will quantify the ``compatibility'' of these three fundamental distributions as a possible joint distribution on $\X \times \Y \times \Z$. 

\myparagraph{Prediction Paths of FSL and ZSP} 
For context, we contrast our setup with previous theoretical analyses of FSL, aiming to 1) highlight the fundamental differences between SSL-for-FSL and SSL-for-ZSP, 2) describe assumptions we make (and do not make) to best align with applications.
First, we consider two common SSL tasks that precede FSL. In unimodal contrastive learning, $X$ and $Z$ are augmented/corrupted images, and the pretext task is to identify examples derived from the same (``$+$'') or different (``$-$'') underlying image \citep{Chen2020ASimple}. In reconstructive SSL, the encoder is pre-trained to predict a hidden portion of the raw/embedded image \citep{assran2023self}. 
Foundational works such as \citet{saunshi2019atheoretical} and \citet{wang2020understanding} explain the success of these SSL-for-FSL pipelines by the following mechanism: the labels $\Y$ used in the downstream task form a latent variable mixture model for the pre-training set, i.e.~$Q_{X,Z} = \sum_{\y \in \Y} Q_{X, Z|Y= \y} \cdot Q_Y(\y)$. Thus, generalization guarantees hinge upon the fact that learning parameters of the pre-training distribution must inherently capture its latent variables (the downstream labels). This theory is visualized in \Cref{fig:dag} (left \& center); observe that if the dotted arrows were absent, the only path to solve the pretext task is \emph{through} the label.
This FSL ``prediction path'' motivates another prevalent assumption of exact/approximate conditional independence of $X$ and $Z$ given $Y$ (e.g., as in \citet{lee2021predicting}). We avoid this assumption, which is unrealistic in the multimodal context as the dependence between an image and its caption is unlikely to be fully explained by a coarse label such as ``cat''. Moreover, this \emph{latent label model} assumes equality of the marginals $P_X = Q_X$ on $\X$. As a concrete example, this amounts to assuming that the marginal distribution of images on the Internet ($Q_X$) is equal to that of CIFAR-10 images ($P_X$). We explicitly track this mismatch in our generalization bounds.

For ZSP, the prevailing SSL pretext task is multimodal contrastive learning (\Cref{fig:dag}, right), wherein the foundation model learns a similarity function $(\x, \z) \mapsto \ip{\balpha(\x), \bbeta(\z)}$. To discuss a joint distribution $P \equiv P_{X, Y, Z}$, we adopt a \emph{latent caption model} that associates $X \sim P_X$ with an unobserved $\Z$-valued latent variable $Z$ (i.e.~an unobserved caption). Because pre-training connects $X$ to $Z$ and prompting then connects $Y$ to $Z$, the ideal dependence structure for ZSP is fundamentally different from FSL; if $X$ and $Y$ are conditionally independent given $Z$, the direct and indirect predictors are in fact equivalent. Indeed, the tower property of conditional expectation gives the identity
\begin{align*}
    \etastar(\x) &= \E{P}{\fstar(Y)|X}(\x)\\
    &= \E{P}{\E{P}{\fstar(Y)|Z, X}|X}(\x)\\    
    &= \E{P}{\E{P}{\fstar(Y)|Z}|X}(\x). &(X \indep Y |Z)
\end{align*}
The final expression is not equal to~\eqref{eq:est:ts} because of the difference between $(Q_{X, Z}, \rho_{Y, Z})$ and $(P_{X, Z}, P_{Y, Z})$. Additionally, $X$ and $Y$ are not necessarily conditionally independent given $Z$. 
These discrepancies are precisely exposed in our analysis via a measure of distribution mismatch and a measure of the conditional dependence of $X$ and $Y$ given $Z$. The latter formalizes the information-theoretic cost of using natural language as a proxy for image classification.

\myparagraph{Representations of the Indirect Predictor}
We establish several central identities involving the indirect predictor~\eqref{eq:est:ts}. These expressions will strengthen the justification for $\eta_\rho$ as the target function of ZSP and naturally lead to two classes of learning methods that we analyze in \Cref{sec:theory}.
As a preview, consider the example of balanced binary classification ($r(\y) = \ind\br{\y = 1}$) and the classifier that returns $1$ when $\eta_\rho(\x) \geq 1/2$ and $0$ otherwise. We will show that there exist encoders $\balpha: \X \rightarrow \R^d$ and $\bbeta: \Z \rightarrow \R^d$, and a sequence of scalars $\sigma_1 \geq \ldots \geq \sigma_d \geq 0$ such that if $\rho_Z \approx Q_Z$ and $d$ is sufficiently large, then this classifier is equivalent to
\begin{align}
    \x \mapsto \textstyle\argmax_{\y \in \Y} \ip{\balpha(\x), \E{\rho_{Y, Z}}{\bbeta(Z)|Y = \y}}_{\sigma},\label{eq:proxy_classifier}
\end{align}
where $\ip{\u, \v}_{\sigma} := \sum_{i=1}^d \sigma_i u_i v_i$. This expression mirrors~\eqref{eq:proxy} down to a rescaling of the inner product. We now present the expressions that are used to derive~\eqref{eq:proxy_classifier}.

For the first, let $Q_X$ and $Q_Z$ be the marginals of $Q_{X, Z}$ on $\X$ and $\Z$, respectively. We introduce the fundamental \emph{conditional mean operator} $\M_{Z|X}: \ltwo(Q_Z) \rightarrow \ltwo(Q_X)$, which assigns to any $g \in \ltwo(Q_Z)$ the function $\x \mapsto \E{Q_{X, Z}}{g(Z)|X}(\x)$. Then, it holds by definition that
\begin{align}
    \eta_\rho(\x) = [\M_{Z|X} g_\rho](\x). \label{eq:framework:cond_mean}
\end{align}
For the second, 
consider the case in which $Q_{X, Z} \ll Q_XQ_Z$\footnote{A distribution $\mu$ on $\msc{U}$ is \emph{absolute continuous} with respect to another distribution $\nu$ (denoted $\mu \ll \nu$) if $\nu(A) = 0 \implies \mu(A) = 0$ for every measurable set $A \sse \msc{U}$. If so, there exists a \emph{Radon-Nikodym derivative} $\frac{\d\mu}{\d\nu}: \msc{U} \rightarrow \R_{\geq 0}$ such that for every measurable set $A \sse \msc{U}$, it holds that $\mu(A) = \int_A \frac{\d\mu}{\d\nu}(\u) \d \nu(\u)$.}, where $Q_X Q_Z$ denotes the probability distribution of the pair $(X, Z)$ drawn independently as $X \sim Q_X$ and $Z \sim Q_Z$. 
Then, we define the Radon-Nikodym derivative $\Rsans := \frac{\d Q_{X, Z}}{\d Q_X Q_Z}: \X \times \Z \rightarrow \R_{\geq 0}$. 
The function $\Rsans$, called the \emph{information density}\footnote{This term actually refers to $(\x, \z) \mapsto \log \Rsans(\x, \z)$, but for simplicity, we use it for $\Rsans$---see \citet[Eq. (11)]{dytso2023meta}.} has a long history in statistics and information theory. Using $\Rsans$ (\Cref{lem:info_density2}, \Cref{sec:a:background:buja}), the indirect predictor writes as
\begin{align}
    \eta_\rho(\x) &= \E{Q_Z}{g_\rho(Z)\Rsans(\x, Z)} \notag\\
    &= \E{\rho_{Y,Z}}{\fstar(Y)\Rsans(\x, Z)} + \operatorname{err}(Q_Z, \rho_Z),\label{eq:framework:rnd}
\end{align}
where $\operatorname{err}(Q_Z, \rho_Z)$ term measures the discrepancy between the marginal distributions of the captions generated during pre-training and prompting, respectively. The expressions~\eqref{eq:framework:cond_mean} and~\eqref{eq:framework:rnd}, while equal at the population level, motivate two categories of approaches for learning/estimation that have different statistical properties. The ``conditional mean'' approach uses pre-training data to learn the operator $\M_{Z|X}$ and prompts to approximate the function $g_\rho$. On the other hand, the ``information density'' approach learns the function $\Rsans$ during pre-training, and approximates the expectation over $\rho_{Y, Z}$ using prompts. 
The information density approach is particularly reflective of the prompting aspect of~\eqref{eq:proxy}, as one may perceive $\z_k^{\y}$ for $k = 1, \ldots, m$ and $\y \in \Y$ as $M = m\abs{\Y}$ as samples from $\rho_{Y, Z}$ with $\rho_Y$ chosen to be uniform on $\Y$. These are then used to replace the expectation in~\eqref{eq:framework:rnd}.

Finally, we tie back to~\eqref{eq:proxy_classifier} and describe the formal connection between $\M_{Z|X}$ and $\Rsans$. 
In \Cref{prop:lancaster} (\Cref{sec:a:background:buja}), we prove the decomposition of the form
\begin{align}
    \Rsans(\x, \z) = \ip{\balpha(\x), \bbeta(\z)}_{\sigma} + \varepsilon_d,\label{eq:lancaster}
\end{align}
where $\sigma_d, \varepsilon_d \rightarrow 0$ as $d \rightarrow \infty$. Then,~\eqref{eq:proxy_classifier} follows under the given conditions by plugging~\eqref{eq:lancaster} into~\eqref{eq:framework:rnd}.

The encoders $\balpha = (\alpha_1, \ldots, \alpha_d)$ and $\bbeta = (\beta_1, \ldots, \beta_d)$, and constants $(\sigma_i)_{i=1}^d$ are none other than the components of the truncated singular value decomposition (SVD) of $\M_{Z|X}$ (\Cref{prop:svd}, \Cref{sec:a:background:buja}). The SVD of $\M_{Z|X}$ and the information density $\Rsans$ characterize the full dependence structure of $Q_{X, Z}$; because $\Rsans$ is identically $1$ when $Q_{X, Z} = Q_X Q_Z$, we may define the (squared) \emph{mean square contingency} dependence measure
\begin{align}
    I(X; Z) &= \E{Q_{X} Q_{Z}}{(\Rsans(X, Z) - 1)^2}\label{eq:msc}\\
    &= \norm{\M_{Z|X}}_{\HS}^2 - 1 \notag\\
    &= \textstyle\sum_{i=2}^\infty \sigma_i^2,\label{eq:framework:spectrum}
\end{align}
where $\norm{\cdot}_{\HS}$ denotes the Hilbert-Schmidt norm (see \Cref{def:trace_hs}, \Cref{sec:a:background:compact}), and the identities are proven in \Cref{prop:lancaster}. The right-hand side of~\eqref{eq:msc} can also be interpreted as the $\chi^2$-divergence $D_{\chi^2}(Q_{X, Z} \Vert Q_{X} Q_{Z}) := \Ex_{Q_X Q_Z}[(\frac{\d Q_{X, Z}}{\d Q_{X} Q_{Z}}(X, Z) - 1)^2]$ between the joint distribution and the product of the marginals (see \Cref{def:msc}).

\section{Generalization Guarantees for ZSP}\label{sec:theory}

In this section, we prove generalization guarantees for ZSP methods by comparing $\etastar$ to $\eta_\rho$ and $\eta_\rho$ to an estimator $\heta_\rho$, based on an $N$-sized pre-training set and $M$-sized prompt set (recall that $M = m\abs{\Y}$ in~\eqref{eq:proxy}). While there are some subtleties in the sampling models between various methods, one can consider  $(X_1, Z_1), \ldots, (X_N, Z_N) \iidsim Q_{X, Z}$ and $(Y_1, Z_1'), \ldots, (Y_M, Z_M') \iidsim \rho_{Y, Z}$ for intuition purposes (see \Cref{sec:a:complexity:sampling} for a detailed description).
We consider specific instances of both the conditional mean and information density approaches, based on learning theory in reproducing kernel Hilbert space (RKHS); our arguments do not intend to interpret foundation modeling as a kernel method, but to use the detailed analysis of the statistical errors in kernel methods to gain insight. In particular, we aim to expose two key dependences for the random triple $(X, Y, Z)$: the dependence between $X$ and $Z$ (which governs pre-training) and the conditional dependence between $X$ and $Y$ given $Z$ (which governs downstream prediction). Similar statistical guarantees for other function classes (reviewed in \Cref{sec:a:objectives}) can be plugged into our framework, which intends to capture the end-to-end performance from pre-training to downstream prediction. 

For $h \in \ltwo(P_X)$, we define the norm $\norm{h}_{\ltwo(P_X)}^2 := \int_{\X} h^2(\x) \d P_X(\x)$, using analogous notation for other probability distributions. We will assume throughout the paper that $\fstar$ is bounded by $\fstarbound$ with probability one under $P_Y$ and $\rho_Y$, so that $\eta_\rho, \etastar \in \ltwo(P_X)$. Given a square-integrable $\heta_\rho$, we first control the mean squared error (MSE) via $\norm{\etastar - \heta_\rho}_{\ltwo(P_X)}^2 \leq$
\begin{align}
    2\underbrace{\norm{\etastar - \eta_\rho}_{\ltwo(P_X)}^2 }_{\text{information-theoretic error}} + 2\underbrace{\norm{\eta_\rho - \heta_\rho}_{\ltwo(P_X)}^2}_{\text{estimation error}}.\label{eq:theory:decomp1}
\end{align}
The information-theoretic error captures the prompt bias and residual dependence that differentiates indirect and direct prediction, whereas the estimation error is a familiar term in statistical analysis. We discuss in \Cref{sec:a:complexity:classification} how to convert the MSE bounds to risk bounds for classification.

\myparagraph{Prompt Bias and Residual Dependence}
Here, we control the information-theoretic error term in~\eqref{eq:theory:decomp1}. We state our assumptions regarding conditional probability informally and defer the formal descriptions using the language of regular conditional distributions to \Cref{sec:a:dependence}. We work within the latent caption model from \Cref{sec:framework}, for which we consider a joint distribution $P_{X, Y, Z}$ on $\X \times \Y \times \Z$ which equals the evaluation distribution $P_{X, Y}$ when marginalized over $\Z$.
Similar to the information density $\Rsans$ from \Cref{sec:framework}, we introduce the conditional information density
\begin{align}
    \Ssans_{\z} := \frac{\d P_{X, Y|\z}}{\d (P_{X|\z} P_{Y|\z})}: \X \times \Y \rightarrow \R_{\geq 0},\label{eq:theory:rnd0}
\end{align}
where $P_{X, Y|\z}$ denotes the conditional distribution of $(X, Y)$ given $Z = \z$, and $P_{X|\z} P_{Y|\z}$ is defined analogously. 
This naturally motivates the conditional dependence measure given by
\begin{align}
    I(X; Y|\z) = \E{P_{X|\z} P_{Y|\z}}{(\Ssans_{\z}(X, Y) - 1)^2},\label{eq:cond_msc}
\end{align}
called the \emph{conditional mean square contingency}.
Finally, consider the following regularity assumption on the joint distribution $P_{X, Y, Z}$, also discussed in \Cref{sec:a:dependence}.
\begin{assumption}\label{asm:rcd1}
    $P_{X, Y, Z}$ on $\X \times \Y \times \Z$ satisfies the following: \emph{1) Agreement of caption distribution:} $P_X$-almost all $\x \in \X$,~$P_{Z|\x}$ exists and $P_{Z|\x} = Q_{Z|\x}$. \emph{2) Regularity of conditional distributions:} For $P_Z$-almost all $\z \in \Z$,~$P_{X, Y|\z}$ exists, $P_{X, Y|\z} \ll P_{X|\z} P_{Y|\z}$, and the conditional information density~\eqref{eq:theory:rnd0} satisfies $\E{P_{X, Y|\z}}{\Ssans_{\z}(X, Y)} < +\infty$ and $\E{P_{X, Y, Z}}{\Ssans_{Z}(X, Y)} < +\infty$.
\end{assumption}
To measure the bias of the prompt distribution $\rho_{Y, Z}$, we denote the analog of~\eqref{eq:g_rho} under $P_{Y, Z}$ as
\begin{align*}
    g_{P_{Y, Z}}(\z) = \E{P_{Y, Z}}{\fstar(Y)|Z}(\z).
\end{align*}
We may now state the main result, proved in \Cref{sec:a:dependence}.
\begin{theorem}\label{thm:res_dep1}
    Under \Cref{asm:rcd1},
    \begin{align}
        \norm{\eta_\rho - \etastar}_{\ltwo(P_X)}^2 \lesssim \underbrace{\E{P_Z}{I(X; Y|Z)}}_{\text{residual dependence}} +\underbrace{\norm{g_\rho - g_{P_{Y, Z}}}_{\ltwo(P_Z)}^2}_{\text{prompt bias}}.\label{eq:res_dep1}
    \end{align}
\end{theorem}
To give context to \Cref{thm:res_dep1}, conditional independence relations have previously been used to describe the performance of multimodal contrastive SSL for FSL. We are particularly inspired by the \emph{multi-view redundancy} theory of \citet{tosh2021contrastive}, which states informally that the population FSL predictor can approach the performance of the idealized direct predictor that is given \emph{both} $X$ and $Z$ at test time, if $X \indep Y | Z$ and $Z \indep Y | X$ approximately hold. However, the theory of graphical models \citep[Proposition 3.1]{lauritzen1996graphical} asserts that both conditional independence relations hold only if $(X,Z) \indep Y$, that is, neither view has information predictive of the label. This can be seen intuitively from \Cref{fig:dag} by breaking the arrows $X \rightarrow Y$ and $Z \rightarrow Y$. Notice that we compare only to the direct predictor~\eqref{eq:est:bayes} given $X$ (which is reflective of practice), so that we need only that $X \indep Y | Z$ (i.e.~$X$ depends on $Y$ through $Z$) to close the gap. The prompt bias term~\eqref{eq:res_dep1} captures the possible incompatibility of the prompt distribution $\rho_{Y, Z}$ with $(P_{X, Y}, Q_{X, Z})$---we call prompt strategies unbiased (see \Cref{sec:a:complexity:sampling}) when this term is zero.

\myparagraph{Sample Complexity and Distribution Mismatch}
The first step in our estimation error analysis is to pass the $\ltwo(P_X)$-norm term $\norm{\eta_\rho - \heta_\rho}_{\ltwo(P_X)}^2$ from~\eqref{eq:theory:decomp1} to the $\ltwo(Q_X)$-norm counterpart $\norm{\eta_\rho - \heta_\rho}_{\ltwo(Q_X)}^2$. We then establish high-probability bounds on the $\ltwo(Q_X)$-norm term, with respect to the random sampling of the pre-training and prompting data. Because this initial step follows from a standard distribution shift argument (based on either a bounded likelihood ratio assumption or an additive error in total variation distance), we defer it to \Cref{sec:a:complexity} (see \Cref{lem:distribution_shift}). Conceptually, the two examples below are derived from estimating the component of either~\eqref{eq:framework:cond_mean} or~\eqref{eq:framework:rnd} that involves $Q_{X, Z}$ using the pre-training set and the one that involves $\rho_{Y, Z}$ using the prompt strategy. In both cases, we discuss the convergence rates of state-of-the-art RKHS-based methods. As we review \Cref{sec:a:background:rkhs}, these rates are typically expressed in terms of two quantities: \emph{source condition} constants, which measure the smoothness of the target function being learned, and \emph{eigendecay exponents} of covariance operators, which measure the effective dimension of the data. It will serve our purposes to interpret the rates in terms of the dependence between $X$ and $Z$ under $Q_{X, Z}$, under the following assumption.
\begin{assumption}
    The pre-training distribution satisfies $Q_{X, Z} \ll Q_{X} Q_Z$, and the information density $\Rsans$ is contained in $\ltwo(Q_XQ_Z)$ (i.e.~$I(X;Z)$  is well-defined).
\end{assumption}
Due to the technical overhead of each method (especially regarding mis-specified function classes), we provide high-level descriptions below and defer detailed descriptions of the specific estimation procedures and formal assumptions to \Cref{sec:a:complexity:conditional_mean} (conditional mean) and \Cref{sec:a:complexity:rn_derivative} (information density).
We denote by $\delta \in (0, 1]$ a failure probability, and $\polylog(\cdot)$ a term that is poly-logarithmic in its input. 

\myparagraph{Example 1: Nonparametric Regression}
This approach, based on~\eqref{eq:framework:cond_mean}, uses the pre-training set to produce an estimate $\hM_{Z|X}$ of the conditional mean operator and the prompts to produce an approximation $\hg_\rho: \Z \rightarrow \R$ of $g_\rho$. For the former, we use as an example the spectral regularization learning method of \citet{meunier2024optimal}, which produces a conditional mean embedding function $\Fhat:\X \rightarrow \calG$, for an RKHS $\calG$ of real-valued functions of $\Z$.  For any $g \in \calG$, we then define $[\hM_{Z|X} g](\x) = \ipsmall{g, \Fhat(\x)}_\calG$. Note that $\Fhat$ predicts a target that is itself a function--such methods are therefore referred to as ``vector-valued'' regression. By the Reisz representation theorem, a similar function $F_\star$ can be constructed such that $[\M_{Z|X} g](\x) = \ipsmall{g, F_\star(\x)}_\calG$.
For $\hg_\rho$, we consider standard kernel regularized least-squares (e.g.,~\citet{smale2007learning}) applied to $M$ i.i.d.~draws from $\rho_{Y, Z}$. Assuming that $g_\rho \in \calG$, one can then pass the problem to controlling $\norm{\hg_\rho - g_\rho}_\calG^2$ and $\norm{\Fhat - F_\star}_{\ltwo(Q_X; \calG)}^2$, where $\ltwo(Q_X; \calG)$ denotes a Bochner space (reviewed in \Cref{sec:a:background:rkhs}).

To derive the convergence rates below, we show in \Cref{sec:a:complexity:conditional_mean} that the source condition on $F_\star$ can be expressed in terms of the singular decay exponent of $\M_{Z|X}$ (i.e.~$\sigma_i \sim i^{-\gamma_{X, Z}}$ from~\eqref{eq:framework:spectrum}), and the eigendecay exponents $\gamma_X$ and $\gamma_Z$ of the covariance operators of $Q_X$ and $Q_Z$, respectively. Additionally, $\omega_\rho > 1/2$ is a parameter controlling the convergence rate of the prompt-based estimate of $g_\rho$. The parametrization below is chosen so that one may interpret $\omega_\rho$ as a similar source condition for the target function $g_\rho$.
In the well-specified case (when $F_\star$ is contained in the hypothesis class), we describe the convergence rate with the aggregated exponent
\ZH{The theorem statement needs to be double checked against the one updated in Appendix. Double check the pointers to Appendix for proofs.}
\begin{align*}
    q(t) = (2\gamma_{X, Z} + \gamma_Z - 1)^{t}\gamma_X^{1 - t} \geq 1, \quad t \in [0, 1)
\end{align*}
where $t$ depends on $F_\star$. The result below corresponds to \Cref{thm:complexity:conditional_mean} in \Cref{sec:a:complexity:conditional_mean}, which relies on a basis alignment assumption to aggregate the singular/eigendecays.
\begin{theorem}[Informal]\label{thm:complexity_vvkr}
    For $\heta_\rho(\x) = \ipsmall{\hg_\rho, \Fhat(\x)}_\calG$, there exist $t \in [0, 1)$ and $C(Q_{X, Z}) \geq 0$ (independent of $(N, M, \delta)$) such that
    \begin{align}
        \norm{\heta_\rho - \eta_\rho}_{\ltwo(Q_X)}^2 \lesssim \polylog\p{1/\delta} \sbr{N^{-\frac{q(t)}{q(t) + 1}} + C(Q_{X, Z}) M^{-\frac{2\omega_\rho - 1}{2\omega_\rho + 1}}}\label{eq:complexity_vvkr}
    \end{align}
    with probability at least $1 - \delta$ for $N$ sufficiently large.
\end{theorem}
Let us interpret the constant $q(t)$. First, the dependence on $N$ ranges between $O(N^{-1/2})$ and the parametric rate $O(N)$. Convergence is faster when $\gamma_{X, Z} \gg 1$ or $\gamma_Z \gg 1$. The first case implies near-independence of $X$ and $Z$, for which learning is easy as $\Fhat(\x)$ is essentially constant over $\x \in \X$. The second case indicates that the $Z$ variable is near-finite dimensional, or that the vector-valued nature of the problem has been reduced to standard univariate regression. Convergence is slower if $\gamma_X \gg 1$, or if the effective dimension of $\X$ is small relative to the effective dimension of $Z$. The balancing constant $C(Q_{X, Z})$ (shown explicitly in \Cref{thm:complexity:conditional_mean}) decays with $\gamma_{X, Z}$ and $\gamma_Z$, so as $(X, Z)$ becomes more independent or $Z$ approaches finite dimensions, the variance from prompt sampling decreases. We also discuss the mis-specified case in \Cref{sec:a:complexity:conditional_mean}.

\myparagraph{Example 2: Radon-Nikodym Derivative Estimation}
This approach, based on~\eqref{eq:framework:rnd}, considers pre-training to return a learned information density $\Rhat: \X \times \Z \rightarrow \R_{\geq 0}$. By approximating the prompt distribution $\rho_{Y, Z}$ with $\hrho_{Y, Z}$ (e.g.~the empirical measure in the result below), one may define the estimator $\heta_\rho(\x) = \Ex_{\hrho_{Y,Z}}[\fstar(Y)\Rhat(\x, Z)]$. Similar in spirit to the previous example, we consider the kernel Radon-Nikodym derivative estimation with the spectral regularization procedure of \citet{nguyen2024onregularized}. The convergence rate of $\Rhat$ to $\Rsans$ is governed by a source condition constant $\beta \geq 1$ associated to $\Rsans$ (see \Cref{sec:a:complexity:rn_derivative}). We interpret this constant analogously to $q(t)$, in that we prove a relationship to the singular decay exponent $\gamma_{X, Z}$, but is not directly expressible in terms of the latter. \ZH{The theorem statement needs to be double checked against the one updated in Appendix. Double check the pointers to Appendix for proofs. Remember to talk about the trade-off.}
The following result corresponds to \Cref{thm:complexity:rn_derivative} in \Cref{sec:a:complexity:rn_derivative}.
\begin{theorem}[Informal]\label{thm:complexity_rn_derivative}
    For $\heta_\rho(\x) = \Ex_{\hrho_{Y,Z}}[\fstar(Y)\Rhat(\x, Z)]$, and $\rho_Z \ll Q_Z$, there exists $C_{\Rsans, \rho}(Q_X) \geq 0$ (independent of $(N, M, \delta)$) such that
    \begin{align*}
        \norm{\heta_\rho - \eta_\rho}_{\ltwo(Q_X)}^2 \lesssim \polylog\p{1/\delta} \sbr{N^{-\frac{\beta}{\beta + 1}} + C_{\Rsans, \rho}(Q_X) M^{-1}}\hspace{-1pt} + \hspace{-1pt} D_{\chi^2}(\rho_Z \Vert Q_Z)
    \end{align*}
     with probability at least $1 - \delta$ for all  $N$ sufficiently large.
\end{theorem}
Notice that the bound of \Cref{thm:complexity_rn_derivative}
includes a divergence term between $\rho_Z$ (the captions generated by prompting) and $Q_Z$ (the captions of the pre-training set). This term comes precisely from the error term in~\eqref{eq:framework:rnd}. This elucidates the fact that the conditional mean approach and the information density are not equivalent representations of the pre-training distribution, as one needs \emph{both} $\Rsans$ and $Q_Z$ in order to identify the conditional mean. 
The parametric rate $M^{-1}$ reflects that samples are used to learn a joint expectation over $\rho_{Y, Z}$, which is an easier statistical problem than estimating the regression function of $Y$ on $Z$ that appears in \Cref{thm:complexity_vvkr}. Thus, the information density approach may enjoy faster statistical convergence, at the expense of bias from the distribution mismatch on $\Z$. The constant $C_{\Rsans, \rho}(Q_X)$ relates to the $\ltwo(Q_X)$-norm of the random function $\x \mapsto \fstar(Y)\Rsans(\x, Z)$ for $(Y, Z) \sim \rho_{Y, Z}$; the error from finite prompts decays when this norm is light-tailed.

In both \Cref{thm:complexity_vvkr} and \Cref{thm:complexity_rn_derivative}, we aim to highlight not particular convergence rates of the chosen methods, but the framework that leads to proving them. Similar results can also be leveraged in our framework. SSL procedures such as noise contrastive estimation have been related to the estimation of $\Rsans$ \citep{gutmann2012noise}. For example, \citet[Theorem 11]{tosh2021contrastive} upper bounds $\norm{\Rhat - \Rsans}_{\ltwo(Q_XQ_Z)}^2$ using the suboptimality of the population risk, allowing for empirical risk minimization-style analysis.

\section{Experiments}\label{sec:experiments}
In \Cref{sec:intro}, we asked how the downstream task performance depends on the pre-training distribution $Q_{X, Z}$, evaluation distribution $P_{X, Y}$, and prompting strategy $\rho_{Y, Z}$. At the population level, we captured the dependence on $Q_{X, Z}$ and $P_{X, Y}$ using the residual dependence $\E{P_Z}{I(X; Z)}$ and incorporated $\rho_{Y, Z}$ via the prompt bias (\Cref{thm:res_dep1}). In the first experiment, we create a simulated setting in which the residual dependence can be controlled and investigate whether it is indeed a determining factor for the empirical performance of CLIP \citep{radford2021learning} and VICReg \citep{bardes2022vicreg} models in practice. In the second experiment, we solve an image classification task in which the images have both captions and labels (i.e.~we may sample from a true joint distribution $P_{X, Y, Z}$). This allows us to understand the effect of prompt bias by comparing template-based prompting strategies to the unbiased setting $\rho_{Y, Z} = P_{Y, Z}$. To understand the dependence on $\rho_{Y, Z}$ at a sample level, we explore how downstream performance scales with the number of prompts $M$ in both the second experiment (unbiased prompting) and third experiment (LLM-based prompting). We are particularly interested in verifying the dependence on $M$ (which is the dominant error when $N \gg M$) derived in \Cref{thm:complexity_rn_derivative}).
\Cref{sec:a:experiments} contains further details of the experiments and code for reproduction can be found at \href{github.com/ronakdm/zeroshot}{\url{github.com/ronakdm/zeroshot}}.

\myparagraph{Models, Datasets, and Evaluation}
For foundation models, we use three publicly available CLIP models from the OpenCLIP repository \citep{ilharco2021openclip}: ResNet50 pre-trained on YFCC15M \citep{thomee2016yfcc}, NLLB-CLIP pre-trained on a subset of LAION COCO \citep{visheratin2023nllb}, and ViT-B/32 pre-trained on the DataComp medium pool \citep{gadre2023datacomp}. Our evaluation datasets include five standard benchmarks: the Describable Textures Dataset or DTD \citep{cimpoi14describing}, Flowers 102 \citep{nilsback2008automated}, FGVC Aircraft \citep{maji2013fine}, SUN397 \citep{xiao2010sun}, and ImageNet-1k \citep{deng2009imagenet}. For some experiments, we make use of the ImageNet-Captions dataset \citep{fang2023data}, which pairs a subset of ImageNet images collected from Flickr with their original captions. 
Evaluation occurs via zero-shot classification top-$k$ accuracy, in which a test example is considered to be classified correctly if the true class is contained within the elements of $\Y$ with the $k$ largest scores as computed by~\eqref{eq:proxy}. Evaluation is done using tools from the \href{https://github.com/LAION-AI/CLIP_benchmark?tab=readme-ov-file}{CLIP Benchmark repository}. In \Cref{fig:ideal} and \Cref{fig:llm}, ``templates'' refers to using all \href{https://github.com/LAION-AI/CLIP_benchmark/blob/main/clip_benchmark/datasets/en_zeroshot_classification_templates.json}{default community-curated prompts} available in CLIP Benchmark. Finally, detailed descriptions of the prompt sampling schemes are collected and compared to the theory in \Cref{sec:a:complexity:sampling}.

\begin{figure}[t]
    \centering
    \includegraphics[width=0.8\linewidth]{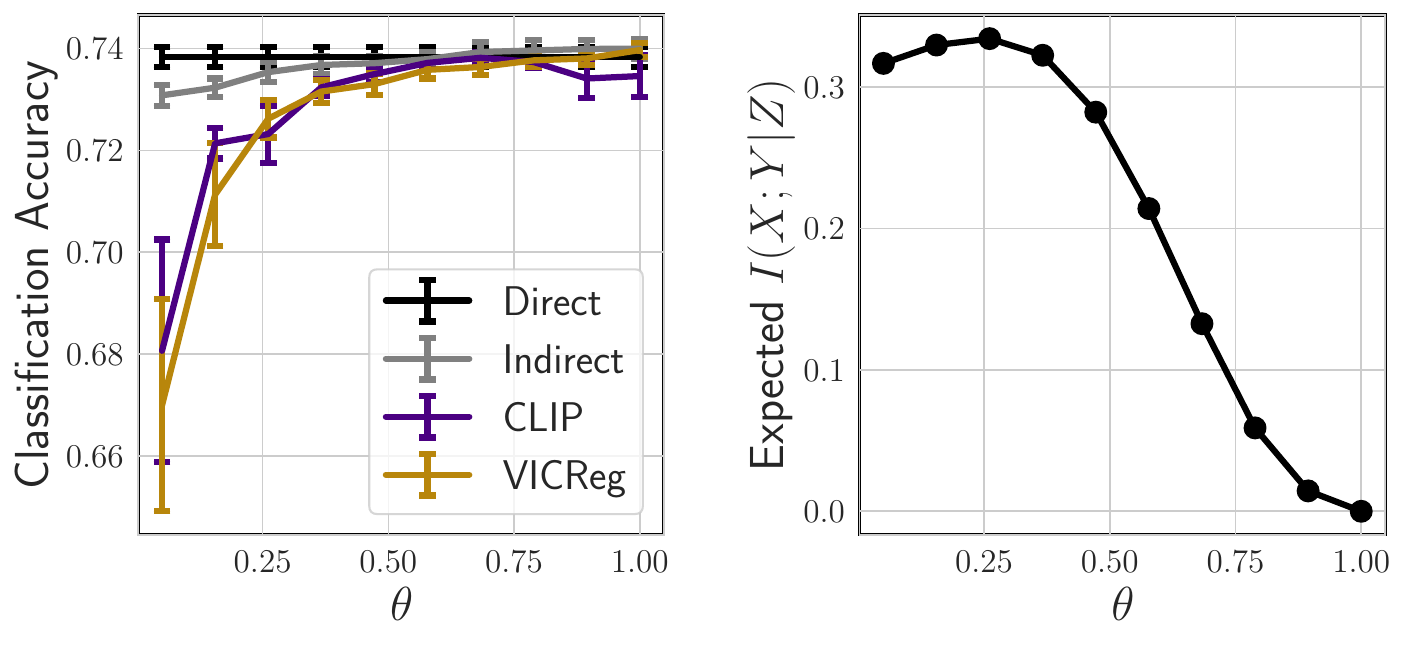}
    \caption{{\bf Results: Residual Dependence Simulation.} Simulation for $(X, Z, Y)$ described in \Cref{sec:a:experiments:simulation}. {\bf Left:} The $y$-axis is the accuracy of classifying $Y$ given $X$ and the $x$-axis is the parameter $\theta$ controlling the residual dependence $I(X;Y|Z)$ as in~\eqref{eq:re:resid}. {\bf Right:} The $y$-axis shows $\E{P_Z}{I(X;Y|Z)}$ as computed in \Cref{sec:a:experiments:simulation}. Error bars indicate standard errors from 10 seeds, which govern the data used for estimating expected values and randomness in the training procedures for CLIP and VICReg.}
    \label{fig:simulation}
\end{figure}

\begin{figure}[t]
    \centering
    \includegraphics[width=0.8\linewidth]{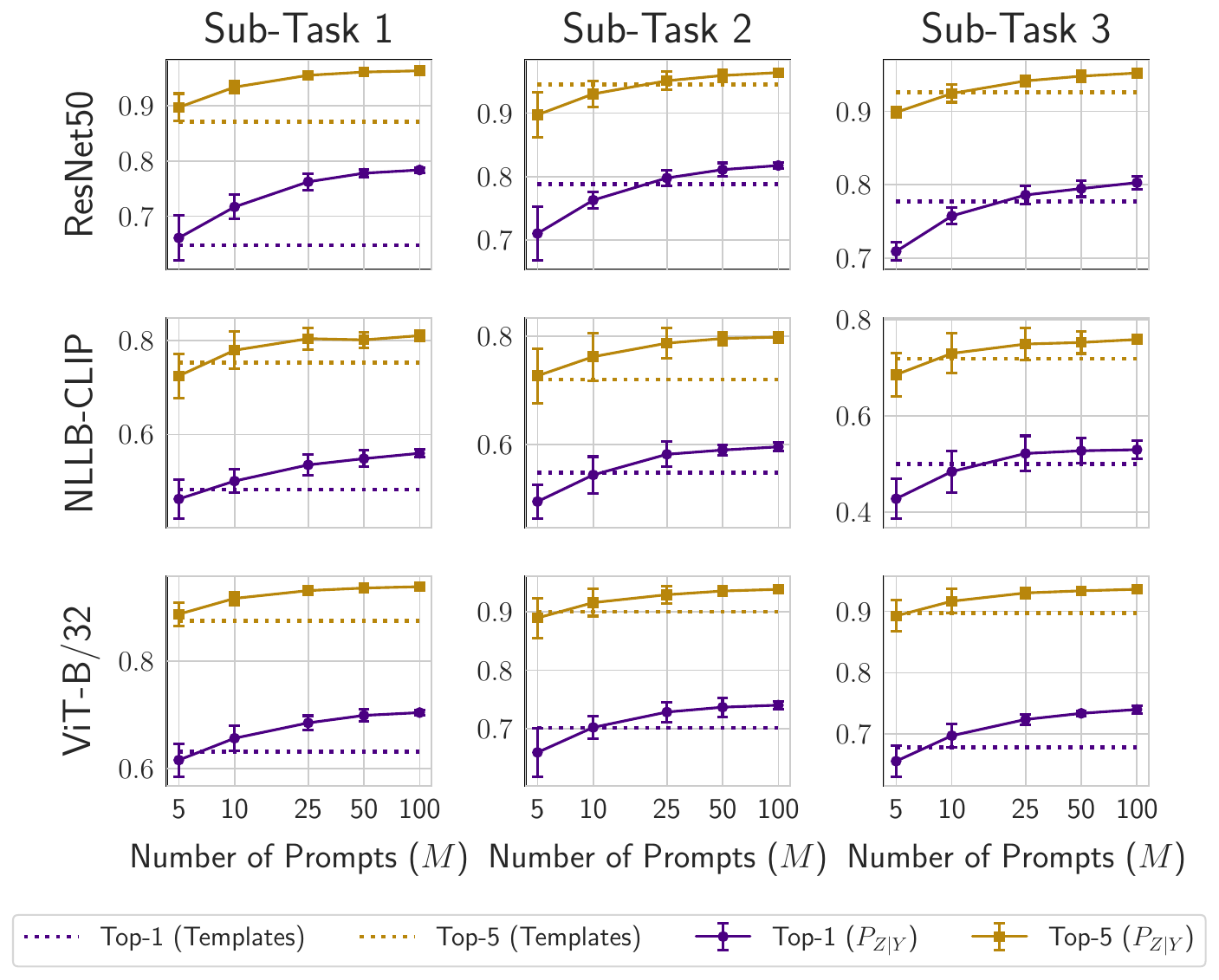}
    \caption{{\bf Results: Unbiased Prompting.} Pre-trained models are varied along the rows and sub-tasks (subsets of 50 ImageNet-1k class) are varied along columns. In all plots, the $x$-axis denotes the number of prompts sampled for each class embedding and the $y$-axis denotes top-$k$ zero-shot classification accuracy. Error bars indicate standard deviations across 10 seeds for prompt sampling.}
    \label{fig:ideal}
\end{figure}

\begin{figure*}[t]
    \centering
    \includegraphics[width=\linewidth]{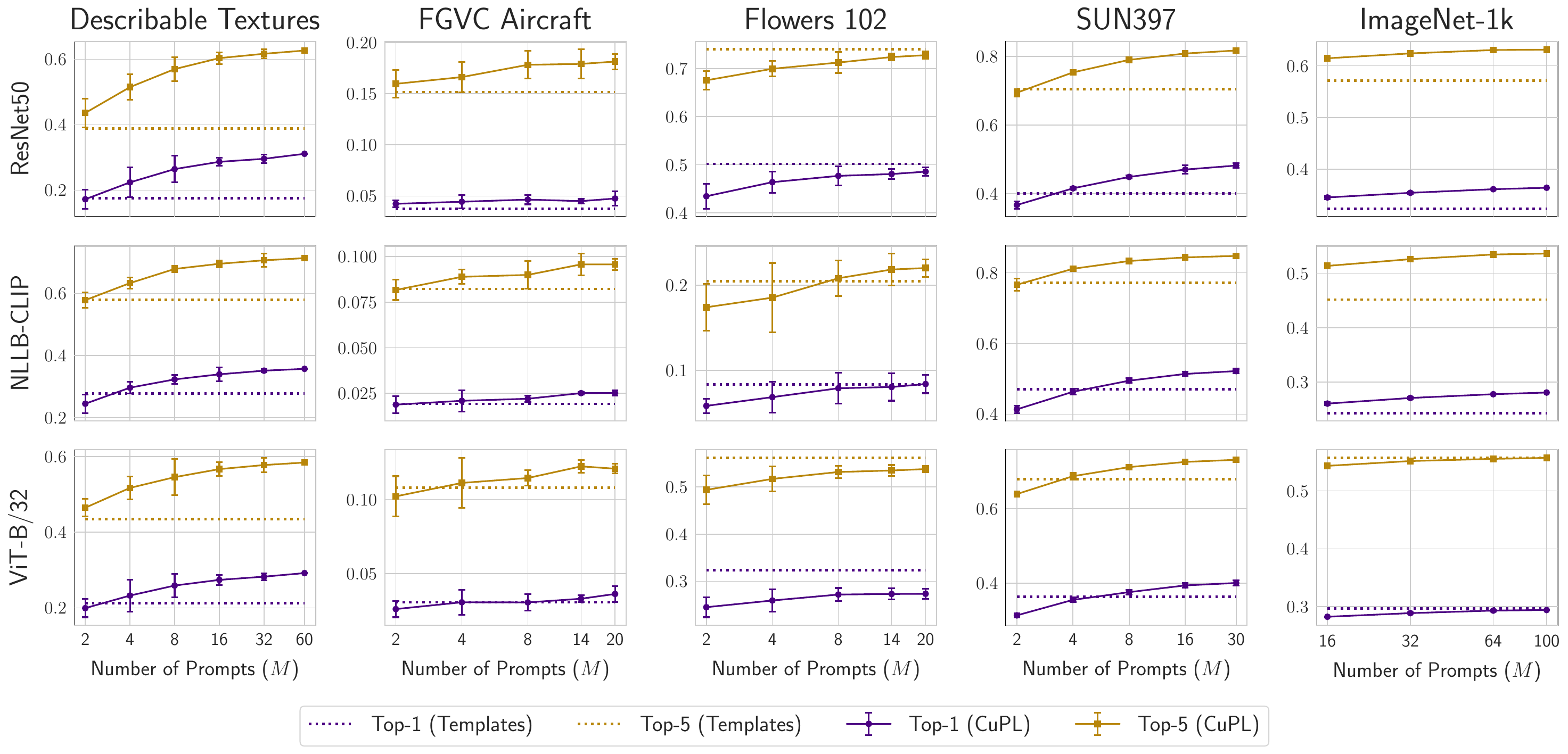}
    \caption{{\bf Results: Class-Conditional Prompting.} Pre-trained models are varied along the rows and evaluation datasets are varied along columns. In all plots, the $x$-axis denotes the number of prompts sampled for each class embedding and the $y$-axis denotes top-$k$ zero-shot classification accuracy. Error bars indicate standard deviations across 10 seeds for prompt sampling.}
    \label{fig:llm}
\end{figure*}

\myparagraph{Classification Accuracy and Residual Dependence}
We consider a simulated binary classification task in which all distributions are compatible (i.e.~$Q_{X, Z} = P_{X, Z}$ and $\rho_{Y, Z} = P_{Y, Z}$ for some $P_{X, Y, Z}$) and the predictors~\eqref{eq:est:bayes} and~\eqref{eq:est:ts} can be computed analytically.
We also include the zero-shot predictor~\eqref{eq:proxy} learned by both the CLIP and VICReg objectives. Our goals are two-fold in this simulation: 1) to empirically observe that as $\E{P_Z}{I(X; Y|Z)} \rightarrow 0$, the predictive performance of the indirect predictor $\eta_\rho$ does indeed approach that of $\etastar$, and 2) that the predictors generated by common SSL methods used in practice have similar performance trends as $\eta_\rho$. As for the data-generating process, we consider $\X = \Z = \R^d$ and a pair of Gaussian distributions $(P_{X, Z|Y = 0}, P_{X, Z|Y = 1})$, where given $Y = \y$,
\begin{align*}
    \begin{bmatrix}
        X \\
        Z
    \end{bmatrix}
    \sim
    \mc{N}\p{
    \begin{bmatrix}
        \bmu_{X|\y} \\
        \bmu_{Z|\y}
    \end{bmatrix},
    \begin{bmatrix}
        \C_{XX|\y} & \C_{XZ|\y} \\
        \C_{ZX|\y} & \C_{ZZ|\y}
    \end{bmatrix}
    }
\end{align*}
with class-conditional mean vectors $\bmu_{X|\y}, \bmu_{Z|\y} \in \R^d$ and covariance matrices $\C_{XX|\y}, \C_{ZX|\y}, \C_{ZZ|\y} \in \R^{d \times d}$. In order to control the conditional dependence between $X$ and $Y$ given $Z$, we fix all parameters except for $\bmu_{Z|\y}$ and $\C_{ZX|\y}$ (for $\y = 0, 1$), and define them using a tunable parameter $\theta \in [0, 1]$ in a way such that the conditional distribution of $Y$ given $X = \x$ stays constant. We make it so that as $\theta \rightarrow 1$, 
$I(X;Y|\z) \rightarrow 0$. Finally, to measure classification accuracy, we directly draw samples from $P_{Y, Z}$ to simulate unbiased prompting. The full mathematical details are given in \Cref{sec:a:experiments:simulation}. We observe both of the intended effects; the left panel of \Cref{fig:simulation} demonstrates that as $\theta$ approaches $1$, the indirect, CLIP, and VICReg predictors approach the performance of the direct predictor in terms of classification performance. The right panel confirms that $\theta$ indeed controls $\E{P_Z}{I(X; Y|Z)}$ in an approximately monotonic fashion.

\myparagraph{Prompting without Bias with Observations from $P_{Z, Y}$}
Next, we illustrate the importance of the prompt bias term in \Cref{thm:res_dep1} by considering an  ImageNet-Captions dataset, in which we may observe the joint sample $(X, Y, Z)$. 
We compare the standard prompting technique using pre-defined templates to the unbiased strategy that draws samples directly from $P_{Y, Z}$. We design three sub-tasks by randomly selecting collections of 50 classes from each of 998 classes, reserving held-out prompting examples for which we can draw from $P_{Z|Y = \y}$ for each $\y \in \Y$ (see the additional details in \Cref{sec:a:experiments}). The zero-shot classification accuracy on a held-out evaluation set is plotted in \Cref{fig:ideal}.
Observe that the threshold at which unbiased prompting outperforms the 18 default templates is approximately $M = 10$ across tasks. However, the performance of the unbiased approach only saturates at $M = 100$ and can have enormous benefits (almost 15\% absolute increase in top-1 accuracy for the ResNet50 on Sub-Task 1) in performance. Thus, for models that have not yet been saturated from pre-training, prompting can close surprisingly wide gaps in zero-shot classification accuracy.

\myparagraph{Class-Conditional Prompting with Language Models}
As mentioned in \Cref{sec:intro}, we investigate CuPL as a means to implement class-conditional prompting (sampling from $\rho_{Z|Y = \y}$ for each $\y \in \Y$) with LLMs. Our experimental setup and scientific goals differ from those used in \citet{pratt2023what}: 1) we use lightweight encoders that have not saturated their performance during pre-training, as opposed to the large-scale ViT-L/14 architecture, 2) we quantify the variability of classification accuracy with respect to prompting by generating up to fifty times as many prompts per experiment, and 3) we employ LlaMA 3 \citep{grattafiori2024llama3herdmodels}, which is free and accessible to other, as opposed to GPT-3 \citep{brown2020language}.
The results are shown in \Cref{fig:llm}, where we order the datasets in increasing number of classes per task: 47, 100, 102, 397, and 998. Similar phenomena as in \Cref{fig:ideal} are observed, although the approximate saturation threshold varies per dataset from 20 for Flowers 102 and FGVC Aircraft up to 60 for DTD. Note that the choice of defaults heavily influences the baseline performance. Surprisingly, the Flowers 102 dataset uses a single default: ``a photo of a \underline{\hspace{0.5cm}}, a type of flower'', and is often able to outperform the class-conditional LLM approach on average. On the other hand, the DTD templates of the form ``a photo of a \underline{\hspace{0.5cm}} \{texture, pattern, thing, object\}'' are dramatically outperformed by our LLM-generated captions 
, with a nearly 20\% increase in top-5 accuracy on the ResNet50 and ViT-B/32 architectures.

\section{Conclusion}\label{sec:conclusion}
We showed how zero-shot prediction (ZSP) can be theoretically understood 
as an indirect prediction path from another modality to the label. We presented two viewpoints on categorizing ZSP methods---the conditional mean approach and the information density approach---and framed a decomposition formula for their generalization abilities.
Our theoretical results and experiments highlighted the role of residual dependence and prompt bias in defining the fundamental limits of ZSP. Interesting venues for future work include the extension of our analysis 
to classes of distribution shifts between the pre-training distribution and the downstream distribution, 
and to causal generative modeling \citep{scetbon2024fixed,zhang2024towards}.

\section*{Acknowledgements}The authors are grateful to
D. Hsu, 
E. Perkovi\'{c}, and 
N. Srebro
for fruitful discussions related to this work. 
The authors also thank the reviewers and the area chair
for valuable comments. 
This work was supported by NSF DMS-2023166, CCF-2019844, DMS-2134012, NIH, and IARPA 2022-22072200003. Part of this work was performed while R. Mehta and Z. Harchaoui were visiting the Simons Institute for the Theory of Computing. 

\bibliographystyle{abbrvnat}
\bibliography{bib}

\clearpage
\appendix
\onecolumn
\addcontentsline{toc}{section}{Appendix} 
\part{Appendix} %
\parttoc %

\clearpage

\section{Notation}\label{sec:a:notation}
\begin{table}[ht]
    \centering

    \begin{adjustbox}{max width=\linewidth}
    \renewcommand{\arraystretch}{1.2}
    \begin{tabular}{cc}
    \toprule
        {\bf Symbol} & {\bf Description}\\
        \midrule
        $\x \in \X$, $\y \in \Y$, $\z \in \Z$      & 
        \begin{tabular}{c}
            Instances and sample spaces for data modalities/views, \\
            often images, labels, and captions.
        \end{tabular}
        \\
        $\balpha, \bbeta$      & Encoders $\balpha: \X \rightarrow \R^d$ and $\bbeta: \Z \rightarrow \R^d$.\\
        $(X, Y, Z)$      & Random variable realized in $\X \times \Y \times \Z$.\\
        $P_{X, Y}$      & Evaluation distribution over $\X \times \Y$.\\
        $Q_{X, Z}$      & Pre-training distribution over $\X \times \Z$.\\
        $\rho_{Y, Z}$      & Prompting distribution over $\Y \times \Z$.\\
        
        \midrule

        $\fstar$      & A function $\fstar: \Y \rightarrow \R$.\\
        $\eta_\star(\x)$ & Direct predictor $\E{P_{X, Y}}{\fstar(Y)|X}(\x)$.\\
        $g_\rho(\z)$ & Prediction function $\E{\rho_{Z, Y}}{\star(Y)|Z}(\z)$.\\
        $\eta_\rho(\x)$ & Indirect predictor $\E{Q_{X, Z}}{g_\rho(Z)|X}(\x)$.\\
        $N$ & Sample size of pre-training set $(X_1, Z_1), \ldots, (X_N, Z_N) \iidsim Q_{X, Z}$.\\
        $M$ & Number of prompts $(Y_1, Z_1) \ldots, (Y_M, Z_M) \iidsim \rho_{Y, Z}$.\\

        \midrule
        
        $\ltwo(P_X)$ & 
        \begin{tabular}{c}
             Set containing equivalence classes of measurable functions $h: \X \rightarrow \R$ \\
             satisfying $\norm{h}_{\ltwo(P_X)}^2 = \int h^2(\x)\d P_X(\x) < +\infty$.
        \end{tabular}\\
        $\M_{Z|X}$ & Conditional mean operator $[\M_{Z|X} g](\z) = \E{Q_{X, Z}}{g(Z)|X}(\x)$.\\
        $\Rsans$ & Information density $\frac{\d Q_{X, Z}}{\d Q_X Q_Z}: \X \times \Z \rightarrow \R_{\geq 0}$.\\
        $D_{\chi^2}(P \Vert Q)$ & $\chi^2$-divergence $\E{U \sim Q}{(\frac{\d P}{\d Q}(U) - 1)^2}$.\\
        $I(X;Z)$ & Mean square contingency $D_{\chi^2}(Q_{X, Z} \Vert Q_X Q_Z)$.\\
        $(\sigma_i)_{i=1}^\infty$ & Singular values of $\M_{Z|X}$.\\
        $(\alpha_i, \beta_i)_{i=1}^\infty$ & Left and right singular functions of $\M_{Z|X}$.\\
        $\Ssans_{\z}$ & Conditional information density $\frac{\d P_{X, Y|\z}}{\d P_{X|\z} Q_{X|\z}}: \X \times \Y \rightarrow \R_{\geq 0}$.\\
        $I(X;Y|\z)$ & Conditional mean square contingency $D_{\chi^2}(P_{X, Y|\z} \Vert P_{X|\z} P_{Y|\z})$.\\
        $\norm{\cdot}_{\HS(\calG, \calH)}$ & Hilbert-Schmidt norm of a linear operator from $\calG$ to $\calH$.\\

        \bottomrule
    \end{tabular}
    \end{adjustbox}
    \vspace{6pt}
    \caption{Notation used throughout the main text.}
    \label{tab:notation}
\end{table}
In the appendix, we use slightly more explicit notation. For example, the product measure of $Q_X$ and $Q_Z$ on $\X \times \Z$ is denoted $Q_X \otimes Q_Z$. The bracket notation $[\cdot]_X$ and $[\cdot]_Z$ are used to indicate equivalence classes in $\ltwo(Q_X)$ and $\ltwo(Q_Z)$, respectively. Such changes are marked as they are introduced.

\clearpage

\section{Technical Background}\label{sec:a:background}
In this section, we review the necessary background and construct any theoretical tools used in our analyses in a self-contained manner. \Cref{sec:a:background:l2} describes the broadest function class we consider and gives a rigorous description of the conditional means that we employ in this work. \Cref{sec:a:background:compact} reviews the basic classes of linear operators (trace class, Hilbert-Schmidt, etc.) that we consider. \Cref{sec:a:background:buja} contains central tools regarding the structure of bivariate distributions. Finally, \Cref{sec:a:background:rkhs} contains a brief introduction to reproducing kernel Hilbert spaces and some recent statistical results used in the proofs of \Cref{thm:complexity_vvkr} and \Cref{thm:complexity_rn_derivative}.

\subsection{Conditional Expectation and the Hilbert Space \texorpdfstring{$\ltwo$}{L2}}\label{sec:a:background:l2}
Consider a common probability space $(\Omega, \msc{F}, \prob)$ and a topological space $\X$ equipped with its Borel $\sigma$-algebra $\calB(\X)$. Given a random variable $X: \Omega \rightarrow \X$ representing some observable data, we consider $P_X$ to be the law of $X$, i.e.~$P_X(B) = \prob(X^{-1}(B))$ for every Borel set $B \in \calB(\X)$. Our goal is to define $\ltwo(P_X)$, a Hilbert space containing equivalence classes of functions that are square integrable under $P_X$. As an intermediate step, we will also construct a Hilbert space $\Lsans(\msc{G})$ for the $\sigma$-algebra $\msc{G} \sse \msc{F}$, which contains equivalence classes of $\msc{G}$-measurable functions that are square integrable under $\prob$. Having both of these constructions will be helpful in working with conditional mean operators in a rigorous manner. 

\myparagraph{Quotient Space}
As a starting point, consider the set
\begin{align*}
    \Lsans_+(\msc{F}) := \br{\text{$\msc{F}$-measurable functions } u: \Omega \rightarrow \R \text{ satisfying } \norm{u}_{\Lsans_+(\msc{F})}^2 \defeq \int_\Omega u^2(\omega) \d \prob(\omega) < \infty}.
\end{align*}
For any $u, v \in \Lsans_+(\msc{F})$, consider the equivalence relation ``$\sim$'' defined by
\begin{align}
    u \sim v \iff \exists \Omega_1 \in \msc{F} \text{ such that } u(\omega) = v(\omega) \ \forall \omega \in \Omega_1 \text{ and } \prob(\Omega_1) = 1.\label{eq:equivalence_relation}
\end{align}
For any $u_+ \in \Lsans_+(\msc{F})$, we define $[u_+]_\sim \in \Lsans(\msc{F})$ as indexing the equivalence class containing all functions that differ from $u_+$ only on a set of $\prob$-measure zero.
The global Hilbert space will be defined using the quotient of $\Lsans_+(\msc{F})$ under this equivalence relation.
\begin{lemma}\label{lem:hilbert_space}
    The quotient space $\Lsans(\msc{F}) = \Lsans_+(\msc{F}) / \sim$ is a Hilbert space with the addition and scalar multiplication rules
    \begin{align*}
        (u, v) \mapsto au + bv \defeq [au_+ + bv_+]_\sim \text{ for some } u_+ \in u \text{ and } v_+ \in v,
    \end{align*}
    for scalars $a, b\in \R$ and the inner product
    \begin{align*}
        (u, v) \mapsto \ip{u, v}_{\Lsans(\msc{F})} \defeq \int_\Omega u_+(\omega) v_+(\omega) \d \prob(\omega) \text{ for some } u_+ \in u \text{ and } v_+ \in v,
    \end{align*}
    where the definitions are independent of the choice of $u_+$ and $v_+$.
\end{lemma}
\begin{proof}
    It is easy to verify that the addition, scalar multiplication, and inner product operations are well-defined (i.e.~are invariant to the choice of $u_+$ and $v_+$). Define the norm $u \mapsto \norm{u}_{\Lsans(\msc{F})} \defeq \sqrt{\ip{u, u}_{\Lsans(\msc{F})}}$, and consider a Cauchy sequence $(u\pow{n})_{n=1}^\infty$ in $\Lsans(\msc{F})$. To confirm completeness, we identify a limit of this sequence as an element of $\Lsans(\msc{F})$. First, consider an arbitrary sequence $u_+\pow{1}, u_+\pow{2}, \ldots$ where $u_+\pow{n} \in u\pow{n}$ for all $n \geq 1$. Then, we have by the Riesz-Fischer theorem \citep[Theorem 13.7]{schilling2017measures}, there exists a limit $u_+ \in \Lsans_+(\msc{F})$ such that
    \begin{align}
        \lim_{n \rightarrow \infty} \norm{u_+\pow{n} - u_+}_{\Lsans_+(\msc{F})} \rightarrow 0.\label{eq:thm:riesz_fischer}
    \end{align}
    We then define $\lim_{n \rightarrow \infty} u\pow{n} \defeq [u_+]_\sim$, and see that
    \begin{align*}
        \norm{u\pow{n} - [u_+]_\sim}_{\Lsans(\msc{F})} = \norm{u\pow{n}_+ - u_+}_{\Lsans_+(\msc{F})} \rightarrow 0 \text{ as } n \rightarrow \infty,
    \end{align*}
    where the last step follows by~\eqref{eq:thm:riesz_fischer} and completes the proof.
\end{proof}
Next, we construct closed subspaces of $\Lsans(\msc{F})$ which contain random variables that are measurable functions of another random variable. Notice that for $u,v \in \Lsans(\msc{F})$, the statement $u = v$ indicates equality of two partitions, namely collections of random variables that differ pairwise on sets of measure zero. 
Letting $\sigma(X)$ denote the $\sigma$-algebra generated by $X$, define the set
\begin{align}
    \Lsans_+(\sigma(X)) \defeq \br{u \in \Lsans_+(\msc{F}) \text{ s.t. } u \text{ is $\sigma(X)$-measurable}}.\label{eq:meas_func_subspace1}
\end{align}
Then, using the equivalence relation~\eqref{eq:equivalence_relation}, we define the space
\begin{align*}
    \Lsans(\sigma(X)) \defeq \Lsans_+(\sigma(X)) / \sim.
\end{align*}
In the upcoming \Cref{coro:closed_subspaces}, we will confirm that $\Lsans(\sigma(X))$ is indeed a closed subspace of $\Lsans(\msc{F})$ for any random variable $X$. Before doing so, we consider the induced probability space $(\X, \mc{B}(\X), P_X)$. Then, we define the related linear space 
\begin{align*}
    \ltwo_+(P_X) := \br{\text{measurable functions } f: \X \rightarrow \R \text{ satisfying } \norm{f}_{\ltwo_+(P_X)}^2 \defeq \int_{\X} f^2(\x) \d P_X(\x) < \infty}.
\end{align*}
We define an analogous equivalence relation ``$\sim_X$'' defined as
\begin{align}
    f \sim_X g \iff \exists \X_1 \in \mc{B}(\X) \text{ such that } f(\x) = g(\x) \ \forall \x \in \X_1 \text{ and } P_X(\X_1) = 1,\label{eq:equivalence_relation_A}
\end{align}
and the quotient $\ltwo(P_X) \defeq \ltwo_+(P_X) / \sim_X$. These sets are related to one another in the following lemma.
\begin{corollary}\label{coro:closed_subspaces}
    The set $\Lsans(\sigma(X))$ is a Hilbert space with respect to the inner product used in \Cref{lem:hilbert_space}, whereas $\ltwo(P_X)$ is a Hilbert space with respect to the analogous inner product for $(\X, \mc{B}(\X), P_X)$. Furthermore, $\Lsans(\sigma(X))$ is a closed subspace of $\Lsans(\msc{F})$, and it holds that
    \begin{align}
        \Lsans(\sigma(X)) = \ltwo(P_X) \circ X \defeq \br{[f_+(X(\cdot))]_\sim: f_+ \in \ltwo_+(P_X)}.\label{eq:composition}
    \end{align}
\end{corollary}
\begin{proof}
    That $\Lsans(\sigma(X))$ and $\ltwo(P_X)$ are Hilbert spaces follows by identical arguments to \Cref{lem:hilbert_space}. Additionally, we may invoke \citet[Lemma 27.1]{schilling2017measures} to assert that $\Lsans(\sigma(X))$ is a closed subspace of $\Lsans(\msc{F})$. Finally, to show~\eqref{eq:composition}, we will show that
    \begin{align*}
        \Lsans_+(\sigma(X)) = \ltwo_+(P_X) \circ X \defeq \br{f_+(X(\cdot)): f_+ \in \ltwo_+(P_X)}
    \end{align*}
    and take the quotient with respect to ``$\sim$'' on either side to complete the proof. First, $\ltwo_+(P_X) \circ X \sse \Lsans_+(\sigma(X))$ holds because $f_+(X)$ is clearly $\sigma(X)$-measurable and
    \begin{align}
        \norm{f_+(X)}_{\Lsans_+(\msc{F})}^2 = \int_\Omega (f_+(X(\omega)))^2 \d \prob(\omega) \overset{P_X = X_\# \prob}{=} \int_{\X} f_+^2(\x) \d P_X(\x) = \norm{f_+}_{\ltwo_+(P_X)}^2 < \infty. \label{eq:equiv_norm}
    \end{align}
    To show that $\Lsans_+(\sigma(X)) \sse \ltwo_+(P_X) \circ A$, first note that for any $\sigma(X)$-measurable random variable $U$, there exists a measurable function $g_+: \X \rightarrow \R$ such that $U = g_+(X)$ \citep[Exercise 1.3.8]{durrett2019probability}. Applying~\eqref{eq:equiv_norm} gives $\norm{g_+}_{\ltwo_+(P_X)}^2 = +\infty \implies \norm{g_+(X)}_{\Lsans_+(\msc{F})}^2 = +\infty$, which yields a contradiction as $\norm{g_+(X)}_{\Lsans_+(\msc{F})}^2 = \norm{U}_{\Lsans_+(\msc{F})}^2 < +\infty$. Thus, $\norm{g_+}_{\ltwo_+(P_X)}^2 < +\infty$, completing the proof.
\end{proof}

\myparagraph{Conditional Expectation}
Using \Cref{coro:closed_subspaces}, for any collection of random variables $(X, Z, Y)$, we can now construct the Hilbert subspaces $\ltwo(P_{X, Y})$, $\ltwo(P_{X})$, $\ltwo(P_{Z})$. We can then identify them with conditional expectations, i.e.~projections onto $\Lsans(\sigma(X, Y))$, $\Lsans(\sigma(X))$, $\Lsans(\sigma(Z))$, respectively. This is done in the definition below.
\begin{definition}[Conditional Expectation]\label{def:conditional_expectation}
    For any random variable $U \in \Lsans(\msc{F})$, we define the \emph{conditional expectation} $\E{}{U|\sigma(X)}$ as the orthogonal projection of $U$ onto $\Lsans(\sigma(X))$, or
    \begin{align*}
        \E{}{U|\sigma(X)} := \argmin_{u \in \Lsans(\sigma(X))} \norm{u - U}_{\Lsans(\msc{F})}^2,
    \end{align*}
    which uniquely exists due to the closedness of $\Lsans(\sigma(X))$ and the projection theorem \citep[Theorem 26.13]{schilling2017measures}. Owing to \Cref{coro:closed_subspaces}, we will also define the \emph{conditional expectation function}
    \begin{align*}
        \E{}{U|X}: \X \rightarrow \R  
    \end{align*}
    as any measurable function satisfying the conditions $[\E{}{U|X}]_{\sim} \in \ltwo(P_X)$ and $\E{}{U|\sigma(X)}(\omega) = \E{}{U|X}(X(\omega))$ for $\prob$-almost every $\omega \in \Omega$. The specific function choice will not affect any of the forthcoming arguments.
\end{definition}
Here, we defined the conditional expectation as an element of $\Lsans(\sigma(X))$ and associated it with a function in $\ltwo(P_X)$. Without the squared-integrability requirement, the conditional expectation may also be defined using the familiar tower property. We include the tower property below for completeness.
\begin{lemma}{\citep[Theorem 27.12]{schilling2017measures}}\label{lem:tower}
    Consider $U \in \Lsans(\msc{F})$ and $X: \Omega \rightarrow \X$. Then, for every measurable set $A \in \sigma(X)$, it holds that
    \begin{align*}
        \int_A U(\omega) \d \prob(\omega) = \int_A \E{}{U|\sigma(X)}(\omega) \d \prob(\omega) =  \int_{X(A)} \E{}{U|X}(\x) \d P_X(\x).
    \end{align*}
\end{lemma}
We will make use of both the projection property and tower property throughout this manuscript. While conditional expectation may be defined for specific integrable functions, we may wish to define probability measures whose integrals can produce all conditional expectations simultaneously---this ideal is captured by \emph{regular conditional distributions (r.c.d.'s)} \citep{shorack2000probability}, which we recall below.

\begin{definition}\label{def:rcd}
    Consider random variables $(U, V): \Omega \rightarrow \msc{U} \times \msc{V}$. Let $\calB(\msc{U})$ denote the Borel $\sigma$-algebra on $\msc{U}$. A map: $\mu: \msc{V} \times \calB(\msc{U}): \rightarrow [0, 1]$ is called a \emph{regular conditional distribution (r.c.d.)} if the following two properties hold:
    \begin{enumerate}
        \item For each $A \in \calB(\msc{U})$ and $\v \in \msc{V}$, it holds that
        \begin{align*}
            \mu(\v, A) = \E{P_{U, V}}{\ind_A(U)|V}(\v),
        \end{align*}
        for the conditional expectation defined in \Cref{def:conditional_expectation}.
        \item For $P_V$-almost every $\v \in \msc{V}$, $\mu(\v, \cdot)$ is a probability measure on $\calB(\msc{U})$.
    \end{enumerate}
\end{definition}
This will primarily be used for the conditional dependence arguments in \Cref{sec:a:dependence}.

\subsection{Compact Operators}\label{sec:a:background:compact}
We collect several generalities about Hilbert spaces and linear operators (hereafter, simply ``operators'') between them. Many computations will require expanding an element of a separable Hilbert space onto an orthonormal basis.
\begin{definition}[Separability, Orthonormal Basis, Complete Orthonormal System]\label{def:orthonormal_basis}
    For a Hilbert space $(\calH, \ip{\cdot, \cdot}_\calH)$ over $\R$, the orthonormal system $e_1, e_2, \ldots \in \calH$ of vectors is called an \emph{orthonormal basis (ONB)} or \emph{complete orthonormal system (CONS)} of $\calH$ if any of the following properties hold, which are equivalent by \citet[Theorem 26.21]{schilling2017measures}.
    \begin{enumerate}
        \item For every $h \in \calH$, $\ip{h, e_i}_\calH = 0$ for all $i \geq 1$ implies that $h\equiv 0$.
        \item $\bigcup_{n=1}^\infty \Span\br{e_1, \ldots, e_n}$ is dense in $\calH$.
        \item For every $h \in \calH$, it holds that  $h = \sum_{i=1}^\infty \ip{h, e_i}_{\calH}e_i$.
        \item For every $h \in \calH$, it holds that $\sum_{i=1}^\infty \abs{\ip{h, e_i}_{\calH}}^2 = \norm{h}_{\calH}^2$.
        \item For every $h, h' \in \calH$, it holds that $\sum_{i=1}^\infty \ip{h, e_i}_{\calH} \ip{h', e_i}_{\calH} = \ip{h, h'}_\calH$.
    \end{enumerate}
    If there exists a countable orthonormal basis, then $\calH$ is called \emph{separable} \citep[Definition 26.23 \& Theorem 26.24]{schilling2017measures}.
\end{definition}
When linear operators are compact, then we may decompose them in a way that generalizes the eigendecomposition and singular value decomposition for matrices.
\begin{definition}[Compact Operator]
    A linear operator $\M: \calG \rightarrow \calH$ between Hilbert spaces $\calG$ and $\calH$ is called \emph{compact} if for every totally bounded subset $B \sse \calG$, the image $\M(B)$ is relatively compact (i.e.~the closure of $\M(B)$ is compact) in $\calH$.
\end{definition}
Compact operators are bounded, and every bounded linear operator $\M$ admits a unique \emph{adjoint operator} $\M^*$ satisfying $\ip{h, \M g}_{\calH} = \ip{\M^* h, g}_{\calG}$ for all $g \in \calG$ and $h \in \calH$. An operator $\Top: \calH \rightarrow \calH$ is called \emph{self-adjoint} if $\Top = \Top^*$. Next, we collect two operator decompositions that will be used repeatedly. We refer the reader to \citet[Chapter IV]{gohberg2003classes} and \citet[Chapter X]{gohberg2003classes} for further discussion on these topics. Just as their analogs for matrices, we refer to them as the \emph{eigendecomposition} and \emph{singular value decomposition}, respectively.

\begin{theorem}{\citep[Chapter IV, Theorem 5.1]{gohberg2003classes}}\label{thm:eigen}
    Let $\Top: \calH \rightarrow \calH$ be a compact, self-adjoint operator on a separable Hilbert space $\calH$ on $\R$. Then, there exists a countable orthonormal basis $\br{e_j}_{j \in J}$ of $\calH$ and a sequence of non-zero real numbers $\br{\lambda_i}_{i \in I}$ with $\lambda_i \rightarrow 0$, $I \sse J$, and for all $h \in \calH$, we have that
    \begin{align}
        \Top h = \sum_{i \in I} \lambda_i \ip{h, e_i}_\calH e_i.\label{eq:operators:eigen}
    \end{align}
    Furthermore, if $\ip{h, \Top h}_\calH \geq 0$ for all $h \in \calH$ (i.e.~$\Top$ is \emph{positive semidefinite}), then we may take $\lambda_i > 0$ for all $i \in I$, and order them in a non-increasing sequence.. We call $\br{\lambda_i}_{i \in I}$ the non-zero \emph{eigenvalues} of $\Top$.
\end{theorem}
\begin{theorem}{\citep[Chapter X, Theorem 4.2]{gohberg2003classes}}\label{thm:svd}
    Let $\M: \calG \rightarrow \calH$ be a compact operator between separable Hilbert spaces $\calG$ and $\calH$ on $\R$. Then, there exists an orthonormal basis $\br{u_j}_{j \in J}$ of $\calH$, an orthonormal basis $\br{v_k}_{k \in K}$ of $\calG$, and a sequence of positive real numbers $\br{s_i}_{i \in I}$ with $s_i \rightarrow 0$ such that the following statements hold.
    \begin{itemize}
        \item All collections are at most countable, i.e.~$I, J, K \sse \mathbb{N}$, and $I \sse J \cap K$.
        \item For all $g \in \calG$ and $h \in \calH$, we have that
        \begin{align}
            \M g = \sum_{i \in I} s_i \ip{g, v_i}_\calG u_i \text{ and }  \M^* h = \sum_{i \in I} s_i \ip{h, u_i}_\calH v_i.\label{eq:operators:svd}
        \end{align}
    \end{itemize}
     We call $\br{s_i}_{i \in I}$ the non-zero \emph{singular values} of $\M$, which can be ordered in a non-increasing sequence.
\end{theorem}
The sets $J$ and $K$ are used to index the bases of $\calH$ and $\calG$, so they may be larger in cardinality than $I$, which only indexes the non-zero eigenvalue and singular values, respectively. 
We will also consider more specific classes of compact operators.
\begin{definition}\label{def:trace_hs}
    A compact operator $\M$ with singular values $\br{s_i}_{i \in I}$ (\Cref{thm:svd}) is called \emph{trace class} if $\sum_{i \in I} s_i < +\infty$ (the singular values are summable) and \emph{Hilbert-Schmidt} if $\sum_{i \in I} s_i^2 < +\infty$ (the singular values are square summable).
\end{definition}
Using the singular value decomposition, we see that if $\M$ is Hilbert-Schmidt, then $\M \M^*$ and $\M^* \M$ are self-adjoint trace class operators.
The set of all Hilbert-Schmidt operators $\M:\calG \rightarrow \calH$ between Hilbert spaces $(\calG, \ip{\cdot, \cdot}_\calG)$ and $(\calH, \ip{\cdot, \cdot}_\calH)$ will be denoted by $\HS(\calG, \calH)$. This is itself a Hilbert space with the inner product
\begin{align*}
    \ip{\A, \B}_{\HS(\calG, \calH)} = \sum_{j \in J} \ip{\A g_j, \B g_j}_\calH
\end{align*}
where $\br{g_j}_{j \in J}$ can be taken to be \emph{any} orthonormal basis of $\calG$. Similarly,  let $\br{h_k}_{k \in K}$ be an arbitrary orthonormal basis of $\calH$. Then, the Hilbert-Schmidt norm $\norm{\A}_{\HS(\calG, \calH)}$ will be defined as
\begin{align}
    \norm{\A}_{\HS(\calG, \calH)}^2 &= \ip{\A, \A}_{\HS(\calG, \calH)} \notag\\
    &= \sum_{j \in J} \ip{\A g_j, \A g_j}_\calH \notag\\
    &= \sum_{j \in J} \sum_{k \in K} \sum_{l \in K} \ip{\A g_j, h_k}_\calH  \ip{\A g_j, h_l}_\calH \ip{h_k, h_l}_\calH \notag\\
    &= \sum_{j \in J} \sum_{k \in K} \ip{h_k, \A g_j}_{\calH}^2  = \sum_{j \in J} \sum_{k \in K} \ip{\A^* h_k, g_j}_{\calG}^2. \label{eq:compact:hsnorm}
\end{align}
Using the singular value decomposition, we see that~\eqref{eq:compact:hsnorm} is equal to the sum of the squared singular values referenced in \Cref{def:trace_hs}.
For $h \in \calH$ and $g \in \calG$, we define the rank-one operator $h \otimes g: \calG \rightarrow \calH$ via $(h \otimes g)g' = \ip{g, g'}_\calG h$ for all $g' \in \calG$. For an operator $\A \in \HS(\calG, \calH)$, the following identity regarding rank-one operators will be useful for norm computations:
\begin{align*}
    \ip{h, \A g}_\calH = \ip{\A^* h, g}_\calG = \ip{\A, h \otimes g}_{\HS(\calG, \calH)} = \ip{\A^*, g \otimes h}_{\HS(\calH, \calG)}.
\end{align*}
Finally, we will often compute Hilbert-Schmidt norms using assumptions on the singular decays of the operator in question.

\begin{lemma}\label{lem:singular_decay}
    Let $\M: \calG \rightarrow \calH$ be a Hilbert-Schmidt operator with singular values $\br{s_i}_{i \in I}$ (\Cref{thm:svd}). Assume that $I = \mathbb{N}$ and that there exist constants $c, C, \gamma > 0$ such that $ci^{-\gamma} \leq s_i \leq Ci^{-\gamma}$ for all $i \in \mathbb{N}$. Then, $\gamma > 1/2$, and it holds that
    \begin{align*}
        \frac{c^2}{2\gamma - 1} \leq \norm{\M}_{\HS(\calG, \calH)}^2 \leq \frac{2\gamma C^2}{2\gamma - 1}.
    \end{align*}
\end{lemma}
\begin{proof}
    The requirement that $\gamma > 1/2$ follows from the square summability of $\br{s_i}_{i=1}^\infty$ and the bound $s_i \geq ci^{-\gamma}$.
    For the upper bound, write
    \begin{align*}
        \sum_{i=1}^\infty s_i^2\leq C^2 \sum_{i=1}^\infty i^{-2\gamma} = C^2 \sum_{i=1}^\infty \int_{i-1}^i \ceil{x}^{-2\gamma}\d x \leq C^2 \p{1 + \int_1^\infty x^{-2\gamma} \d x} = \frac{2\gamma C^2}{2\gamma - 1}.
    \end{align*}
    For the lower bound, write
    \begin{align*}
        \sum_{i=1}^\infty s_i^2 \geq c^2 \sum_{i=1}^\infty i^{-2\gamma} = c^2 \sum_{i=1}^\infty \int_i^{i+1} \floor{x}^{-2\gamma}\d x \geq c^2 \int_1^\infty x^{-2\gamma} \d x  = \frac{c^2}{2\gamma - 1},
    \end{align*}
    the result as desired.
\end{proof}

\subsection{The Conditional Mean Operator}\label{sec:a:background:buja}
This section contains key properties of the conditional mean operator $\M_{Z|X}$ and the information density $\Rsans$ from \Cref{sec:framework}, based on the foundations of \Cref{sec:a:background:l2} and \Cref{sec:a:background:compact}. As we shall show, owing to a particular \emph{Lancaster decomposition} (\Cref{prop:lancaster}), both operators enjoy convenient spectral representations and relate to a measure of dependence---the mean-squared contingency.  

Recall the probability space $(\Omega, \msc{F}, \prob)$. Consider Borel measurable spaces $(\X, \calB(\X))$ and $(\Z, \calB(\Z))$, and a random variable $(X, Z): \Omega \rightarrow \X \times \Z$. We denote by $Q_{X, Z}$ the law of $(X, Z)$, i.e.~$Q_{X, Z}(B) = \prob\p{(X, Z)^{-1}(B)}$ for every $B \in \calB(\X \times \Z$). Note that by \citet[Corollary 27.24]{schilling2017measures}, the Hilbert spaces $\ltwo(Q_X)$, $\ltwo(Q_Z)$, $\ltwo(Q_{X, Z})$, and $\ltwo(Q_X \otimes Q_Z)$ are separable, a fact we will maintain in this section. We use the notation $[\cdot]_X$ and $[\cdot]_Z$ to index equivalence classes in $\ltwo(Q_X)$ and $\ltwo(Q_Z)$, respectively. In other words, for a measurable function $h: \X \rightarrow \R$, we will write $[h]_X \in \ltwo(Q_X)$ to indicate that $\int_{\X} h^2(\x) \d Q_X(\x) < + \infty$. Recall the conditional mean function introduced in \Cref{def:conditional_expectation}. We define the conditional mean operator
\begin{align}
    \M_{Z|X} &: \ltwo(Q_Z) \rightarrow \ltwo(Q_X) \notag\\
    \M_{Z|X} [g]_Z &= [\E{Q_{X, Z}}{g(Z)|X}(\cdot)]_X. \label{eq:buja:cmo}
\end{align}
The specific function $g \in [g]_Z$ chosen for the output conditional expectation is not relevant, as all choices will result in the same equivalence class. We define $\M_{X|Z}$ as the analogous operator for the conditional mean of $h(X)$ given $Z$ for $[h]_X \in \ltwo(Q_X)$.

\myparagraph{Spectral Representation}
In the case that $\M_{Z|X}$ is compact, the conditional mean operator admits a singular value decomposition, which will be instrumental in obtaining several important properties.

\begin{proposition}[Singular Value Decomposition of the Conditional Mean Operator]\label{prop:svd}
    Let $\M_{Z|X}: \ltwo(Q_Z) \rightarrow \ltwo(Q_X)$ be compact. There exists a countable orthonormal basis $\br{\alpha_j}_{j \in J}$ of $\ltwo(Q_X)$, a countable orthonormal basis $\br{\beta_k}_{k \in K}$ of $\ltwo(Q_Z)$, and a countable sequence of positive real numbers $\br{\sigma_i}_{i \in I}$ satisfying $\sigma_i \searrow 0$, $I \sse J \cap K$, and the following statements in addition:
    \begin{itemize}
        \item We may take $\sigma_1 = 1$, $\ones_\X \in \alpha_1$, and $\ones_\Z \in \beta_1$, where $\ones_\X$ is identically $1$ on $\X$ and $\ones_\Z$ is identically $1$ on $\Z$.
        \item For all $[g]_Z \in \ltwo(Q_Z)$ and $[h]_X \in \ltwo(Q_X)$, we have that
        \begin{align}
            \M_{Z|X} [g]_Z = \sum_{i \in I} \sigma_i \ip{[g]_Z, \beta_i}_{\ltwo(Q_Z)} \alpha_i \text{ and }  \M_{X|Z} [h]_X = \sum_{i \in I} \sigma_i \ip{[h]_X, \alpha_i}_{\ltwo(Q_X)} \beta_i.\label{eq:operators:svd_cond_mean}
        \end{align}
    \end{itemize}
\end{proposition}
\begin{proof}
    Beyond the direct application of \Cref{thm:svd}, we must prove the statement regarding $(\sigma_1, \alpha_1, \beta_1)$ and that $\M_{Z|X}^* = \M_{X|Z}$ (which relates~\eqref{eq:operators:svd} to~\eqref{eq:operators:svd_cond_mean}) to achieve the desired result. For the first, we appeal to the variational representation of the first singular value $\sigma_1$ \citep[Section IV.1, Eq. (2)]{gohberg1990classes}, which states that
    \begin{align}
        \sigma_1 = \sup\br{\norm{\M_{Z|X} g}_{\ltwo(Q_X)}: g \in \ltwo(Q_Z), \norm{g}_{\ltwo(Q_Z)} = 1}.\label{eq:cond_mean:max_singular_value}
    \end{align}
    We will show that $\beta_1 = [\ones_\Z]_Z$ achieves the supremum. Then, it will hold that $\sigma_1 = 1$ and $\alpha_1 = \M_{Z|X} \beta_1 = [\ones_\X]_X$, because any version of the conditional expectation $\E{Q_{X, Z}}{1|X}$ is $Q_X$-almost surely equal to $1$. Consider any $[g]_Z \in \ltwo(Q_Z)$ satisfying $ \norm{g}_{\ltwo(Q_Z)} = 1$. Then, we apply Jensen's inequality and the tower property (\Cref{lem:tower}) to achieve
    \begin{align*}
        \norm{\M_{Z|X} [g]_Z}_{\ltwo(Q_X)}^2 &= \int_\X \p{\E{Q_{X, Z}}{g(Z)|X}(\x)}^2 \d Q_X(\x)\\ 
        &\leq \int_\X \E{Q_{X, Z}}{g^2(Z)|X}(\x) \d Q_X(\x)\\ 
        &= \E{Q_Z}{g^2(Z)} = \norm{g}_{\ltwo(Q_Z)}^2 = 1.
    \end{align*}
    Setting $g(\z)  = \ones_{\Z}(\z) \equiv 1$ achieves the upper bound, hence also achieving the supremum in~\eqref{eq:cond_mean:max_singular_value}. Next, to prove that $\M_{Z|X}^* = \M_{X|Z}$, we similarly consider $[h]_X \in \ltwo(Q_X)$ and write
    \begin{align*}
        \ip{[h]_X, \M_{Z|X} [g]_Z}_{\ltwo(Q_X)} &= \E{Q_{X}}{h(X) \E{Q_{X, Z}}{g(Z)|X}}\\
        &= \E{Q_{X, Z}}{h(X) g(Z)}\\
        &= \E{Q_{Z}}{\E{Q_{X, Z}}{h(X)|Z} g(Z)}\\
        &= \ip{\M_{X|Z} [h]_X,  [g]_Z}_{\ltwo(Q_Z)},
    \end{align*}
    which satisfies the adjoint relationship and completes the proof.
\end{proof}

\myparagraph{Lancaster Decomposition}
In the remaining proofs of this section, we do not differentiate an equivalence class in an $\ltwo$-space with its component functions, as the distinction will be clear from context. First, using the orthonormal bases defined in \Cref{prop:svd}, we may form a convenient orthonormal basis of $\ltwo(Q_X \otimes Q_Z)$.
\begin{lemma}\label{lem:bivariate:onb}
    The collection $\br{\alpha_j \beta_k}_{j \in J, k \in K}$ from \Cref{prop:svd}, where $\br{\alpha_j}_{j \in J}$ is a countable orthonormal basis of $\ltwo(Q_X)$ and $\br{\beta_k}_{k \in K}$ is a countable orthonormal basis of $\ltwo(Q_Z)$, forms an orthonormal basis of $\ltwo(Q_X \otimes Q_Z)$.
\end{lemma}
\begin{proof}
    We first show that $\br{\alpha_j \beta_k}_{j \in J, k \in K}$ is an orthonormal system. For any indices $i, i' \in I$ and $j, j' \in J$, it holds via independence that
    \begin{align*}
        \ip{\alpha_j \beta_j, \alpha_{j'} \beta_{k'}}_{\ltwo(Q_X \otimes Q_Z)} &= \ip{\alpha_j, \alpha_{j'}}_{\ltwo(Q_X)} \ip{\beta_k, \beta_{k'}}_{\ltwo(Q_Z)}\\
        &= \begin{cases}
            1 & \text{ if } j = j' \text{ and } k = k'\\
            0 & \text{ otherwise}
        \end{cases}.
    \end{align*}
    To establish that this orthonormal system is now complete, we use the first equivalent condition in \Cref{def:orthonormal_basis}. Consider $s \in \ltwo(Q_X \otimes Q_Z)$ such that $\ip{s, \alpha_j \beta_k}_{\ltwo(Q_X \otimes Q_Z)} = 0$ for all $j \in J$ and $k \in K$. Then, via Fubini's theorem \citep[Corollary 14.9]{schilling2017measures}, it holds that
    \begin{align*}
        0 &= \int_\Z \underbrace{\p{\int_\X  s(\x, \z) \alpha_j(\x)  \d Q_X(\x)}}_{g_j(\z)} \beta_k(\z) \d Q_Z(\z) = \ip{g_j, \beta_k}_{\ltwo(Q_Z)}.
    \end{align*}
    Because $\br{\beta_k}_{k \in K}$ forms an ONB, it holds that the equivalence class of $g_j$ is the zero element in $\ltwo(Q_Z)$, or in other words, $g_j(\z) = 0$ for $Q_Z$-almost all $\z \in \Z$. Due to the fact that $J$ is countable, we have that 
    \begin{align*}
        \Z_1 := \bigcap_{j \in J} \br{\z \in \Z: g_j(\z) = 0} = \br{\z \in \Z: g_j(\z) = 0 \ \forall j \in J}
    \end{align*}
    is a probability one set under $Q_Z$. Because $\br{\alpha_j}_{j \in J}$ is an ONB of $\ltwo(Q_X)$, it also holds that
    \begin{align*}
        \br{\z \in \Z: g_j(\z) = 0 \ \forall j \in J} \sse \Z_1' = \br{\z \in \Z: s(\x, \z) = 0 \text{ for $Q_X$-almost all } \x \in \X},
    \end{align*}
    indicating that the right-hand side is also a probability one set under $Q_Z$.
    Then, applying again the iterated integral,
    \begin{align*}
        \int_{\X \times \Z} s^2(\x, \z) \d (Q_X \otimes Q_Z)(\x, \z) &= \int_{\Z} \p{\int_{\X} s^2(\x, \z) \d Q_X(\x)} \d Q_Z(\z)\\
        &= \int_{\Z_1'} \p{\int_{\X} s^2(\x, \z) \d Q_X(\x)} \d Q_Z(\z)\\
        &= 0,
    \end{align*}
    indicating the $s(\x, \z) = 0$ for $(Q_X \otimes Q_Z)$-almost all $(\x, \z) \in \X \times \Z$. This completes the proof.
\end{proof}

This basis allows us to relate the conditional mean operator to a particular Radon-Nikodym derivative. As a result, both can be used to measure the dependence between $X$ and $Z$ \citep{lancaster1958thestructure}.

\begin{proposition}[Lancaster Decomposition]\label{prop:lancaster}
    Assume that $Q_{X, Z} \ll Q_X \otimes Q_Z$, in which case there exists a Radon-Nikodym derivative $\Rsans = \frac{\d Q_{X, Z}}{\d (Q_X \otimes Q_Z)}$. Then, the following identity holds:
    \begin{align}
        \Rsans = \sum_{i \in I} \sigma_i \alpha_i \beta_i.\label{eq:bivariate:lancaster}
    \end{align}
    In particular, the operator $\M_{Z|X}$ is Hilbert-Schmidt if and only if $\Rsans \in \ltwo(Q_X \otimes Q_Z)$, with the equality
    \begin{align*}
        \norm{\M_{Z|X}}_{\HS(\ltwo(Q_Z), \ltwo(Q_X))}^2 = \norm{\Rsans}_{\ltwo(Q_X \otimes Q_Z)}^2 = \sum_{i \in I} \sigma_i^2.
    \end{align*}
\end{proposition}
\begin{proof}
    Using \Cref{lem:bivariate:onb}, we represent $\Rsans$ on the ONB $\br{\alpha_j \beta_k}_{j \in J, k \in K}$. For any $j \in J$ and $k \in K$, use the definition of the Radon-Nikodym derivative \citep[Theorem 20.2]{schilling2017measures} to write
    \begin{align*}
        \ip{\Rsans, \alpha_j \beta_k}_{\ltwo(Q_X \otimes Q_Z)} &= \E{Q_X \otimes Q_Z}{\Rsans(X, Z) \alpha_j(X) \beta_k(Z)}\\
        &= \E{Q_{X, Z}}{\alpha_j(X) \beta_k(Z)}\\
        &= \E{Q_{X}}{\alpha_j(X) \E{Q_{X, Z}}{\beta_k(Z)|X}}\\
        &= \ip{\alpha_j, \M_{Z|X}\beta_k}_{\ltwo(Q_X)}\\
        &= \begin{cases}
            \sigma_i &\text{ if } j = k = i \in I\\
            0 & \text{ otherwise}
        \end{cases},
    \end{align*}
    where we recall $I$ as the set indexing the non-zero singular values of $\M_{Z|X}$ (see \Cref{prop:svd}). This proves~\eqref{eq:bivariate:lancaster}, the first claim. For the second claim, we use the orthonormality of $\br{\alpha_j \beta_k}_{j \in J, k \in K}$ in $\ltwo(Q_X \otimes Q_Z)$, so that
    \begin{align*}
        \norm{\Rsans}_{\ltwo(Q_X \otimes Q_Z)}^2 = \sum_{i \in I} \sigma_i^2 = \norm{\M_{Z|X}}_{\HS(\ltwo(Q_Z), \ltwo(Q_X))},
    \end{align*}
    so that square-summability of $\br{\sigma_i}_{i \in I}$ implies finiteness and equality of the left-hand and right-hand sides above.
\end{proof}

The formulas in \Cref{sec:framework} simply equated $I = \mathbb{N} = \br{1, 2, \ldots}$ for ease of presentation. For completeness, the $\varepsilon_d$ term in~\eqref{eq:lancaster} represents the tail of~\eqref{eq:bivariate:lancaster}, i.e.,
\begin{align*}
    \varepsilon_d = \sum_{i=d+1}^\infty \sigma_i \alpha_i \beta_i,
\end{align*}
which vanishes as $d \rightarrow \infty$ because $\sigma_i \rightarrow 0$ and $\alpha_i \beta_i$ is unit-norm in $\ltwo(Q_X \otimes Q_Z)$.

The Radon-Nikodym derivative $\Rsans = \frac{\d Q_{X, Z}}{\d (Q_X \otimes Q_Z)}$ is also useful for converting conditional expectation computations into marginal expectation computations. In this sense, we may say that $\Rsans$ acts as a kernel for an integral operator representation of $\M_{Z|X}$, where the integral is taken with respect to $Q_Z$. The following identity is referenced by \citet[Section 3]{buja1990remarks} and \citet[Lemma 1, Eq.~(14)]{dytso2023meta}. We provide a self-contained proof below.
\begin{lemma}\label{lem:info_density}
    Adopt the setting of \Cref{prop:lancaster}. Then, for all $g \in \ltwo(Q_Z)$ and $h \in \ltwo(Q_X)$, it holds that
    \begin{align*}
        \E{Q_{X, Z}}{g(Z)|X}(\x) &= \E{Q_{Z}}{g(Z) \Rsans(\x, Z)} \text{ for $Q_X$-almost all } \x \in \X,\\
        \E{Q_{X, Z}}{h(X)|Z}(\z) &= \E{Q_{X}}{h(X) \Rsans(X, \z)} \text{ for $Q_Z$-almost all } \z \in \Z.
    \end{align*}
\end{lemma}
\begin{proof}
    We prove the first identity, whereas the second follows by a symmetric argument.
    To confirm that the two functions are equal almost surely, it is sufficient to prove that for any measurable set $A \in \sigma(X)$ (the $\sigma$-algebra generated by $X$), the relation 
    \begin{align}
        \int_A \E{Q_{X, Z}}{g(Z)|X}(\x) \d Q_X(\x) = \int_A \E{Q_{Z}}{g(Z) \Rsans(\x, Z)} \d Q_X(\x). \label{eq:zsp:almost_surely_equal}
    \end{align}
    By the definition of conditional expectation, we have that
    \begin{align*}
        \int_A \E{Q_{X, Z}}{g(Z)|X}(\x) \d Q_X(\x) &= \int_{\X} \E{Q_{X, Z}}{g(Z)|X}(\x) \ind_{A}(\x) \d Q_X(\x)\\
        &= \E{Q_{X, Z}}{g(Z)\ind_{A}(X)}\\
        &= \E{\red{Q_X \otimes Q_Z}}{g(Z)\ind_{A}(X)\red{\Rsans(X, Z)}},
    \end{align*}
    where the last step follows from the Radon-Nikodym theorem \citep[Theorem 20.2]{schilling2017measures}. Next, we compute the expectation, taken under the product measure, using Fubini's theorem \citep[Corollary 14.9]{schilling2017measures}. That is,
    \begin{align*}
        \int_A \E{Q_{X, Z}}{g(Z)|X}(\x) \d Q_X(\x) &= \E{Q_X \otimes Q_Z}{g(Z)\ind_{A}(X)\Rsans(X, Z)}\\
        &= \int_{A} \p{\int_{\Z} g(\z) \Rsans(\x, \z)\d Q_Z(\z)} \d Q_X(\x)\\
        &= \int_{A} \E{Q_Z}{g(Z) \Rsans(\x, Z)} \d Q_X(\x).
    \end{align*}
    This achieves~\eqref{eq:zsp:almost_surely_equal} and completes the proof.
\end{proof}

While \Cref{lem:info_density} applies for a general function $g$, the function $g_\rho$ from~\eqref{eq:g_rho} is itself a conditional mean. This can be leveraged to produce yet another identity, which acts as a technical lemma for the proof of \Cref{thm:complexity_rn_derivative}.
\begin{lemma}\label{lem:info_density2}
    Recall that $g_\rho(\z) := \E{\rho_{Y, Z}}{\fstar(Y)|Z}(\z)$ for $\fstar \in \ltwo(P_Y)$. Assume in addition that $\fstar \in \ltwo(\rho_Y)$. Then,
    \begin{align*}
        \eta_\rho(\x) = \E{\rho_{Y, Z}}{\fstar(Y) \Rsans(\x, Z)} + \int_{\Z} g_\rho(\z)\Rsans(\x, \z) \p{\d Q_Z(\z) - \d \rho_Z(\z)}.
    \end{align*}
    for $Q_X$ almost all $\x \in \X$.
\end{lemma}
\begin{proof}
    By \Cref{lem:info_density}, we already have that for $Q_X$-almost all $\x \in \X$, the identity
    \begin{align*}
        \eta_\rho(\x) &= \E{Q_Z}{g_\rho(Z)\Rsans(\x, Z)} \\
        &= \E{\rho_Z}{g_\rho(Z)\Rsans(\x, Z)}  + \E{Q_Z}{g_\rho(Z)\Rsans(\x, Z)} - \E{\rho_Z}{g_\rho(Z)\Rsans(\x, Z)} \\
        &= \E{\rho_Z}{g_\rho(Z)\Rsans(\x, Z)}  + \int_{\Z} g_\rho(\z)\Rsans(\x, \z) \p{\d Q_Z(\z) - \d \rho_Z(\z)}.
    \end{align*}
    Now, unpacking the first term on the right-hand side above, we recognize that for fixed $\x \in \X$, the random variable $\Rsans(\x, Z)$ is $\sigma(Z)$-measurable, so via the properties of conditional expectation \citep[Theorem 27.11 (vii)]{schilling2017measures} in $\lone(\rho_Z)$, we may write
    \begin{align*}
        \E{\rho_Z}{g_\rho(Z)\Rsans(\x, Z)} &= \E{\rho_Z}{\E{\rho_{Y, Z}}{\fstar(Y)|Z}\Rsans(\x, Z)} = \E{\rho_Z}{\E{\rho_{Y, Z}}{\fstar(Y) \Rsans(\x, Z)|Z}}.
    \end{align*}
    Using the expression above and the tower property of the conditional expectation (\Cref{lem:tower}), we write
    \begin{align*}
        \E{\rho_Z}{g_\rho(Z)\Rsans(\x, Z)} &= \E{\rho_Z}{\E{\rho_{Y, Z}}{\fstar(Y) \Rsans(\x, Z)|Z}} = \E{\rho_{Y, Z}}{\fstar(Y) \Rsans(\x, Z)},
    \end{align*}
    completing the proof.
\end{proof}

\myparagraph{Mean Square Contingency}
Both singular value decomposition from \Cref{prop:svd} and the Radon-Nikodym derivative $\Rsans$ from \Cref{prop:lancaster} can be used to calculate a dependence measure between $X$ and $Z$ \citep{buja1990remarks}. This dependence measure arises in nonlinear canonical correlation analysis and alternating conditional expectations~\cite{breiman1985estimating,bickel1993efficient}. 

\begin{definition}[Mean Square Contingency]\label{def:msc}
    Assume that $\M_{Z|X}$ is Hilbert-Schmidt. Assume that $I = \mathbb{N}$ (\Cref{prop:svd}), where we may append zeros if $I$ is finite. Recalling that $\sigma_1 = 1$, define the \emph{mean square contingency} $I(X; Z)$ as any of the expressions
    \begin{align*}
        I(X; Z) := \norm{\M_{Z|X}}_{\HS(\ltwo(Q_Z), \ltwo(Q_X))}^2 - 1 = \sum_{i=2}^\infty \sigma_i^2.
    \end{align*}
\end{definition}
Our definition of the mean square contingency is in fact the square of the quantity that was originally introduced as such by \citet{renyi1959onmeasures}, which is shown below.
\begin{definition}\label{def:msc2}
    Assume that $Q_{X,Z} \ll Q_X \otimes Q_Z$, so that $\Rsans$ exists, and that $\Rsans \in \ltwo(Q_X \otimes Q_Z)$. Define the \emph{R\'enyi mean squared contingency} \citet[Eq. (13)]{renyi1959onmeasures} as
    \begin{align*}
        I_{\text{R\'enyi}}(X; Z) :=  \norm{\Rsans - 1}_{\ltwo(Q_X \otimes Q_Z)} = \sqrt{\int_{\X \times \Z} \p{\Rsans(\x, \z) - 1}^2 \d (Q_X \otimes Q_Z)(\x, \z)}.
    \end{align*}
    If $Q_{X,Z}$ is absolutely continuous with respect to a measure $\nu$ on $\X \times \Z$, with joint density $q_{X, Z}$ and marginal densities $(q_X, q_Z)$, we have that
    \begin{align*}
        I_{\text{R\'enyi}}(X; Z) = \sqrt{\int_{\X \times \Z} \p{\Rsans(\x, \z) - 1}^2 q_X(\x) q_Z(\z) \d \nu(\x, \z)}.
    \end{align*}
    Written in the form above, $I_{\text{R\'enyi}}(X; Z)$ may also be called the \emph{$\chi^2$-functional} \citep{buja1990remarks}.
\end{definition}
We apply $I(X; Z) = I_{\text{R\'enyi}}(X; Z)^{\red{2}}$ to achieve the sequence of identities following~\eqref{eq:msc}.
Using the singular decay computations from \Cref{lem:singular_decay} with $c = C = 1$, we have that if $\sigma_i = i^{-\gamma}$, then
\begin{align}
    \frac{1}{2\gamma - 1} - 1 \leq I(X; Z) \leq \frac{1}{2\gamma - 1} \iff \frac{1}{2}\frac{I(X; Z) + 2}{I(X; Z) + 1}\leq \gamma \leq \frac{1}{2}\frac{I(X; Z) + 1}{I(X; Z)}.\label{eq:background:buja:decay}
\end{align}
For simplicity, we will use the upper bounds to describe the order of the quantities, that is,
\begin{align}
    I(X; Z) \sim \frac{1}{2\gamma - 1} \iff  \gamma \sim \frac{I(X; Z) + 1}{2 I(X; Z)}.
\end{align}
We employ this relation in the sample complexity calculations in \Cref{sec:a:complexity}.

\subsection{Reproducing Kernel Hilbert Spaces}\label{sec:a:background:rkhs}
In this section, we review facts about the interplay between reproducing kernel Hilbert spaces (RKHSs), the $\ltwo$-spaces defined in \Cref{sec:a:background:l2}, the Hilbert-Schmidt spaces from \Cref{sec:a:background:compact}, and Bochner spaces (introduced below). The goal is to provide the necessary background in order to understand the results regarding kernel-based estimation methods that are used in other parts of the manuscript. The analyses in \Cref{sec:a:complexity} rely on being able to decompose some target function (related to the dependence between $X$ and $Z$) so that it may be estimated in multiple ways. One method involves estimating the Radon-Nikodym derivative $\Rsans$ introduced in \Cref{prop:lancaster}. The second technique relies on vector-valued regression, with a target function denoted by $F_\star$. Most of the setup below is in service of the vector-valued regression estimation portion.

We maintain the Borel spaces $(\X, \calB(\X))$ and $(\Z, \calB(\Z))$ from \Cref{sec:a:background:buja}, with the topological assumption that $\X$ and $\Z$ are second countable, locally compact, and Hausdorff. In addition, $\calH$ and $\calG$ each denote a separable reproducing kernel Hilbert space (RKHS) containing real-valued functions of $\X$ and real-valued functions of $\Z$, respectively. We let $\phi: \X \rightarrow \calH$ and $\psi: \Z \rightarrow \calG$ be the canonical feature maps and let $k:\X \times \X \rightarrow \R$ and $l: \Z \times \Z \rightarrow \R$ be the reproducing kernels for $\calH$ and $\calG$. The boundedness assumptions  $\sup_{\x, \x' \in \X}k(\x, \x') \leq \kmax <\infty$ and  $\sup_{\z, \z' \in \Z}l(\z, \z') \leq \lmax <\infty$ are maintained throughout the paper. 

\myparagraph{Bochner Space}
We will adopt the equivalence class notation first introduced in \Cref{sec:a:background:l2}, with respect to probability measures $Q_X$ and $Q_Z$. That is, for any two real-valued measurable functions $f: \X \rightarrow \R$ and $h: \X \rightarrow \R$, we say that $f \sim_X h$ if
\begin{align*}
    Q_X\p{\br{\x \in \X: f(\x) \neq h(\x)}} = 0.
\end{align*}
The notation $[f]_X$ denotes an equivalence class with respect to the equivalence relation $\sim_X$, with representative $f$. We say that $[f]_X \in \ltwo(Q_X)$ if
\begin{align*}
    \int_\X h^2(\x) \d Q_X(\x) < \infty \text{ for some, or equivalently all } h \in [f]_X.
\end{align*}
We define $\sim_Z$, $[\cdot]_Z$, and $\ltwo(Q_Z)$ similarly. 
We introduce a similar construction to $\ltwo(Q_X)$ for \emph{vector-valued} functions, i.e., those whose outputs lie in a Hilbert space. For measurable functions $F: \X \rightarrow \calG$ and $H: \X \rightarrow \calG$, we will define the equivalence relation $F \sim_X H$ via $Q_X\p{\br{\x \in \X: F(\x) \neq H(\x)}} = 0$, and corresponding equivalence classes will be denoted $[F]_X$. We define the \emph{Bochner space} $\ltwo(Q_X; \calG)$ via $[F]_X \in \ltwo(Q_X; \calG)$ if
\begin{align*}
    \int_\X \norm{H(\x)}_\calG^2 \d Q_X(\x) < \infty \text{ for some, or equivalently all } H \in [F]_X.
\end{align*}
Analogous to $\ltwo(Q_X)$, this is a set of equivalence classes of vector-valued functions. Recall from \Cref{sec:a:background:compact} and \Cref{sec:a:background:buja} that we use $\HS(\mc{U}, \mc{V})$ to denote the space of Hilbert-Schmidt operators mapping from a Hilbert space $\mc{U}$ to another Hilbert space $\mc{V}$. The following result allows us to relate elements of the Bochner space $ \ltwo(Q_X;\calG)$ to elements of a space of Hilbert-Schmidt operators $\HS(\ltwo(Q_X), \calG)$. For $[f]_X \in \ltwo(Q_X)$ and $g \in \calG$, the notation $f(\cdot) g$ refers to the function mapping $\x \in \X$ to $f(\x)g \in \calG$.
\begin{theorem}{\citep[Theorem 12.6.1]{aubin2000applied}}\label{thm:aubin:isomorphism}
    There exists a function $\Phi: \HS(\ltwo(Q_X), \calG) \rightarrow \ltwo(Q_X;\calG)$ that is a bijective linear transformation satisfying
    \begin{align*}
        \norm{\C}_{\HS(\ltwo(Q_X), \calG)} = \norm{\Phi(\C)}_{\ltwo(Q_X;\calG)} \text{ for all } \C \in \HS(\ltwo(Q_X), \calG),
    \end{align*}
    and for every $[f]_X \in \ltwo(Q_X)$ and $g \in \calG$, associates
    \begin{align}
        \Phi(g \otimes [f]_X) = [f(\cdot) g]_X \iff g \otimes [f]_X = \Phi^{-1}([f(\cdot) g]_X)\label{eq:main:rankone}
    \end{align}
    for the rank-one operator $g \otimes [f]_X \in \HS(\ltwo(Q_X), \calG)$.
\end{theorem}
Based on the definition of the Hilbert-Schmidt norm in~\eqref{eq:compact:hsnorm} (\Cref{sec:a:background:compact}, \Cref{thm:aubin:isomorphism} will make computation of $\ltwo(Q_X;\calG)$-norms more convenient by relating them to $\HS(\ltwo(Q_X), \calG)$-norms. The following technical lemma can be used to simplify computations regarding the inverse of this isomorphism.
\begin{lemma}\label{lem:main:estimation:technical}
    Let $(g_j)_{j \in J}$ be any countable orthonormal basis of $\calG$, and let $\C = \Phi^{-1}([F]_X)$ for some $F: \X \rightarrow \calG$ such that $[F]_X \in \ltwo(Q_X; \calG)$. Define the functions $(f_j)_{j \in J}$ via $f_j(\x) := \ip{g_j, F(\x)}_\calG$. Then, $[f_j]_X \in \ltwo(Q_X)$ for all $j \in J$, and we have the identity
    \begin{align*}
        \C &= \sum_{j \in J} g_j \otimes [f_j]_X,
    \end{align*}
    where the convergence is interpreted in terms of $\HS(\ltwo(Q_X), \calG)$.
\end{lemma}
\begin{proof}
    First, consider the case in which we can write the equivalence class of $F$ in $\ltwo(Q_X, \calG)$ in the form
    \begin{align}
        [F]_X = \sum_{j \in J} [f_j(\cdot) g_j]_X, \label{eq:main:hsnorm3}
    \end{align}
    for some sequence of functions $f_1, f_2, \ldots \in \ltwo(Q_X)$. Then, because $\Phi^{-1}$ is a linear isometry, it is a bounded (hence continuous) operator with respect to the norm on $\ltwo(Q_X, \calG)$. This implies via continuity
    \begin{align*}
        \C = \Phi^{-1}\p{[F]_X} &= \Phi^{-1}\p{\textstyle\sum_{j \in J} [f_j(\cdot) g_j]_X} = \sum_{j \in J}\Phi^{-1}\p{[f_j(\cdot) g_j]_X} = \sum_{j \in J} g_j \otimes [f_j]_X,
    \end{align*}
    where the last step follows because $\Phi^{-1}$ satisfies the relation~\eqref{eq:main:rankone}. To achieve the identity~\eqref{eq:main:hsnorm3}, we fix any $\x \in \X$, we expand $F(\x) \in \calG$ onto the basis $(g_j)_{j \in J}$ to write
    \begin{align*}
        F(\x) = \sum_{j \in J} \underbrace{\ip{g_j, F(\x)}_\calG}_{f_j(\x)} g_j.
    \end{align*}
    To pass this pointwise equality to~\eqref{eq:main:hsnorm3}, consider any sequence of $\calG$-valued functions $(H_j)_{j \in J}$ such that $H_j(\x) = f_j(\x) g_j$ for all $\x \in \X_j \sse \X$, where $Q_X(\X_j) = 1$. Similarly, consider $H_0: \X \rightarrow \calG$ such that $H_0(\x) = F(\x)$ for all $\x \in \X_0$ with $Q_X(\X_0) = 1$. Thus, we have that
    \begin{align*}
        \textstyle H_0(\x) = \sum_{j \in J} H_j(\x) \text{ for all } \x \in \X_0 \cap \p{\bigcap_{j \in J} \X_j},
    \end{align*}
    and because $J$ is countable, this implies that $H_0(\x) = \sum_{j \in J} H_j(\x)$ for $Q_X$-almost all $\x \in \X$, granting~\eqref{eq:main:hsnorm3}.
    It remains to be shown that $[f_j]_X \in \ltwo(Q_X)$. This follows by the Bochner-square integrability of $F$, as
    \begin{align*}
        \int_\X f_j^2(\x) \d Q_X(\x) = \int_\X \ip{g_j, F(\x)}_\calG^2 \d Q_X(\x) \leq \norm{g_j}_\calG^2 \int_\X \norm{F(\x)}_\calG^2 \d Q_X(\x) < \infty,
    \end{align*}
    which completes the proof.
\end{proof}
In the sequel, we will define a statistical learning problem in which the target function $F_\star$ is an element of $\ltwo(Q_X;\calG)$. For the kernel-based estimation approach, the estimation function $\Fhat \equiv \Fhat_{\lambda}$ (where $\lambda$ is a to-be-specified regularization parameter) will live in a particular vector-valued RKHS that will be isometrically isomorphic to $\HS(\calH, \calG)$. Using \Cref{thm:aubin:isomorphism}, $F_\star$ will be associated to en element $\C_\star \in \HS(\ltwo(Q_X), \calG)$ via an isometric isomorphism. We introduce the concept of embeddings and interpolation spaces to describe exactly where $\C_\star$ lies in between $\HS(\calH, \calG)$ and $\HS(\ltwo(Q_X), \calG)$ (via a \emph{source condition}).

\myparagraph{Embedding Operator}
See \Cref{sec:a:background:compact} for a review of the terminology surrounding compact operators.
Consider the \emph{embedding operator} $\Iop_X: \calH \rightarrow \ltwo(Q_X)$, which identifies a function $h \in \calH$ with its equivalence class $[h]_X \in \ltwo(Q_X)$. Under the bounded kernel assumption, we have that $\Iop_X$ is compact, and moreover Hilbert-Schmidt, with norm bounded as $\norm{\Iop_X}_{\HS(\calG, \ltwo(Q_X))} \leq \sqrt{\kmax}$ \citep[Lemma 2.3]{steinwart2012mercers}. We denote its adjoint by $\Sop_X := \Iop_X^*$, and finally, construct the self-adjoint, trace class operator
\begin{align}
    \Top_X := \Iop_X \Sop_X: \ltwo(Q_X) \rightarrow \ltwo(Q_X).\label{eq:background:Tx}
\end{align}
Applying the eigendecomposition \Cref{thm:eigen}, there exists an orthonormal basis of $\Cl(\Range(\Iop_X)) \sse \ltwo(Q_X)$, denoted $([e_{X, i}]_X)_{i \in I}$, and a sequence of positive, non-increasing eigenvalues $(\mu_{X, i})_{i \in I}$ such that
\begin{align}
    \Top_X = \sum_{i \in I} \mu_{X, i} \ip{[e_{X, i}]_X, \cdot}_{\ltwo(Q_X)} [e_{X, i}]_X.\label{eq:main:eigendecomposition}
\end{align}
Note that we have only used the index set $I$ from \Cref{thm:eigen}, as opposed to the larger set $J$ for which we can define an ONB for the entirety of $\ltwo(Q_X)$, not only $\Cl(\Range(\Iop_X))$. Analogous to $\Top_X$, we can also define the uncentered covariance operator
\begin{align*}
    \C_X = \Sop_X \Iop_X: \calH \rightarrow \calH.
\end{align*}
Similar to~\eqref{eq:main:eigendecomposition}, $\C_X$ enjoys an eigendecomposition
\begin{align}
    \C_X = \sum_{i \in I} \mu_{Z, i} \ipsmall{\cdot, \mu_{X, i}^{1/2} e_{X, i}}_{\calG} \mu_{X, i}^{1/2} e_{X, i}.\label{eq:main:covariance}
\end{align}
The equation above implicitly contains another fact, which is that the equivalence classes in~\eqref{eq:main:eigendecomposition} all contain representatives that are in $\calH$. This defines the collection $\br{e_{X, i}}_{i \in I}$, which forms an ONB of $\Null(\Iop_X)^\perp \sse \calH$. Combining~\eqref{eq:main:eigendecomposition} and~\eqref{eq:main:covariance}, the embedding can be described using a singular value decomposition
\begin{align}
    \Iop_X = \sum_{i \in I} \mu_{X, i}^{1/2} \p{[e_{X, i}]_Z \otimes (\mu_{X, i}^{1/2}  e_{X, i})}.\label{eq:main:estimation:basis:H}
\end{align}
Lastly, we define $(\Iop_Z, \Sop_Z, \Top_Z, \Cop_Z)$ as the analogous operators for $\ltwo(Q_Z)$ and $\calG$.

\myparagraph{Interpolation Spaces and the Inclusion Map}
For any $\alpha \geq 0$, we define the operator
\begin{align}
    \Top_X^{\alpha/2} &= \sum_{i \in I} \mu_{X, i}^{\alpha/2} \ip{[e_{X, i}]_X, \cdot}_{\ltwo(Q_X)} [e_{X, i}]_X, \label{eq:background:rkhs:power}\\
    \Dom(\Top_X^{\alpha/2}) &= \br{[f]_X \in \ltwo(Q_X): \sum_{i \in I} \mu_{X, i}^{\alpha/2} \ip{[f]_X, [e_{X, i}]_X}_{\ltwo(Q_X)} < \infty},\notag
\end{align}
which is considered to be well-defined when $\Dom(\Top_X^{\alpha/2}) \neq \varnothing$. Then, we define the \emph{$\alpha$-interpolation space} $[\calH]^\alpha$ via
\begin{align*}
    [\calH]^\alpha = \br{\sum_{i \in I} a_i \mu_{X, i}^{\alpha/2}[e_{X, i}]_X: (a_i)_{i \in I} \in \ell_2(I)} \sse \ltwo(Q_X),
\end{align*}
where $(a_i)_{i \in I} \in \ell_2(I)$ indicates that $\sum_{i \in I} a_i^2 < +\infty$. When $\alpha = 0$, we recover $[\calH]^0 = \Cl(\Range(\Iop_X))$, whereas for $\alpha = 1$, $[\calH]^1$ is isometrically isomorphic to $\Null(\I_X)^\perp \sse \calH$ \citep{fischer2020sobolev}. Thus, for $\alpha \in (0, 1)$, we interpret $[\calH]^\alpha$ as an ``interpolation'' between the well-behaved functions in the RKHS $\calH$ and the elements of $\ltwo(Q_X)$.
Associated to each $[\calH]^\alpha$ is the \emph{inclusion map} $\I_X^{\alpha, \infty}$, which simply views an element $[h]_X \in [\calH]^\alpha$ as an element of $\linf(Q_X)$ (this requires the boundedness of the kernel). Here, $\linf(Q_X)$ denotes equivalence classes of real-valued functions on $\X$ that have a finite essential supremum under $Q_X$.
We write $\I_X^{\alpha, \infty}: [\calH]^\alpha \hookrightarrow \linf(Q_X)$ when the inclusion map is continuous (see \Cref{asm:rkhs:source}).

We use the standard generalization of these notions onto spaces of vector-valued functions \citep{li2024towards, meunier2024optimal}: for any $\beta \geq 0$, we define the \emph{$\beta$-interpolation norm} for $\C \in \HS(\ltwo(Q_X), \calG)$ via the formula
\begin{align}
    \norm{\C}_\beta := \norm{\C \Top_X^{-\beta/2}}_{\HS(\ltwo(Q_X), \calG)} \in [0, +\infty].\label{eq:rkhs:interpolation}
\end{align}
This norm, when finite, will be used to define the source condition of the target function $F_\star$ alluded to in \Cref{sec:theory}, as we may compute $\norm{\C_\star}_\beta$ for $\C_\star := \Phi^{-1}([F_\star]_X)$ (see \Cref{thm:aubin:isomorphism}). While we phrase the condition in terms of the constant $\beta$ above in order to relate it to the kernel methods and inverse problem literature below, we will use the constructions of \Cref{sec:a:background:buja} to phrase the finiteness of~\eqref{eq:rkhs:interpolation} for $\C_\star$ in terms of the mean square contingency (\Cref{def:msc}) in \Cref{sec:a:complexity}.

\myparagraph{Vector-Valued Spectral Regularization Learning}
We may now describe estimation techniques for an $\ltwo(Q_X; \calG)$-valued target function that fall into the category of vector-valued spectral regularization learning. We give only a brief overview in order to state the statistical convergence guarantees; see \citet{meunier2024optimal} for a detailed description, including computational properties of the estimator.
As we prove in \Cref{lem:main:construction}, there exists a function $F_\star: \X \rightarrow \calG$ such that $[F_\star]_X \sse \ltwo(Q_X; \calG)$ and for every $g \in \calG$,
\begin{align}
    \E{Q_{X, Z}}{g(Z)|X}(\x) = \ip{g, F_\star(\x)}_\calG.\label{eq:background:rkhs:target}
\end{align}
For each $\x \in \X$, we also refer to $F_\star(\x)$ as the conditional mean embedding of $Z$ given $X = \x$. Note that for a fixed $g \in \calG$, we do not assume that $\x \mapsto \ip{g, F_\star(\x)}$ is an element of an RKHS $\calH$. This avoids some of the technical challenges raised, for instance, in \citet{klebanov2020arigorous, klebanov2021thelinear}, where this requirement places strong implicit restrictions on the chosen kernel and RKHS. Instead, the mis-specified case is handled using vector-valued interpolation spaces.

Next, using \Cref{thm:aubin:isomorphism}, we associate to $F_\star$ the element $\C_\star = \Phi^{-1}(F_\star) \in \HS(\ltwo(Q_X), \calG)$.
Given independent and identically distributed pre-training data $(X_1, Z_1), \ldots, (X_N, Z_N)$  drawn from $Q_{X, Z}$, define the empirical (uncentered) auto-covariance and cross-covariance operator
\begin{align*}
    \hC_{XX} = \frac{1}{N}\sum_{i=1}^N \phi(X_i) \otimes \phi(X_i) \text{ and } \hC_{ZX} = \frac{1}{N}\sum_{i=1}^N \psi(Z_i) \otimes \phi(X_i).
\end{align*}
Let $f_\lambda: \R_{\geq 0} \rightarrow \R_{\geq 0}$ denote the spectral cutoff function
\begin{align}
    f_\lambda(x) = \begin{cases}
        x^{-1} &\text{ if } x \geq \lambda\\
        0 &\text{ otherwise}
    \end{cases}.\label{eq:background:rkhs:cutoff}
\end{align}
We can interpret $f_\lambda(x)$ as a regularized inverse that behaves in a reasonable manner near $x = 0$. A similar function corresponding to the more familiar Tikhonov regularization is $f_\lambda(x) = (x + \lambda)^{-1}$.
While other options for $f_\lambda$ (i.e.~filter functions) exist owing to the tools of regularization theory \citep{bauer2007onregularization}, the spectral cutoff function will be sufficient for our purposes, as it allows for the simplest statement of the upcoming results.
For a self-adjoint positive semidefinite operator $\C$, we define $f_\lambda(\C)$ as replacing each eigenvalue $\mu_i \geq 0$ of $\C$ with $f_\lambda(\mu_i)$ in the eigendecomposition (see \Cref{thm:eigen}). For regularization parameter $\lambda > 0$, we define the estimator
\begin{align}
    \Fhat_\lambda(\cdot) := \hC_{\lambda}\phi(\cdot) \text{ for } \hC_{\lambda} := \hC_{ZX} f_\lambda(\hC_{XX}): \calH \rightarrow \calG.\label{eq:background:rkhs:estimate}
\end{align}
Now, consider the following assumptions, which include the source condition.
\begin{assumption}{\citep[Assumptions (SRC), (MOM), (EVD), (EMB)]{meunier2024optimal}}\label{asm:rkhs:source}
    \begin{enumerate}
        \item There exist positive constants $\beta > 0$ and $B > 0$ such that $\norm{F_\star}_\beta := \norm{\C_\star}_\beta \leq B$.
        \item For positive constants $\sigma^2, c > 0$ the Bernstein moment condition
        \begin{align*}
            \Ex_{Q_{X, Z}}\sbr{\norm{\psi(Z) - F_\star(X)}_\calG^2|X}(\x) \leq \frac{1}{2}q!\sigma^2c^{q-2} 
        \end{align*}
        is satisfied for $Q_X$-almost all $\x \in \X$ and all $q \geq 2$. 
        \item There exist constants $D > 0$ and $p < 1$ such that
        \begin{align*}
            \mu_{X, i} \leq D i^{-1/p}.
        \end{align*}
        \item For $\alpha \in [p, 1]$, the inclusion map $\I_X^{\alpha, \infty}: [\calH]^\alpha \hookrightarrow \linf(Q_X)$ is bounded, with operator norm $\norm{I^{\alpha, \infty}_X}_{\op} \leq A$.
    \end{enumerate}
\end{assumption}
Note that the first assumption is always satisfied for $\alpha = 1$, due to boundedness of the kernel (as the $[\calH]^1$ norm can be associated to the RKHS norm of an element of $\calH$). We pay particular attention to the constant $\beta$ which defines the aforementioned source condition. In \Cref{sec:a:complexity:conditional_mean}, we translate this condition into one regarding the dependence between $X$ and $Z$, using the tools from \Cref{sec:a:background:buja}. We refer to the case when $\beta \geq 1$ as the \emph{well-specified} case.  We also employ one additional assumption to state the result.
\begin{assumption}[Sub-Gaussian Tail]\label{asm:rkhs:tail}
    There exists a positive constant $\tau > 0$ such that the following holds:
    \begin{align*}
        \P{Q_X}{\norm{F_\star(X)}_\calG > t} \leq 2e^{-\frac{t^2}{2\nu^2}}.
    \end{align*}
\end{assumption}
\Cref{asm:rkhs:tail} is only used to replace a statement of the form ``for $N \geq 1$ sufficiently large'' from \citet[Theorem 4]{meunier2024optimal} with a quantitative condition on $N$. It is used to control the probability that $\norm{F_\star(X_i)}_\calG > t$ for any $i = 1, \ldots, N$ for the choices of $t$ specified in the proof of \citet[Theorem 8]{meunier2024optimal}.

\begin{theorem}{\citep[Theorem 4]{meunier2024optimal}}\label{thm:vvkr}
    Consider a failure probability $\delta \in (0, 1]$, the estimate $\Fhat_{\lambda}$ defined in~\eqref{eq:background:rkhs:estimate}, and the target function $F_\star$ defined in~\eqref{eq:background:rkhs:target}. Under \Cref{asm:rkhs:source} and \Cref{asm:rkhs:tail}, there exists a constant $C > 0$ (independent of $N$ and $\delta$) such that the following statements hold.
    \begin{itemize}
        \item {\bf Case 1: $\beta + p > \alpha$.} If $N^{\p{\frac{1}{2}\p{1 + \frac{p - \alpha}{p + \beta}} + \frac{\alpha - \beta}{2\alpha}}} \geq 2\nu^2\polylog(N/\delta)$ and $\lambda = \Theta(N^{-\frac{1}{\beta+p}})$, then
        \begin{align*}
            \norm{[\Fhat_{\lambda}]_X - F_\star}_{\ltwo(Q_X; \calG)}^2 \leq C\polylog(1/\delta) N^{-\frac{\beta}{\beta + p}}
        \end{align*}
        with probability at least $1 - \delta/4$.
        \item {\bf Case 2: $\beta + p \leq \alpha$.} If $N^{\frac{\alpha-\beta}{\alpha}} \geq 2\nu^2\polylog(N/\delta)$ and $\lambda = \Theta((N/\polylog(N))^{-\frac{1}{\alpha}})$, then
        \begin{align*}
            \norm{[\Fhat_{\lambda}]_X - F_\star}_{\ltwo(Q_X; \calG)}^2 \leq C\polylog(1/\delta) (N/\log^2(N))^{-\frac{\beta}{\alpha}}
        \end{align*}
        with probability at least $1 - \delta/4$.
    \end{itemize}
\end{theorem}
This result is applied in \Cref{sec:a:complexity:conditional_mean} and provides an example of the ``conditional mean'' approach outlined in \Cref{sec:framework} and \Cref{sec:theory}. Regarding the setting of $\lambda$ in \Cref{thm:vvkr}, the argument follows the typical recipe of defining an element $F_\lambda \in \ltwo(Q_X; \calG)$ which represents the population version of the regularized predictor. Let $\norm{\cdot}_\gamma$ denote the $\gamma$-interpolation norm for $\gamma \in [0, 1]$, which is equal to $\norm{\cdot}_{\ltwo(Q_X;\calG)}$ when $\gamma = 0$. Then, the approximation error $\norm{[F_{\lambda}]_X - F_\star}^2_\gamma$ decays according to $\lambda^{\beta - \gamma}$ when using the spectral cutoff regularizer, which reflects the classical analyses of \citet{smale2007learning}. In the well-specified case, the estimation error, $\norm{[\Fhat_{\lambda}]_X - [F_{\lambda}]_X}_\gamma$ decomposes into multiple terms which include irreducible noise terms of order $N^{-1}\lambda^{-\alpha/2}$ and additional approximation terms of order $N^{-1/2}\lambda^{(\beta-\alpha)/2}$. By using $\lambda = \Theta(N^{-\frac{1}{\beta+p}})$ and the condition $\beta + p > \alpha$ from Case 1, the irreducible noise error converges at rate $N^{-1/2}$ whereas the approximation term converges at rate $N^{-\beta/2(\beta + p)}$. Note that these rates will be squared in \Cref{thm:vvkr}. The argument for Case 2 follows similarly.

\myparagraph{Radon-Nikodym Derivative Estimation}
To set the stage for this technique, we describe a function class in which $\Rhat: \X \times \Z \rightarrow \R_{\geq 0}$ will live. Let $\calS$ denote a separable reproducing kernel Hilbert space (RKHS) of real-valued functions on $\X \times \Z$, with canonical feature map $\varphi: \X \times \Z \rightarrow \R$ and reproducing kernel $\kappa: (\X \times \Z) \times (\X \times \Z) \rightarrow \R$. As before, we first assume boundedness of the kernel, i.e., $\sup\br{\kappa(\x, \z, \x', \z'): (\x, \z), (\x', \z') \in \X \times \Z} \leq \kappa_{\max{}}$.

Let us describe the estimation procedure, which relies on a similar spectral regularization technique as the one described for vector-valued regression. Because the Radon-Nikodym derivative being estimated is $\frac{\d Q_{X, Z}}{\d (Q_X \otimes Q_Z)}$, we consider having samples from both distributions available. In particular, we observe $\Np$ paired examples $(X_1, Z_1), \ldots, (X_{\Np}, Z_{\Np}) \sim Q_{X, Z}$ and $\Nu$ unpaired examples $(X'_1, Z'_1), \ldots, (X'_{\Nu}, Z'_{\Nu}) \sim Q_{X} \otimes Q_Z$. Define the uncentered covariance operators
\begin{align}
    \hCopp = \frac{1}{\Np}\sum_{i=1}^{\Np} \varphi(X_i, Z_i) \otimes \varphi(X_i, Z_i), \quad \hCopu = \frac{1}{\Nu}\sum_{i=1}^{\Nu} \varphi(X'_i, Z'_i) \otimes \varphi(X'_i, Z'_i),\label{eq:background:rn:samples}
\end{align}
representing the paired and unpaired examples, respectively. Then, using the spectral cutoff function $f_\lambda$ (see~\eqref{eq:background:rkhs:cutoff}), we define the estimate
\begin{align}
    \Rhat \equiv \Rhat_{\lambda} = f_\lambda(\hCopu) \hCopp \ones, \label{eq:background:rkhs:rn_estimate}
\end{align}
where $\ones(\x, \z) = 1$ for all $(\x, \z) \in \X \times \Z$. Because $f_\lambda$ can be viewed as a regularized inverse, $\Rhat$ can readily be interpreted as the ``ratio'' of the covariance operator of the paired sample over that of the unpaired sample.

To state the assumptions for the analysis, we require an analogous operator to $\Iop_X$ and $\Iop_Z$ introduced earlier in this section. 
We then define the \emph{embedding operator} $\Iop_{X,Z}: \calS \rightarrow \ltwo(Q_X \otimes Q_Z)$, which takes an element $\Ssans \in \calS$ and indexes its equivalence class in $\ltwo(Q_X \otimes Q_Z)$. We will not need to define an explicit notation for the equivalence class for this discussion, but will do so in \Cref{sec:a:complexity:rn_derivative}. Due to \citet[Lemma 2.3]{steinwart2012mercers}, the bounded kernel assumption implies  that $\Iop_{X,Z}$ is Hilbert-Schmidt with norm bounded as $\norm{\Iop_{X,Z}}_{\HS(\calS, \ltwo(Q_X \otimes Q_Z))} < \sqrt{\kappa_{\max{}}}$. 

Recall the powers of operators introduced in~\eqref{eq:background:rkhs:power}. We will use a similar construction for this technique as well.
Define the (compact) adjoint operator $\Sop_{X,Z} = \Iop_{X,Z}^*: \ltwo(Q_X \otimes Q_Z) \rightarrow \calS$. Via \Cref{thm:eigen}, let $(\mu_i)_{i \in I}$ denote the non-zero eigenvalues of the compact, trace class operator $\Sop_{X,Z} \Iop_{X,Z}$, where we consider $I = \mathbb{N}$ for simplicity. Let the degrees of freedom function be defined as
\begin{align*}
    \df(\lambda) := \sum_{i=1}^\infty \frac{\mu_i}{\mu_i + \lambda}.
\end{align*}
Consider the following assumption.
\begin{assumption}{\citep[Eq. (9) and Remark 13]{nguyen2024onregularized}}\label{asm:main:estimation:rn:source}
    There exists an absolute constant $C_{\df}$ and a constant $\alpha > 1$ such that $\df(\lambda) \leq C_{\df} \lambda^{-1/\alpha}$.
    There exists a $\beta \geq 1$, along with an element $\Ssans_{Q_{X, Z}} \in \Null(\I_{X, Z})^\perp$, such that 
    \begin{align*}
         \Rsans = (\Sop_{X,Z} \Iop_{X, Z})^\beta \Ssans_{Q_{X, Z}}.
    \end{align*}
\end{assumption}
The upper bound on $\df(\lambda)$ reflects a polynomial eigendecay of order $\mu_i \sim i^{-\alpha}$ (see \citet[Section 7.6.6]{bach2024learning}).
\Cref{asm:main:estimation:rn:source} is more specific than the one stated in the referenced work, in that we use the specific index function $x \mapsto x^{\beta}$, growing at least linearly. Furthermore, their result may achieve faster convergence rates than the one stated in \Cref{coro:background:rn} using an additional source condition on the feature map $\varphi$. However, our intention is not necessarily to provide convergence rates that are optimal in a particular parameter regime, but ones that are informative with regard to the dependence structure of $Q_{X, Z}$. To this end, we do not incorporate the additional condition.

To state the result, we define $\lambda_\star$ as the solution of
\begin{align*}
    \frac{\df(\lambda)}{\lambda} = \Nu,
\end{align*}
which is guaranteed to exist as $\frac{\df(\lambda)}{\lambda}$ is decreasing from $+\infty$ to $0$ on the interval $(0, +\infty)$. Observe the following.
\begin{theorem}{\citep[Proposition 10 and Lemma 11]{nguyen2024onregularized}}\label{thm:background:rn2}
    Consider a failure probability $\delta \in (0, 1]$ and constant $\beta$ from \Cref{asm:main:estimation:rn:source}.
    Consider the estimate $\Rhat \equiv \Rhat_{\lambda}$ defined in~\eqref{eq:background:rkhs:rn_estimate} and the target function $\Rsans$ defined in~\eqref{prop:lancaster}. Finally, define
    \begin{align*}
        K_{\max{}} := 1 + (4\kappa_{\max}^2 + \kappa_{\max})^2.
    \end{align*}
    There exists a constant $C > 0$ (independent of $N$ and $\delta$) such that for all $\lambda \in [\lambda_\star, \kappa_{\max{}}]$,
    \begin{align}
        \norm{\Rhat - \Rsans}_{\calS} \leq C\polylog(1/\delta) \sbr{K_{\max{}} \lambda^{\beta} + \frac{K_{\max{}}^{1/2}}{\lambda}\p{\Np^{-1/2} + \Nu^{-1/2}}}\label{eq:main:estimation:rn:sample_complexity2}
    \end{align}
    with probability at least $1 - \delta/2$.
\end{theorem}
By optimizing the bound appearing in \eqref{eq:main:estimation:rn:sample_complexity2} in $\lambda$, we get that
\begin{align}
    \lambda \equiv \lambda_{\Np, \Nu} = \p{\frac{\Np^{-1/2} + \Nu^{-1/2}}{K_{\max{}}^{1/2}}}^{\frac{1}{\beta+1}}.\label{eq:rkhs:rn:lambda}
\end{align}
If the expression from~\eqref{eq:rkhs:rn:lambda} falls within $[\lambda_\star, \kappa_{\max{}}]$, this yields the upper bound 
\begin{align}
    \norm{\Rhat - \Rsans}_{\calS} \leq C\polylog(1/\delta) \sbr{K_{\max{}}^{\frac{\beta+2}{2(\beta+1)}}\p{\Np^{-1/2} + \Nu^{-1/2}}^{\frac{\beta}{\beta+1}}}.\label{eq:main:estimation:rn:sample_complexity3}
\end{align}
The condition $\lambda_{\Np, \Nu} \leq \kappa_{\max{}}$ can be satisfied by taking $(\Np, \Nu)$ sufficiently large.  For the condition that $\lambda_{\Np, \Nu} \geq \lambda_\star$, we find an upper bound on $\lambda_\star$ by first deriving an upper bound on $\df(\lambda)/\lambda$, and then solving the resulting equation in $\lambda$. By \Cref{asm:main:estimation:rn:source}, we have that
\begin{align}
    \frac{\df(\lambda)}{\lambda} \leq C_{\df} \lambda^{-(\alpha+1)/\alpha} \implies \lambda_\star \leq \p{\frac{C_{\df}}{\Nu}}^{\frac{\alpha}{\alpha+1}}.\label{eq:background:rnd:lambda_star}
\end{align}
Viewing the dependence of~\eqref{eq:rkhs:rn:lambda} on $\Nu$, we see that if $\beta \geq (1-\alpha)/(2\alpha)$, then there exists $\Nu$ large enough such that~\eqref{eq:rkhs:rn:lambda} is greater than the right-hand side of~\eqref{eq:background:rnd:lambda_star}. This is always satisfied, as $\alpha > 1$ is required for $\Sop_{X,Z} \Iop_{X,Z}$ to be trace class. Thus, we have the following convergence rate.
\begin{corollary}\label{coro:background:rn}
    Adopt the setting of \Cref{thm:background:rn2}. Let $\Nu$ be large enough such that~\eqref{eq:rkhs:rn:lambda} is upper bounded by $\kappa_{\max{}}$ and lower bounded by the right-hand side of~\eqref{eq:background:rnd:lambda_star}. Then, for the choice~\eqref{eq:rkhs:rn:lambda}, it holds that
    \begin{align}
        \norm{\Rhat - \Rsans}_{\calS}^2 \leq C\polylog(1/\delta) \sbr{K_{\max{}}^{\frac{\beta+2}{\beta+1}}\p{\Np^{-1/2} + \Nu^{-1/2}}^{\frac{2\beta}{\beta+1}}}.\label{eq:main:estimation:rn:sample_complexity4}
    \end{align}
\end{corollary}

This result is applied in \Cref{sec:a:complexity:rn_derivative} and provides an example of the ``information density'' approach outlined in \Cref{sec:framework} and \Cref{sec:theory}. In the sequel, we will simply assume that $\Np = \Nu = N / 2$ to simplify the statement of the result. Finally, note that the source condition \Cref{asm:main:estimation:rn:source} does not have any implications for the mis-specified case ($\Rsans \notin \calS$). This aspect of Radon-Nikodym estimation methodology is still an active area of research in statistical learning.

\section{Prompt Bias and Residual Dependence}\label{sec:a:dependence}
This appendix is dedicated to the proof of \Cref{thm:res_dep1}, which controls the population quantity $\norm{\etastar - \eta_\rho}_{\ltwo(P_X)}^2$. The result will follows from \Cref{thm:res_dep}, which is a more mathematically precise version of \Cref{thm:res_dep1} from the main text.

We recall the problem setting of \Cref{sec:theory}. We consider the three central probability measures $P_{X, Y}$ (evaluation distribution), $Q_{X, Z}$ (pre-training distribution), and $\rho_{Y, Z}$ (prompt distribution). We notice that $\etastar$ (from~\eqref{eq:est:bayes}) depends on $P_{X, Y}$, while $\eta_\rho$ (from~\eqref{eq:est:ts}) depends on the pair $(Q_{X, Z}, \rho_{Y, Z})$. Neither component of this term depends on a joint probability over $\X \times \Y \times \Z$. Thus, in order to relate them on common ground, we consider a joint probability measure $P_{X, Y, Z}$, which satisfies certain constraints that make it compatible with the distributions that have observable data. We call this the \emph{latent caption model}.

To proceed, we will need to make several mild regularity conditions on $P_{X, Y, Z}$. We use the notion of regular conditional distribution, or r.c.d. (\Cref{def:rcd}), introduced in \Cref{sec:a:background:l2}. We use more explicit notation in this section (e.g.~$Z=\z$) as compared to \Cref{sec:theory} to emphasize the random variable being conditioned on. The assumption below provides a more formal description of \Cref{asm:rcd1} from \Cref{sec:theory}.
\begin{assumption}\label{asm:rcd}
    The joint probability $P_{X, Y, Z}$ on $\X \times \Y \times \Z$ satisfies the following constraints.
    \begin{itemize}
        \item {\bf Agrees jointly with the evaluation distribution:} For all measurable sets  $A \sse \X \times \Y$, we have that $P_{X, Y, Z}(A \times \Z) = P_{X, Y}(A)$ (i.e.~$P_{X, Y, Z}$ agrees with the given marginal $P_{X, Y}$).
        \item {\bf Agrees conditionally with the pre-training distribution:} There exists a measurable set $\X_1 \sse \X$ with $P_X(\X_1) = 1$ such that the regular conditional distributions $Q_{Z|X = \x}$ and $P_{Z|X = \x}$ on $\Z$ exist. Furthermore, these satisfy $Q_{Z|X = \x} = P_{Z|X = \x}$ for all $\x \in \X_1$.
        \item {\bf Regularity of conditional distributions:} There exists a measurable set $\Z_1 \sse \Z$ with $P_Z(\Z_1) = 1$ such that the regular conditional distributions $P_{X, Y|Z = \z}$ on $\X \times \Y$ exists for all $\z \in \Z_1$. Furthermore, we have the absolute continuity relation $P_{X, Y|Z = \z} \ll P_{X|Z = \z} \otimes P_{Y|Z = \z}$ with Radon-Nikodym derivative 
        \begin{align}
            \Ssans_{\z} := \frac{\d P_{X, Y|Z = \z}}{\d (P_{X|Z = \z} \otimes P_{Y|Z = \z})},\label{eq:main:dependence:rnd}
        \end{align}
        that satisfies $\E{P_{X, Y|Z = \z}}{\Ssans_{\z}(X, Y)} < +\infty$ for each $\z \in \Z_1$ and $\E{P_{X, Y, Z}}{\Ssans_{Z}(X, Y)} < +\infty$.
    \end{itemize}
\end{assumption}
That $P_{X, Y, Z}$ marginalizes to $P_{X, Y}$ is more of an axiomatic property than an assumption, but we phrase it as so to emphasize that $P_{X, Y, Z}$ is meant to describe the evaluation distribution. The assumption that the conditionals $Q_{Z|X}$ and $P_{Z|X}$ match almost surely represents the viewpoint that, after fixing an image $\x$, the latent caption $Z|X = \x$ follows the same relationship to $\x$ as seen during pre-training. Importantly, this does not require or imply that $P_X = Q_X$ or that $P_Z = Q_Z$. The marginal distribution $P_X$ is supplied entirely by the evaluation distribution $P_{X, Y}$, as for any measurable set $A \sse \X$, we have by definition that $P_X(A) = P_{X, Y}(A \times \Y)$. On the other hand, the marginal $P_Z$ can be defined using the Markov kernel $P_{Z|X = \x}$, in that for any measurable $B \sse \Z$, it holds that
\begin{align*}
    P_Z(B) := \int_{\X_1} P_{Z|X = \x}(B) \d P_X(\x) = \int_{\X_1} Q_{Z|X = \x}(B) \d P_X(\x).
\end{align*}
Finally, the absolute continuity condition, i.e., the existence of~\eqref{eq:main:dependence:rnd}, rules out degeneracies such as $Y$ being a deterministic function of $X$ given $Z = \z$ (outside of a set of $P_Z$-measure zero). 
It is also worth pointing out that the first two conditions \Cref{asm:rcd} do not contradict one another. For example, one can consider $P_{X, Y, Z}$ that satisfies the Markov chain $Y \rightarrow X \rightarrow Z$, where $(X, Y)$ is drawn according to $P_{X, Y}$, and $Z$ and $Y$ are conditionally independent given $X$. Then, informally, we have that $P_{Z|X, Y} = P_{Z|X} = Q_{Z|X}$, so $P_{X, Y, Z}$ is uniquely determined. While this example implies the existence of a valid joint probability measure $P_{X, Y, Z}$, it is also, in a sense, showcasing the ``least desirable'' distribution for zero-shot prediction, as the dependence between $X$ and $Z$ does not provide any additional information about $Y$.

We recall some notation from \Cref{sec:theory}. Let
\begin{align*}
    g_{P_{Y, Z}}(\z) = \E{P_{Y, Z}}{\fstar(Y)|Z}(\z).
\end{align*}
Note that $g_{P_{Y, Z}}$ is simply a conditional expectation constructed via \Cref{def:conditional_expectation}, and does not require the existence of an r.c.d.~$P_{Y|Z =\z}$. In the bound, we will encounter a prompt bias term that compares $g_{P_{Y, Z}}$ to $g_\rho$ from~\eqref{eq:g_rho}. This reflects the notion that if $P_{X, Y, Z}$ agrees with two of the three fundamental distributions governing the problem, it will not be able to agree with the prompt distribution $\rho_{Y, Z}$ in general. Finally, the r.c.d.~$P_{X, Y|Z = \z}$ allows us to measure conditional dependence using the \emph{conditional mean squared contingency}, defined by the formula
\begin{align*}
    I(X; Y|Z=\z) = \E{P_{X|Z=\z} \otimes P_{Y|Z = \z}}{\p{1 - \Ssans_{\z}(Y, X)}^2}.
\end{align*}
As is shown in the proof, $I(X; Y|Z=\z)$ and its expectation over $P_Z$ are well-defined under \Cref{asm:rcd}. We are now ready to state the result.
\begin{theorem}\label{thm:res_dep}
    Assume that $\fstar$ is bounded in absolute value by $\fstarbound$. Under \Cref{asm:rcd}, it holds that
    Then, it holds that
    \begin{align}
        \norm{\eta_\rho - \etastar}_{\ltwo(P_X)}^2 \leq 2\underbrace{\norm{g_\rho - g_{P_{Y, Z}}}_{\ltwo(P_Z)}^2}_{\text{prompt bias}} + 2\fstarbound^2\underbrace{\E{P_Z}{I(X; Y|Z)}}_{\text{residual dependence}}. \label{eq:res_dep}
    \end{align}
\end{theorem}
\begin{proof}
    We first establish a useful representation of the conditional mean of $\fstar(Y)$ given $X = \x$, in terms of the (conditional) information density from \Cref{lem:info_density}. Fix $\x \in \X_1$ and $\z \in \Z_1$, the sets on which the regular conditional distributions $P_{Z|X=\x}$ and $P_{X, Y|Z = \z}$ are defined (see \Cref{asm:rcd}). Because of the existence the Radon-Nikodym derivative $\Ssans_{\z}$ from~\eqref{eq:main:dependence:rnd}, we may apply~\Cref{lem:info_density} with $U = Y$, $V = X$, and $h= \fstar$ to write
    \begin{align*}
        \E{P_{X, Y|Z = \z}}{\fstar(Y)|X}(\x) = \underbrace{\E{P_{Y|Z = \z}}{\fstar(Y) \Ssans_{\z}(Y, \x)}}_{=: f(\x, \z)} \text{ for all } (\x, \z) \in \X_1 \times \Z_1.
    \end{align*}
    The chosen notation $\E{P_{X, Y|Z = \z}}{\fstar(Y)|X}(\x)$ indicates that after fixing the probability measure $P_{X, Y|Z = \z}$, we take the conditional expectation of the function $\fstar \in \ltwo(P_{Y|Z=\z})$ via \Cref{def:conditional_expectation}, which does not necessarily posit the existence of the r.c.d.~$P_{Y|X = \x, Z = \z}$.\footnote{This is why we do not write, for instance, $\E{P_{Y|X = \x, Z = \z}}{\fstar(Y)}$. This consideration is purely technical, and the reader may make this substitution for conceptual understanding of the proof.} 
    We have denoted the right-hand side by the function $f(\x, \z)$.
    Integrate both sides over $P_{Z|X = \x}$, then use the tower property of conditional expectation (\Cref{lem:tower}) to achieve
    \begin{align}
        \etastar(\x) = \E{P_{X, Y}}{\fstar(Y)|X}(\x) &= \int_\Z \E{P_{X, Y|Z = \z}}{\fstar(Y)|X}(\x) \d P_{Z|X = \x}(\z) \notag\\
        &= \int_\Z f(\x, \z) \d P_{Z|X = \x}(\z) \notag\\
        &= \E{P_{Z|X=\x}}{f(\x, Z)}.\label{eq:main:dependence:rnd:info_dens}
    \end{align}
    Using the identity~\eqref{eq:main:dependence:rnd:info_dens} and $Q_{Z|X =\x} = P_{Z|X =\x}$ on $\x \in \X_1$ (\Cref{asm:rcd}), we write
    \begin{align}
        &\eta_\rho(\x) - \etastar(\x) \notag\\
        &= \E{Q_{Z|X=\x}}{g_\rho(Z)} - \E{P_{Z|X=\x}}{f(\x, Z)} \notag\\ 
        &= \E{P_{Z|X=\x}}{g_\rho(Z)} - \E{P_{Z|X=\x}}{f(\x, Z)} \notag\\ 
        &= \E{P_{Z|X=\x}}{\p{g_\rho(Z) - g_{P_{Y, Z}}(Z)}} + \E{P_{Z|X=\x}}{\p{g_{P_{Y, Z}}(Z) - f(\x, Z)}}.\notag
    \end{align}
    Taking the integral over $P_X$, we have by the decomposition above that
    \begin{align}
         \norm{\eta_\rho - \etastar}_{\ltwo(P_X)}^2 &= \int_\X \p{\eta_\rho(\x) - \etastar(\x)}^2 \d P_X(\x)\notag\\
         &\leq 2\int_{\X_1} \p{\E{P_{Z|X=\x}}{{g_\rho(Z) - g_{P_{Y, Z}}(Z)}}}^2 \d P_X(\x) \label{eq:main:dependence:rnd1}\\
         &\quad+ 2\int_{\X_1}\p{\E{P_{Z|X=\x}}{{g_{P_{Y, Z}}(Z) - f(\x, Z)}}}^2\d P_X(\x).\label{eq:main:dependence:rnd2}
    \end{align}
    To handle~\eqref{eq:main:dependence:rnd1}, we apply Jensen's inequality for each r.c.d.~$P_{Z|X = \x}$ to achieve
    \begin{align*}
        \int_{\X_1} \p{\E{P_{Z|X=\x}}{\p{g_\rho(Z) - g_{P_{Y, Z}}(Z)}}}^2 \d P_X(\x)
        &\leq \int_{\X_1} \E{P_{Z|X=\x}}{\p{g_\rho(Z) - g_{P_{Y, Z}}(Z)}^2} \d P_X(\x)\\
        &= \E{P_{Z}}{\p{g_\rho(Z) - g_{P_{Y, Z}}(Z)}^2}\\
        &= \norm{g_\rho - g_{P_{Y, Z}}}_{\ltwo(P_Z)}^2.
    \end{align*}
    It remains to control~\eqref{eq:main:dependence:rnd2}. Applying Jensen's inequality for each r.c.d.~$P_{Z|X = \x}$ once again, we have that
    \begin{align}
        \int_{\X_1}\p{\E{P_{Z|X=\x}}{\p{g_{P_{Y, Z}}(Z) - f(\x, Z)}}}^2\d P_X(\x) 
        &\leq \int_{\X_1}\p{\E{P_{Z|X=\x}}{\p{g_{P_{Y, Z}}(Z) - f(\x, Z)}^2}}\d P_X(\x) \notag\\
        &= \E{P_{X, Z}}{\p{g_{P_{Y, Z}}(Z) - f(X, Z)}^2} \notag\\
        &= \int_{\Z_1} \E{P_{X|Z=\z}}{\p{g_{P_{Y, Z}}(\z) - f(X, \z)}^2} \d P_Z(\z),\label{eq:main:dependence:rnd3}
    \end{align}
    where the last step follows due to the existence of the r.c.d.~$P_{X|Z=\z}$ for $\z \in \Z_1$, as $P_{X|Z=\z}(A) := P_{X, Y|Z=\z}(A \times \Y)$ for every measurable $A \sse \X$, and the latter exists by assumption. Using the definition of $g_{P_{Y,Z}}$, write
    \begin{align*}
        g_{P_{Y, Z}}(\z) - f(\x, \z) = \E{P_{Y|Z = \z}}{\fstar(Y) (1 - \Ssans_{\z}(Y, \x))}.
    \end{align*}
    We may substitute this expression into the integrand of~\eqref{eq:main:dependence:rnd3} and apply Jensen's inequality to $P_{Y|Z = \z}$ to achieve
    \begin{align*}
        \E{P_{X|Z=\z}}{\p{g_{P_{Y, Z}}(\z) - f(X, \z)}^2} &= \E{P_{X|Z=\z}}{\p{\E{P_{Y|Z = \z}}{\fstar(Y) (1 - \Ssans_{\z}(Y, X))}}^2}\\
        &\leq \E{P_{X|Z=\z}}{\E{P_{Y|Z = \z}}{\p{\fstar(Y) (1 - \Ssans_{\z}(Y, X))}^2}}\\
        &\leq \norm{\fstar}^2_\infty \E{P_{X|Z=\z}}{\E{P_{Y|Z = \z}}{\p{1 - \Ssans_{\z}(Y, X)}^2}}\\
        &= \norm{\fstar}^2_\infty \E{P_{X|Z=\z} \otimes P_{Y|Z = \z}}{\p{1 - \Ssans_{\z}(Y, X)}^2},
    \end{align*}
    where the final step follows by applying Fubini's theorem \citep[Corollary 14.9]{schilling2017measures} to the product measure $P_{X|Z=\z} \otimes P_{Y|Z = \z}$ for fixed $\z \in \Z_1$. By the definition of mean squared contingency (\Cref{def:msc}), it holds that
    \begin{align}
        \E{P_{X|Z=\z} \otimes P_{Y|Z = \z}}{\p{1 - \Ssans_{\z}(Y, X)}^2} = I(X; Y|Z=\z).\label{eq:main:dependence:cmsc}
    \end{align}
    After confirming that~\eqref{eq:main:dependence:cmsc} is $P_Z$-integrable, substituting this expression back into~\eqref{eq:main:dependence:rnd3} achieves the desired result. Expand the quadratic term and apply the Radon-Nikodym theorem \citep[Theorem 20.1]{schilling2017measures} to achieve
    \begin{align*}
        I(X; Y|Z=\z) &= 1 - 2\E{P_{X|Z=\z} \otimes P_{Y|Z = \z}}{\Ssans_{\z}(Y, X)} + \E{P_{X|Z=\z} \otimes P_{Y|Z = \z}}{\Ssans_{\z}^2(Y, X)}\\
        &= 1 - 2\E{P_{X, Y|Z=\z}}{1} + \E{P_{X, Y|Z=\z}}{\Ssans_{\z}(Y, X)}\\
        &= \E{P_{X, Y|Z=\z}}{\Ssans_{\z}(Y, X)} - 1.
    \end{align*}
    Thus, by integrating against $P_Z$, we see that
    \begin{align*}
        \E{P_Z}{I(X; Y|Z)} = \E{P_{X, Y, Z}}{\Ssans_{Z}(Y, X)} - 1,
    \end{align*}
    where the expectation term is finite by \Cref{asm:rcd}. The proof is complete.
\end{proof}

\section{Sample Complexity and Distribution Mismatch}\label{sec:a:complexity}
This appendix provides the proofs of \Cref{thm:complexity_vvkr} and \Cref{thm:complexity_rn_derivative} by way of \Cref{thm:complexity:conditional_mean} and \Cref{thm:complexity:rn_derivative}, respectively. To recall the bigger picture, we first applied the decomposition~\eqref{eq:theory:decomp1}, which exposed the estimation error term
\begin{align}
    \norm{\heta_\rho - \eta_\rho}_{\ltwo(P_X)}^2,\label{eq:main:estimation:l2_error}
\end{align}
where $P_X$ is the $\X$-marginal of the evaluation distribution $P_{X, Y}$, $\eta_\rho$ is defined by $\eta_\rho(\x) := \E{Q_{X, Z}}{g_\rho(Z)|X}(\x)$ (see~\eqref{eq:g_rho}), and $\heta_\rho$ is one of two estimation procedures that is based on either~\eqref{eq:framework:cond_mean} or~\eqref{eq:framework:rnd}.
By using standard change of measure arguments (collected in \Cref{sec:a:complexity:shift}), we pass the problem of controlling~\eqref{eq:main:estimation:l2_error} in high probability to controlling $\norm{\heta_\rho - \eta_\rho}_{\ltwo(Q_X)}^2$ (i.e.~the mean squared error with respect to the pre-training marginal $Q_X$). Thus, the format of both \Cref{thm:complexity:conditional_mean} and \Cref{thm:complexity:rn_derivative} will be an upper bound on $\norm{\heta_\rho - \eta_\rho}_{\ltwo(Q_X)}^2$ that holds with an arbitrary failure probability $\delta \in (0, 1]$.

The identities~\eqref{eq:framework:cond_mean} and~\eqref{eq:framework:rnd} from \Cref{sec:framework} can be summarized with the equality
\begin{align}
      \M_{Z|X} g_\rho = \eta_\rho = \E{\rho_{Y,Z}}{\fstar(Y)\Rsans(\cdot, Z)} + \operatorname{err}(Q_Z, \rho_Z), \label{eq:est:ts2}
\end{align}
where the $\operatorname{err}(Q_Z, \rho_Z)$ is elaborated on in \Cref{sec:a:complexity:rn_derivative}. In \Cref{sec:a:complexity:conditional_mean}, we consider the left-hand side of~\eqref{eq:est:ts2}, and define $\heta_\rho$ by constructing an estimate $\hM_{Z|X}$ of $\M_{Z|X}$ using pre-training data and $\hg_\rho$ of $g_\rho$ using prompts. This will be referred to as the conditional mean approach. In \Cref{sec:a:complexity:rn_derivative}, we consider the right-hand side of~\eqref{eq:est:ts2} and define $\heta_\rho$ by using an estimate $\Rhat$ of $\Rsans$ using pre-training data and $\hrho_{Y,Z}$ of $\rho_{Y,Z}$ using prompts. This will be referred to as the information density approach. For both approaches, we adopt a parallel structure and break the analysis into the following steps.

\begin{enumerate}
    \item {\bf Decomposing the Global Error:} We first provide a generic upper bound on the mean squared error
    \begin{align}
        \norm{\heta_\rho - \eta_\rho}_{\ltwo(Q_X)}^2\label{eq:complexity:mse_qx}
    \end{align}
    in terms of the individual estimators defined by the pre-training and prompting data. While some additional structure may be employed in these bounds (e.g.~the estimate lives in a reproducing kernel Hilbert space), the decomposition is generally agnostic to the choice of method and can accommodate multiple estimation/learning strategies.
    \item {\bf Interpreting the Source Condition:} The error term related to the pre-training data refers to the distance between the conditional mean operators $\hM_{Z|X}$ and $\M_{Z|X}$ or the information densities $\Rhat$ and $\Rsans$ measured in an appropriate sense. This is initially controlled by substituting a particular estimation method among those reviewed in \Cref{sec:a:background}. As mentioned in \Cref{sec:theory}, the convergence rates of these methods rely on \emph{source conditions} that describe the regularity of the target function. We derive expressions that relate the source conditions to measures of dependence between $X$ and $Z$, so that, in turn, the rate can also be expressed in terms of these fundamental quantities.
    \item {\bf Controlling the Prompting Term:} The error term related to the prompting data will have a high probability bound, which is stated in the form of an assumption. This generality is maintained because the estimation based on the prompting data usually relies on simple primitives such as real-valued regression or finite-dimensional parameter estimation. Statistically, these problems are easier than the vector-valued regression or Radon-Nikodym derivative estimation problems that arise in the pre-training step. Thus, many possible methods can be used, and we provide examples in each case.
    \item {\bf Completing the Proof:} We combine the steps above to state the final bounds on~\eqref{eq:complexity:mse_qx}. They are stated in \Cref{thm:complexity:conditional_mean} and \Cref{thm:complexity:rn_derivative}, respectively.
\end{enumerate}
The steps above comprise the subsections of \Cref{sec:a:complexity:conditional_mean} and \Cref{sec:a:complexity:rn_derivative} below.
The bounds on mean square error on $Q_X$ are tied to misclassification risk on $P_{X, Y}$ via \Cref{sec:a:complexity:shift} and \Cref{sec:a:complexity:classification} to produce end-to-end performance guarantees. We compare the sampling schemes used for prompting that are employed in the theoretical analysis to the sampling schemes used empirically in \Cref{sec:a:complexity:sampling}.

\subsection{Conditional Mean Approach}\label{sec:a:complexity:conditional_mean}
This approach is based on the LHS of~\eqref{eq:est:ts2} yielding the result of \Cref{thm:complexity_vvkr}. The exposition relies heavily on the background introduced in \Cref{sec:a:background:rkhs}. In particular, we maintain the reproducing kernel Hilbert spaces $\calH$ and $\calG$ containing real-valued functions on $\X$ and $\Z$, respectively. We denote by $\ltwo(Q_X; \calG)$ the Bochner space containing equivalence classes of functions mapping from $\X$ to $\calG$. We also use the the bracket notation $[\cdot]_X$ to index a functions equivalence class in $\ltwo(Q_X)$ (or $\ltwo(Q_X;\calG)$ for $\calG$-valued functions). 

\myparagraph{Setup}
We first introduce an element $F_\star$ of $\ltwo(Q_X; \calG)$ which can be used to represent the function $\x \mapsto [\M_{Z|X}g_\rho](\x)$. We then describe how an estimator $\Fhat$ of $F_\star$ and an approximation $\hg_\rho$ of $g_\rho$ can be used to define an estimated predictor $\heta_\rho$. 
Recall the boundedness assumptions  $\sup_{\x, \x' \in \X}k(\x, \x') \leq \kmax <\infty$ and  $\sup_{\z, \z' \in \Z}l(\z, \z') \leq \lmax < \infty$.

\begin{lemma}\label{lem:main:construction}
    It holds that 1) $[\eta_\rho]_X \in \ltwo(Q_X)$, and 2) there exists a function $F_\star: \X \rightarrow \calG$ such that $[F_\star]_X \in \ltwo(Q_X; \calG)$ and
    \begin{align}
        [\E{Q_{X, Z}}{g(Z)|X}(\cdot)]_X = [\ip{g, F_\star(\cdot)}_\calG]_X \text{ for all } g \in \calG.\label{eq:main:estimation:eta_rho_to_F} 
    \end{align}
    In particular, $[\eta_\rho]_X = [\ip{g_\rho, F_\star(\cdot)}_\calG]_X$.
\end{lemma}
\begin{proof}
    Using the notation from \Cref{sec:a:background:l2}, if we show that the random variable $\omega \mapsto g_\rho(Z(\omega))$ is contained in $\Lsans(\msc{F})$, then the first claim holds by the definition of conditional expectation in $\ltwo(Q_X)$ (see \Cref{def:conditional_expectation}). Using the reproducing property of the RKHS $\calG$, we have that
    \begin{align}
        \E{Q_Z}{g_\rho^2(Z)} = \E{Q_Z}{\ip{g_\rho, \psi(Z)}_{\calG}^2} \leq \norm{g_\rho}_\calG^2 \cdot \Ex_{Q_Z}\norm{\psi(Z)}_\calG^2 \leq \lmax \norm{g_\rho}_\calG^2,\label{eq:main:bounded}
    \end{align}
    granting the claim that $[\eta_\rho]_X \in \ltwo(Q_X)$. Next, fix any $\x \in \X$, and define the map
    \begin{align*}
        g \mapsto T_{\x}(g) = [\M_{Z|X}g](\x) = \E{Q_{X, Z}}{g(Z)|X}(\x).
    \end{align*}
    By the same argument as~\eqref{eq:main:bounded}, we have that $\abs{T_{\x}(g)} \leq \sqrt{\lmax} \norm{g}_\calG$, indicating that $T_{\x}$ is a bounded linear functional. By the Riesz representation theorem, there exists an element of $\calG$, denoted as $\E{Q_{X, Z}}{\psi(Z)|X}(\x)$, that satisfies
    \begin{align*}
        \E{Q_{X, Z}}{g(Z)|X}(\x) = \ip{g, \E{Q_{X, Z}}{\psi(Z)|X}(\x)}_\calG \text{ for all } g \in \calG.
    \end{align*}
    Next, given the collection of Riesz representers $\br{\E{Q_{X, Z}}{\psi(Z)|X}(\x): \x \in \X}$, one may construct the mapping
    \begin{align*}
        F_\star: \X \rightarrow \calG, \text{ defined by } \x \mapsto F_\star(\x) = \E{Q_{X, Z}}{\psi(Z)|X}(\x) \in \calG.
    \end{align*}
    It only remains to show that $[F_\star]_X \in \ltwo(Q_X; \calG)$. By Jensen's inequality and the tower property (\Cref{lem:tower}), we have that
    \begin{align*}
        \int_\X \norm{F_\star(\x)}_\calG^2 \d Q_X(\x) \leq \Ex_{Q_Z}\norm{\psi(Z)}_\calG^2 \leq \lmax < \infty,
    \end{align*}
    completing the proof.
\end{proof}

Now that we have identified the vector-valued function of interest, $F_\star$, we can consider an estimation procedure that will return $\Fhat \equiv \Fhat_{\lambda}$, with $[\Fhat]_X \in \ltwo(Q_X; \calG)$ and regularization parameter $\lambda > 0$. 
Then, we may define the estimator $\heta_\rho$ of $\eta_\rho$ via the inner product
\begin{align}
    \heta_\rho(\x) = \ipsmall{\hg_\rho, \Fhat(\x)}_\calG,\label{eq:complexity:conditonal_mean:estimator}
\end{align}
where $\hg_\rho$ satisfies some approximation bound with respect to $g_\rho$. Our decomposition will expose an error term for which we can apply \Cref{thm:vvkr}, which describes the convergence rate of spectral regularization learning.

\subsubsection{Decomposing the Global Error}
Returning to the original quantity we wish to control from~\eqref{eq:main:estimation:l2_error}, we apply the following decomposition.
\begin{lemma}[Error Decomposition]\label{lem:complexity:conditional_mean:decomp}
    For any choice of $\Fhat \in \ltwo(Q_X; \calG)$ it holds that
    \begin{align*}
        \norm{\heta_\rho - \eta_\rho}_{\ltwo(Q_X)}^2 &\leq 3\norm{g_\rho}_\calG^2 \cdot \norm{\Fhat - F_\star}^2_{\ltwo(Q_X; \calG)} + 3 \norm{F_\star}^2_{\ltwo(Q_X; \calG)} \cdot  \norm{\hg_\rho - g_\rho}_\calG^2 \\
        &\quad + 3\norm{\hg_\rho - g_\rho}_\calG^2 \cdot \norm{\Fhat - F_\star}^2_{\ltwo(Q_X; \calG)},
    \end{align*}
\end{lemma}
\begin{proof}
    Using the reproducing property of the RKHS $\calG$ and Young's inequality we have that
    \begin{align}
        &\norm{\heta_\rho - \eta_\rho}_{\ltwo(Q_X)}^2 \notag \\
        &= \int_\X (\heta_\rho(\x) - \eta_\rho(\x))^2 \d Q_X(\x) \notag\\
        &\leq 3\int_\X \ipsmall{g_\rho, \Fhat(\x) - F_\star(\x)}_\calG^2 \d Q_X(\x) + 3\int_\X \ipsmall{\Fhat(\x), \hg_\rho - g_\rho}_\calG^2 \d Q_X(\x)\notag\\
        &\quad + 3\int_\X \ipsmall{\Fhat(\x) - F_\star(\x), \hg_\rho - g_\rho}_\calG^2 \d Q_X(\x).\notag
    \end{align}
    Then, applying the Cauchy-Schwarz inequality in $\calG$, we have that
    \begin{align}
    &\norm{\heta_\rho - \eta_\rho}_{\ltwo(Q_X)}^2\notag\\
        &\leq 3\norm{g_\rho}_\calG^2 \cdot \int_\X \norm{ \Fhat(\x) - F_\star(\x)}_\calG^2 \d Q_X(\x) + 3\p{\int_{\X}\norm{F_\star(\x)}_\calG^2 \d Q_X(\x)} \cdot \norm{\hg_\rho - g_\rho}_\calG^2\notag\\
        &\quad + 3\norm{\hg_\rho - g_\rho}_\calG^2 \int_\X \norm{\Fhat(\x) - F_\star(\x)}_\calG^2 \d Q_X(\x)\notag\\
        &= 3\norm{g_\rho}_\calG^2 \cdot \norm{\Fhat - F_\star}^2_{\ltwo(Q_X; \calG)} + 3 \norm{F_\star}^2_{\ltwo(Q_X; \calG)} \cdot  \norm{\hg_\rho - g_\rho}_\calG^2  \label{eq:main:est:F_error}\\
        &\quad + 3\norm{\hg_\rho - g_\rho}_\calG^2 \cdot \norm{\Fhat - F_\star}^2_{\ltwo(Q_X; \calG)},\notag
    \end{align}
    the result as desired.
\end{proof}
In the decomposition of \Cref{lem:complexity:conditional_mean:decomp}, we observe the dominating terms $\norm{g_\rho}_\calG^2 \cdot \norm{\Fhat - F_\star}^2_{\ltwo(Q_X; \calG)}$ and $\norm{F_\star}^2_{\ltwo(Q_X; \calG)} \cdot  \norm{\hg_\rho - g_\rho}_\calG^2$, along with the higher order term $\norm{\hg_\rho - g_\rho}_\calG^2 \cdot \norm{\Fhat - F_\star}^2_{\ltwo(Q_X; \calG)}$. We consider estimators $\Fhat$ and $\hg_\rho$ based on kernel regularized learning techniques in order to bound the dominating terms, as a function of $N$ and $M$. The bounds are optimized individually with respect to the regularization parameters of each learning objective.

\subsubsection{Interpreting the Source Condition}\label{sec:a:complexity:conditional_mean:source}
To approach this, we associate our function of interest $F_\star \in \ltwo(Q_X;\calG)$ to an object $\C_\star \in \HS(\ltwo(Q_X), \calG)$ by way of an isometric isomorphism introduced in \Cref{thm:aubin:isomorphism}. 
This then allows us to derive a convenient formula for the quantity $\norm{F_\star}_\beta$, which appears in \Cref{asm:rkhs:source}, and relies on the interplay between $\calH$ and $\ltwo(Q_X)$ described in \Cref{sec:a:background:rkhs}.
\begin{lemma}\label{lem:main:hsnorm}
    Let $(g_j)_{j \in J}$ be any orthonormal basis (ONB) of $\calG$ and recall the eigenfunctions $([e_{X, i}]_X)_{i \in I}$ from~\eqref{eq:main:eigendecomposition}. Assuming that $\norm{F_\star}_\beta$ is finite, it holds that
    \begin{align*}
        \norm{F_\star}_\beta^2 &= \sum_{i \in I} \sum_{j \in J} \mu_{X, i}^{-\beta} \ip{\M_{Z|X}[g_j]_Z, [e_{X, i}]_X}_{\ltwo(Q_X)}^2.
    \end{align*}
\end{lemma}
\begin{proof}
    By the definition of $\norm{\cdot}_\beta$, we have that
    \begin{align}
        \norm{F_\star}_\beta = \norm{\C_\star}_\beta = \norm{\C_\star \Top_X^{-\beta/2}}_{\HS(\ltwo(Q_X), \calG)}\label{eq:main:hsnorm1}
    \end{align}
    Then, notice that by the eigendecomposition~\eqref{eq:main:eigendecomposition}, we have that
    \begin{align*}
        \Top_X^{-\beta/2} [f]_X = 0 \text{ for all } [f]_X \in \p{\Cl(\Range(\Iop_X))}^\perp
    \end{align*}
    Thus, when computing the ~\eqref{eq:main:hsnorm1}, we may restrict $\HS(\ltwo(Q_X), \calG)$ to $\HS(\Cl(\Range(\Iop_X)), \calG)$. This allows us to employ the eigenvectors $([e_{X, i}]_X)_{i \in I}$ as a basis of $\Cl(\Range(\Iop_X))$ when computing the norm.
    We have that
    \begin{align*}
        &\norm{F_\star}_\beta^2 \\
        &= \norm{\C_\star \Top_X^{-\beta/2}}_{\HS(\Cl(\Range(\Iop_X)), \calG)}^2 \\
        &= \sum_{i \in I} \sum_{j \in J} \ipsmall{g_j, \C_\star \Top_X^{-\beta/2} [e_{X, i}]_X}_\calG^2 &\text{(by definition)}\\
        &= \sum_{i \in I} \sum_{j \in J} \mu_{X, i}^{-\beta}\ip{g_j, \C_\star [e_{X, i}]_X}_\calG^2  &\text{(by~\eqref{eq:main:eigendecomposition})}\\
        &= \sum_{i \in I} \sum_{j \in J} \mu_{X, i}^{-\beta}\ip{\C_\star, g_j \otimes [e_{X, i}]_X}_{\HS(\ltwo(Q_X), \calG)}^2\\
        &= \sum_{i \in I} \sum_{j \in J} \sum_{k \in J} \sum_{l \in J} \mu_{X, i}^{-\beta}\ip{g_k \otimes [f_k]_X, g_j \otimes [e_{X, i}]_X}_{\HS(\ltwo(Q_X), \calG)} \cdot \ip{g_l \otimes [f_l]_X, g_j \otimes [e_{X, i}]_X}_{\HS(\ltwo(Q_X), \calG)}, & \text{(\Cref{lem:main:estimation:technical})}
    \end{align*}
    where $f_k(\x) = \ip{F_\star(\x), g_k}_\calG = \E{Q_{X, Z}}{g_k(Z)|X}(\x)$. Phrased in terms of the conditional mean operator $\M_{Z|X}: \ltwo(Q_Z) \rightarrow \ltwo(Q_X)$, we have that
    \begin{align*}
        [f_k]_X = \M_{Z|X} [g_k]_Z.
    \end{align*}
    Plugging this into the display above, we have that
    \begin{align*}
    &\norm{F_\star}_\beta^2 \\
    &= \sum_{i \in I} \sum_{j \in J} \sum_{k \in J} \sum_{l \in J} \mu_{X, i}^{-\beta}\ip{\blue{g_k} \otimes \red{(\M_{Z|X} [g_k]_Z)}, \blue{g_j} \otimes \red{[e_{X, i}]_X}}_{\HS(\ltwo(Q_X), \calG)} \ip{\blue{g_l} \otimes \red{(\M_{Z|X} [g_l]_Z)}, \blue{g_j} \otimes \red{[e_{X, i}]_X}}_{\HS(\ltwo(Q_X), \calG)}\\
        &= \sum_{i \in I} \sum_{j \in J} \sum_{k \in J} \sum_{l \in J} \mu_{X, i}^{-\beta}\blue{\ip{g_k, g_j}_{\calG}} \blue{\ip{g_l, g_j}_{\calG}} \cdot \red{\ip{\M_{Z|X} [g_k]_Z, [e_{X, i}]_X}_{\ltwo(Q_X)}}\red{\ip{\M_{Z|X} [g_l]_Z, [e_{X, i}]_X}_{\ltwo(Q_X)}}\\
        &= \sum_{i \in I} \sum_{j \in J} \mu_{X, i}^{-\beta} \ip{\M_{Z|X} [g_j]_Z, [e_{X, i}]_X}_{\ltwo(Q_X)}^2,
    \end{align*}
    where the last step follows from the fact that $g_1, g_2, \ldots$ is an ONB of $\calG$. This completes the proof.
\end{proof}

It remains to select a choice of the collection $(g_j)_{j \in J}$. Note that $([g_j]_Z)_{j \in J}$ does not form an orthonormal system in $\ltwo(Q_Z)$, due to the distortion of the embedding. However, by explicitly writing the embedding $\Iop_Z$ (analogous to $\Iop_X$ introduced in~\eqref{eq:background:Tx}), we can derive one. Consider the singular value decomposition 
\begin{align}
    \Iop_Z = \sum_{k \in K} \mu_{Z, k}^{1/2} \p{[e_{Z, k}]_Z \otimes (\mu_{Z, k}^{1/2}  e_{Z, k})},\label{eq:main:estimation:basis:G}
\end{align}
which is analogous to the one introduced for $\Iop_X$ in~\eqref{eq:main:estimation:basis:H}. The index set $K$ is smaller in cardinality that $J$, as the collection $(e_{Z, k})_{k \in K}$ forms an ONB of $\Null(\Iop_Z)^\perp \sse \calG$, whereas $(g_j)_{j \in J}$ should be an ONB for all of $\calG$. Thus, we can expand the embedding $[g_j]_Z \in \ltwo(Q_Z)$ into
\begin{align*}
    [g_j]_Z = \Iop_Z g_j = \sum_{k \in K} \mu_{Z, k}^{1/2} \ip{g_j, \mu_{Z, k}^{1/2} e_{Z, k}}_\calG [e_{Z, k}]_Z.
\end{align*}
This decomposition allows us to simplify the equality in~\Cref{lem:main:hsnorm} further.
\begin{proposition}\label{prop:complexity:hsnorm}
    In the setting of \Cref{lem:main:hsnorm}, it holds that
    \begin{align}
        \norm{F_\star}_\beta^2 &= \sum_{i \in I} \sum_{j \in J} \mu_{X, i}^{-\beta} \mu_{Z, j} \ip{\M_{Z|X}[e_{Z, j}]_Z, [e_{X, i}]_X}_{\ltwo(Q_X)}^2 \label{eq:main:hsnorm_onb}\\
        &= \norm{\Top_X^{-\beta/2}\M_{Z|X} \Top_Z^{1/2}}_{\HS(\ltwo(Q_Z), \ltwo(Q_X))}^2. \label{eq:main:hsnorm_onb2}
    \end{align}
    In particular, $\norm{F_\star}_0^2 = \norm{F_\star}_{\ltwo(Q_X;\calG)}^2 = \norm{\M_{Z|X} \Top_Z^{1/2}}_{\HS(\ltwo(Q_Z), \ltwo(Q_X))}^2$.
\end{proposition}
\begin{proof}
    The sequence of functions $(\mu_{Z, k}^{1/2} e_{Z, k})_{k \in K}$ form an ONB of $\Null(\Iop_Z)^\perp \sse \calG$. Because $J$ indexes a basis of $\calG$, we have that $K \sse J$. Then, we may complete $(\mu_{Z, k}^{1/2} e_{Z, k})_{k \in K}$ to form the basis $(g_j)_{j \in J}$ of $\calG$, where $g_j = \mu_{Z, j}^{1/2} e_{Z, j}$ for all $j \in K$ and $g_j$ is defined arbitrarily for $j \notin K$. Plug $(g_j)_{j \in J}$ into the right-hand side of the formula given in \Cref{lem:main:hsnorm} gives~\eqref{eq:main:hsnorm_onb}, the first part of the claim.

    For the second equality, we note that $([e_{X, i}]_X)_{i \in I}$ and $([e_{Z, j}]_Z)_{j \in J}$ form orthonormal bases of $\Cl(\Range(\Iop_X))$ and $\Cl(\Range(\Iop_Z))$, respectively. We complete them (using the index sets $\bar{I}$ and $\bar{J}$) to form (possibly uncountable) orthonomal bases of $\ltwo(Q_X)$ and $\ltwo(Q_Z)$. Then, by the definition of the Hilbert-Schmidt norm, it holds that 
    \begin{align*}
        \norm{\Top_X^{-\beta/2}\M_{Z|X} \Top_Z^{1/2}}_{\HS(\ltwo(Q_Z), \ltwo(Q_X))}^2 &= \sum_{i \in \bar{I}} \sum_{j \in \bar{J}} \ip{\Top_X^{-\beta/2}\M_{Z|X} \Top_Z^{1/2} [e_{Z, j}]_Z, [e_{X, i}]_X}_{\ltwo(Q_X)}^2\\
        &= \sum_{i \in I} \sum_{j \in J} \mu_{X, i}^{-\beta} \mu_{Z, j} \ip{\M_{Z|X}[e_{Z, j}]_Z, [e_{X, i}]_X}_{\ltwo(Q_X)}^2,
    \end{align*}
    where we used in the second line that $[e_{X, i}]_X \in \Null(\Top_X^{-\beta/2})$ for $i \in \bar{I} \backslash I$ and $[e_{Z, j}]_Z \in \Null(\Top_Z^{1/2})$ for $j \in \bar{J} \backslash J$.
    This gives the~\eqref{eq:main:hsnorm_onb2} and completes the proof.
\end{proof}
It remains to interpret the equality in \Cref{prop:complexity:hsnorm} to complete the analysis.

\subsubsection{Controlling the Prompting Term}
From the decomposition given in \Cref{lem:complexity:conditional_mean:decomp}, the estimate $\hg_{\rho}$ will be designed as to control the RKHS-norm error $\norm{\hg_\rho - g_\rho}_\calG^2$. We phrase the assumption generically, but in a way that is reflective of the convergence rates seen in real-valued nonparametric regression. Recall the probability space $(\Omega, \msc{F}, \prob)$ introduced in \Cref{sec:a:background:l2}.

\begin{assumption}\label{asm:vvkrr:downstream}
    For constants $\delta \in (0, 1]$, $M\geq 1$, and $\omega_\rho \in (1/2, 1]$, there is an event $\mc{E}(\delta, M, \omega_\rho)$ that is independent of the pre-training data $(X_1, Z_1), \ldots, (X_N, Z_N)$, such that on $\mc{E}(\delta, M, \omega_\rho)$, 
    \begin{align}
        \norm{\hg_\rho - g_\rho}_\calG^2 \leq C\fstarbound^2 \polylog(1/\delta) M^{-\frac{2\omega_\rho - 1}{2\omega_\rho + 1}}.\label{eq:complexity:conditional_mean:prompt}
    \end{align}
    for a constant $C$ independent of $\delta$ and $M$.
    On $(\Omega, \msc{F}, \prob)$, the event $\mc{E}(\delta, M, \omega_\rho)$ occurs with probability at least $1 - \delta/2$.
\end{assumption}

The notation $\omega_\rho$ is chosen for the constant that determines the convergence rate, because it can be interpreted itself as a source condition constant for a real-valued nonparametric regression framework. Indeed, consider the case in which $\hg_{\rho}$ is computed using kernel ridge regression with parameter $\lambda$. Via the proof of their Theorem 2, \citet{smale2007learning} show that with probability at least $1 - \delta/2$,
\begin{align}
    \norm{\hg_\rho - g_\rho}_\calG \leq  C(\rho_{Y, Z})\log(4/\delta) \Big[\underbrace{\fstarbound M^{-1/2} \lambda^{-1}}_{\text{estimation}} + \underbrace{\lambda^{\omega_\rho - 1/2}}_{\text{approximation}}\Big],\label{eq:complexity:conditional_mean:prompt_decomp}
\end{align}
where $C(\rho_{Y, Z})$ is a constant that depends on the prompting measure $\rho_{Y, Z}$ and the choice of kernel. Optimizing the bound yields $\lambda \equiv \lambda_M \sim M^{-1/\p{2\omega_\rho + 1}}$, which ultimately leads to the convergence rate (notice the square) in~\eqref{eq:complexity:conditional_mean:prompt}. We comment that the choice to control the error in $\hg_\rho$ in $\calG$-norm comes from the vector-valued regression framework, in which the output space of the target function always lies in $\ltwo(Q_X; \calG)$. In isolation, the mean squared error of $\hg_\rho$ can be controlled both in $\ltwo(\rho_Z)$-norm as well as interpolation norms in between $\ltwo(\rho_Z)$ and $\calG$ (see \citet{fischer2020sobolev}, for instance). Indeed, when applying the decomposition~\eqref{eq:complexity:conditional_mean:prompt_decomp} in $\ltwo(\rho_Z)$-norm, \citet[Lemma 3]{smale2007learning} show that the approximation error decays as $\lambda^{\omega_\rho}$ (instead of $\lambda^{\omega_\rho - 1/2}$). In this case, the optimum is achieved at $\lambda_M \sim M^{-1/\p{2\omega_\rho + 2}}$, so that $\norm{\hg_\rho - g_\rho}_{\ltwo(\rho_Z)}^2$ enjoys a convergence rate of $M^{-\omega_\rho / (\omega_\rho + 1)}$.

\subsubsection{Completing the Proof}\label{sec:a:complexity:conditional_mean:final}
We may now prove \Cref{thm:complexity_vvkr}. 
Next, we place the requisite conditions on $\beta$, given eigendecay assumptions on $\Top_X$ and $\Top_Z$, and singular decay assumptions on $\M_{Z|X}$ (see \Cref{sec:a:background:compact} for a review of these operator decompositions). Under these assumptions, we will have that all operators will have a countably infinite number of non-zero eigenvalues/singular values.
\begin{assumption}[Eigendecay and Singular Decay]\label{asm:eigendecay}
    Let the eigenvalues of $\Top_X$, eigenvalues of $\Top_Z$, and singular values of $\M_{Z|X}$ be given by $\br{\mu_{X, i}}_{i=1}^\infty$, $\br{\mu_{Z, i}}_{i=1}^\infty$, and $\br{\sigma_i}_{i=1}^\infty$, respectively. There exist positive constants $c, C, \gamma_X, \gamma_Z$, and $\gamma_{X, Z}$ such that for all $i = 1, 2, \ldots$, we have the inclusions
    \begin{align*}
        \mu_{X, i} \in \sbr{ci^{-\gamma_X}, Ci^{-\gamma_X}}, \mu_{Z, i} \in \sbr{ci^{-\gamma_Z}, Ci^{-\gamma_Z}}, \text{ and } \sigma_{i} \in \sbr{ci^{-\gamma_{X, Z}}, Ci^{-\gamma_{X, Z}}}.
    \end{align*}
\end{assumption}
\begin{assumption}[Basis Alignment]\label{asm:alignment}
    There exists a finite index $m \in \mathbb{N}$ and a permutation $\pi: [m] \rightarrow [m]$ such that the operator $\M_{Z|X}$ admits the singular value decomposition
    \begin{align*}
        \M_{Z|X} = \sum_{i=1}^{m} \sigma_{\pi(i)} [e_{Z, i}]_Z \otimes [e_{X, i}]_Z + \sum_{j = m + 1}^\infty \sigma_{i} [e_{Z, i}]_Z \otimes [e_{X, i}]_Z.
    \end{align*}
\end{assumption}
\Cref{asm:alignment} allows us to reason about the finiteness of the Hilbert-Schmidt norm $\norm{\Top_X^{-\beta/2}\M_{Z|X} \Top_Z^{1/2}}_{\HS(\ltwo(Q_Z), \ltwo(Q_X))}^2$ based on the eigendecays of the various operators introduced in \Cref{asm:eigendecay}. These will imply a maximal value of the source condition constant $\beta$.
\begin{lemma}\label{lem:eigendecay}
    Under \Cref{asm:eigendecay} and \Cref{asm:alignment}, it holds that $\norm{F_\star}_\beta < +\infty$ if and only if
    \begin{align}
        \beta < \frac{2\gamma_{X, Z} + \gamma_Z - 1}{\gamma_X}.\label{eq:source_eigendecay}
    \end{align}
\end{lemma}
\begin{proof}
    For ease of presentation, we extend the permutation $\pi$ from \Cref{asm:alignment} so that $\pi(i) = i$ for all $i \geq m+1$. Applying the result from \Cref{prop:complexity:hsnorm}, and using the eigenbases of $\Top_X$ and $\Top_Z$, we see that
    \begin{align*}
        \norm{F_\star}_\beta^2 &= \norm{\Top_X^{-\beta/2}\M_{Z|X} \Top_Z^{1/2}}_{\HS(\ltwo(Q_Z), \ltwo(Q_X))}^2\\
        &\geq \frac{c}{C^{\beta}} \sum_{i=1}^\infty \sum_{j=1}^\infty j^{-\gamma_Z} i^{\beta \gamma_X} \ip{\M_{Z|X}[e_{Z, j}]_Z, [e_{X, i}]_X}_{\ltwo(Q_X)}^2 &\text{(\Cref{asm:eigendecay})}\\
        &= \frac{c}{C^{\beta}} \sum_{i=1}^\infty \sum_{j=1}^\infty \sum_{k=1}^\infty \sum_{l=1}^\infty j^{-\gamma_Z} i^{\beta \gamma_X} \sigma_{\pi(k)}^2\ip{[e_{Z, k}]_Z, [e_{Z, j}]_Z}_{\ltwo(Q_Z)}\ip{[e_{X, k}]_X, [e_{X, i}]_X}_{\ltwo(Q_X)} \\
        &\quad \times \ip{[e_{Z, l}]_Z, [e_{Z, j}]_Z}_{\ltwo(Q_Z)}\ip{[e_{X, l}]_X, [e_{X, i}]_X}_{\ltwo(Q_X)} \notag\\
        &= \frac{c}{C^{\beta}} \sbr{\sum_{i=1}^m  i^{\beta \gamma_X-\gamma_Z} \sigma_{\pi(i)}^2 +  \sum_{i=m+1}^\infty i^{\beta \gamma_X -\gamma_Z}\sigma_{i}^2}&\text{(\Cref{asm:alignment})}\\
        &\geq \frac{c}{C^{\beta}} \sbr{\sum_{i=1}^m  i^{\beta \gamma_X-\gamma_Z} \sigma_{\pi(i)}^2 +  c\sum_{i=m+1}^\infty i^{\beta \gamma_X -\gamma_Z - 2\gamma_{X, Z}}},&\text{(\Cref{asm:eigendecay})}
    \end{align*}
    where the rightmost term is finite only if~\eqref{eq:source_eigendecay} holds. Arguing similarly for the upper bound, we have that
    \begin{align*}
        \norm{F_\star}_\beta^2 \leq \frac{C}{c^{\beta}} \sbr{\sum_{i=1}^m  i^{\beta \gamma_X-\gamma_Z} \sigma_{\pi(i)}^2 +  C\sum_{i=m+1}^\infty i^{\beta \gamma_X -\gamma_Z - 2\gamma_{X, Z}}},
    \end{align*}
    where we may claim that $\norm{F_\star}_\beta^2 < +\infty$ if~\eqref{eq:source_eigendecay} holds.
\end{proof}
We can now wrap together the results of this section. Recalling the estimator $\Fhat \equiv \Fhat_{\lambda}$ based on vector-valued spectral regularization learning, described in \Cref{sec:a:background:rkhs}. The well-specified case refers to the condition that $\beta \geq 1$, indicating that the RKHS in which $\Fhat$ is learned does indeed contain $F_\star$. When $\beta < 1$, we require more sophisticated tools, namely, vector-valued interpolation spaces. In both cases, after establishing the results above, we capture the sample complexity via \Cref{thm:vvkr} from \Cref{sec:a:background:rkhs}.

\myparagraph{Well-Specified Case}
Under \Cref{asm:eigendecay} and \Cref{asm:alignment}, this implies via \Cref{lem:eigendecay} that
\begin{align}
    1 \leq \beta = \p{\frac{2\gamma_{X, Z} + \gamma_Z - 1}{\gamma_X}}^{t} < \frac{2\gamma_{X, Z} + \gamma_Z - 1}{\gamma_X}, \text{ for } t \in [0, 1).\label{eq:complexity_q_alpha1}
\end{align}
Thus, we may use the parameter $t \in [0, 1)$ to measure the degree to which the upper bound is saturated. This yields the following result, which reflects \Cref{thm:complexity_vvkr} from the main text. To state the result, define the quantity
\begin{align}
    q(t) = (2\gamma_{X, Z} + \gamma_Z - 1)^{t} \gamma_X^{1-t}\label{eq:complexity_q_alpha2}
\end{align}
and observe the following, which is an immediate consequence of \Cref{lem:complexity:conditional_mean:decomp}, \Cref{thm:vvkr}, and the formula~\eqref{eq:complexity_q_alpha1}. Note that the constant $p$ in \Cref{thm:vvkr} refers to $1/\gamma_X$ in the notation of this section.
\begin{theorem}\label{thm:complexity:conditional_mean}
    Consider failure probability $\delta \in (0, 1]$. Let \Cref{asm:vvkrr:downstream}, \Cref{asm:eigendecay}, \Cref{asm:alignment}, and the conditions of \Cref{thm:vvkr} hold with $\norm{\Top_X^{-1/2}\M_{Z|X} \Top_Z^{1/2}}_{\HS(\ltwo(Q_Z), \ltwo(Q_X))}^2 < +\infty$. Then, for $\heta_\rho$ defined via~\eqref{eq:complexity:conditonal_mean:estimator}, there exist a constants $t \in [0, 1)$ and $C \geq 0$ such that with probability at least $1 - \delta$, 
    \begin{align*}
        \norm{\heta_\rho - \eta_\rho}_{\ltwo(Q_X)}^2 &\lesssim \polylog(1/\delta) \sbr{N^{-\frac{q(t)}{q(t) + 1}} +\fstarbound^2 \norm{\M_{Z|X} \Top_Z^{1/2}}_{\HS}^2 M^{-\frac{2\omega_\rho - 1}{2\omega_\rho+1}}}
    \end{align*}
    for all $N \geq C\polylog(N/\delta)$. where $\norm{\cdot}_{\HS} = \norm{\cdot}_{\HS(\ltwo(Q_Z), \ltwo(Q_X))}$.
\end{theorem}
The term $\norm{\M_{Z|X} \Top_Z^{1/2}}_{\HS}^2$ is equal (via \Cref{prop:complexity:hsnorm}) to the $\norm{F_\star}^2_{\ltwo(Q_X; \calG)}$ term from \Cref{lem:complexity:conditional_mean:decomp}, and is rendered (along with $\fstarbound^2$ as the constant $C(Q_{X, Z})$ in \Cref{thm:complexity_vvkr}. 

\myparagraph{Mis-Specified Case}
The first inequality of~\eqref{eq:complexity_q_alpha1} holds only when $F_\star$ is well-specified, or contained in the vector-valued RKHS used in the estimation procedure that defines~\eqref{eq:background:rkhs:estimate}. We may employ the interpolation space machinery from \Cref{sec:a:background:rkhs} to achieve a convergence guarantee in this setting. Recall the constant $\alpha \in [1/\gamma_X, 1]$ shown in \Cref{asm:rkhs:source}, which is associated to the continuous embedding $\I_X^{\alpha, \infty}: [\calH]^\alpha \hookrightarrow \linf(Q_X)$. This constant describes the RKHS itself, and not the specific target function $F_\star$. The rate of \Cref{thm:complexity:conditional_mean} may still be achieved for function classes that are ``not too mis-specified'' in the sense of Case 1 from \Cref{thm:vvkr}. The inequality~\eqref{eq:source_eigendecay} provides a sufficient condition for Case 2, that is, when $\beta + 1/\gamma_X \leq \alpha$. Indeed, 
\begin{align}
    \frac{2\gamma_{X, Z} + \gamma_Z}{\gamma_X} \leq \alpha \implies \beta + 1/\gamma_X \overset{\eqref{eq:source_eigendecay}}{<} \frac{2\gamma_{X, Z} + \gamma_Z}{\gamma_X} \leq \alpha. \label{eq:misspecified}
\end{align}
The left-hand side may also be phrased differently as $2\gamma_{X, Z} + \gamma_Z \leq \alpha \gamma_X$. Thus, we may interpret $\alpha \gamma_X \in [1, \gamma_X]$ as a parameter that controls the mis-specification threshold. Concretely, it becomes easier for $F_\star$ to be mis-specified when: $\gamma_{X, Z}$ is low ($(X, Z)$ are highly dependent), $\gamma_Z$ is low (the effective dimension of $Z$ is large), or $\gamma_X$ is high (the effective dimension of the input $X$ is small). Under the sufficient condition~\eqref{eq:misspecified}, along with \Cref{asm:eigendecay} and \Cref{asm:alignment}, the best upper bound on the convergence rate in the current mis-specification model (see \Cref{thm:vvkr}, Case 2), is then
\begin{align*}
    \norm{\heta_\rho - \eta_\rho}_{\ltwo(Q_X)}^2 &\lesssim \polylog(1/\delta) \sbr{N^{-\frac{2\gamma_{X, Z} + \gamma_Z - 1}{\alpha \gamma_X}} +\fstarbound^2 \norm{\M_{Z|X} \Top_Z^{1/2}}_{\HS}^2 M^{-\frac{2\omega_\rho - 1}{2\omega_\rho+1}}}
\end{align*}
for $N$ sufficiently large.

\subsection{Information Density Approach}\label{sec:a:complexity:rn_derivative}
This approach is based on the RHS of~\eqref{eq:est:ts2} and yields the result of \Cref{thm:complexity_rn_derivative}. 
Here, we assume that during the pre-training phase, the user produces an estimated function $\Rhat$, which is an element of a reproducing kernel Hilbert space (RKHS). Unlike in \Cref{sec:a:complexity:conditional_mean}, where we approximated $g_\rho$ using a function $\hg_\rho$ (which aligns with the conditional mean viewpoint), the information density viewpoint in this section warrants estimating the mean of a function under $\rho_{Y, Z}$ directly, using samples $(Y_1, Z_1), \ldots, (Y_M, Z_M) \iidsim \rho_Z$. It is also important to point out a slight difference in the sampling model for the pre-training data. In order to define the estimate~\eqref{eq:background:rkhs:rn_estimate} for our method of choice (and similar Radon-Nikodym derivative estimation techniques), it is typically assumed that we observe data from both distributions in the ratio. In the case of $Q_{X, Z}$ and $Q_X \otimes Q_Z$, this corresponds to observing $\Np$ paired examples and $\Nu$ unpaired examples such that $N = \Np + \Nu$. For simplicity, we assume that $\Np = \Nu = N/2$, but remark that the regime in which $\Nu \gg \Np$ is an interesting and practically relevant model for future investigations.

\myparagraph{Setup}
Let $\calS$ denote a separable reproducing kernel Hilbert space (RKHS) of real-valued functions on $\X \times \Z$, with canonical feature map $\varphi: \X \times \Z \rightarrow \R$ and reproducing kernel $\kappa: (\X \times \Z) \times (\X \times \Z) \rightarrow \R$. 
We will express the error in terms of the RKHS norm difference $\norm{\Rhat - \Rsans}_{\calS}^2$, among other terms that capture a notion of ``distribution mismatch'' between the prompting marginal $\rho_Z$ and the pre-training marginal $Q_Z$. This may also be interpreted as another instance of prompt bias. This error occurs because at prompting time, the user does not necessarily have any data drawn from $Q_Z$. As before, we maintain $\sup\br{\kappa(\x, \z, \x', \z'): (\x, \z), (\x', \z') \in \X \times \Z} \leq \kappa_{\max{}}$.

Recall that the true $\Rsans$ is a kernel for the conditional mean operator when integrated under $Q_Z$ (see \Cref{lem:info_density}), but can also be related via \Cref{lem:info_density2} to the marginal distribution $\rho_Z$:
\begin{align*}
    \eta_\rho(\x) = \E{\rho_{Y, Z}}{\fstar(Y) \Rsans(\x, Z)} + \int_{\Z} g_\rho(\z)\Rsans(\x, \z) \p{\d Q_Z(\z) - \d \rho_Z(\z)}.
\end{align*}
This motivates the approximation $\hrho_{Y, Z}$ expressed directly in terms of the prompt distribution, and the estimator
\begin{align}
    \heta_\rho(\x) = \Ex_{\hrho_{Y, Z}}[\fstar(Y) \Rhat(\x, Z)].\label{eq:complexity:rn_derivative:estimate}
\end{align}
Below, we consider the empirical measure
\begin{align}
    \hrho_{Y, Z} = \frac{1}{M}\sum_{j=1}^M \delta_{(Y_j, Z_j)}\label{eq:complexity:rn_derivative}
\end{align}
so that for fixed $\x \in \X$,~\eqref{eq:complexity:rn_derivative:estimate} reduces to a sample mean.

\subsubsection{Decomposing the Global Error}
The estimation error decomposition below will take the two differences into account: between the marginal distributions $Q_Z$ and $\rho_Z$ and between the joint distribution $\hrho_{Y, Z}$ and $\rho_{Y, Z}$. For the latter, we will define random variables that take values in a Hilbert space (specifically, $\ltwo(Q_X)$). This will allow for controlling deviations between $\hrho_{Y, Z}$ and $\rho_{Y, Z}$ directly for the test functions being integrated. Define the independent and identically random variables $W_1, \ldots, W_M$ by
\begin{align*}
    W_j := \fstar(Y_j)\Rsans(\cdot, Z_j),
\end{align*}
and the element of $\ltwo(Q_X)$ (interpreted as the expectation) $\E{\rho_{Y, Z}}{W_1}: \x \mapsto \E{\rho_{Y, Z}}{\fstar(Y_1)\Rsans(\x, Z_1)}$. 
\begin{lemma}[Error Decomposition]\label{lem:est:rhat}
    Assume the following conditions.
    \begin{itemize}
        \item $\rho_Z \ll Q_Z$ with $Q_Z$-square integrable Radon-Nikodym derivative (i.e.~$\chi^2(\rho_Z \Vert Q_Z) < +\infty$).
        \item $\Rsans$ is contained in $\ltwo(Q_X \otimes \rho_Z)$ and $\ltwo(Q_X \otimes Q_Z)$.
    \end{itemize}
    Then, it holds that
    \begin{align}
        \norm{\heta_\rho - \eta_\rho}_{\ltwo(Q_X)}^2 &\leq 3\fstarbound^2 \big(\kappa_{\max{}}^2 \norm{\Rhat - \Rsans}_{\calS}^2 + \norm{\Rsans}_{\ltwo(Q_X \otimes Q_Z)}^2 \chi^2(\rho_Z \Vert Q_Z)\big) \label{eq:est:rhat}\\
        &\quad + 3\norm{\tfrac{1}{M}\textstyle\sum_{j=1}^M W_j - \E{\rho_{Y, Z}}{W_1}}_{\ltwo(Q_X)}^2.\notag
    \end{align}
\end{lemma}
\begin{proof}
    Using \Cref{lem:info_density2}, we have that for $Q_X$-almost all $\x \in \X$,
    \begin{align*}
        \heta_\rho(\x) - \eta_\rho(\x)  &= \Ex_{\hrho_{Y, Z}}[\fstar(Y) \Rhat(\x, Z)] - \E{\rho_{Y, Z}}{\fstar(Y) \Rsans(\x, Z)}\\
        &\quad + \int_{\Z} g_\rho(\z)\Rsans(\x, \z) \p{\d Q_Z(\z) - \d \rho_Z(\z)}\\
        &= \Ex_{\hrho_{Y, Z}}[\fstar(Y) \ipsmall{\varphi(\x, Z), \Rhat - \Rsans}] \\
        &\quad +  \int_{\Y \times \Z} \fstar(\y) \Rsans(\x, \z) \p{\d\hrho_{Y, Z}(\y, \z) - \d\rho_{Y, Z}(\y, \z)}\\
        &\quad + \int_{\Z} g_\rho(\z)\Rsans(\x, \z) \p{\d Q_Z(\z) - \d \rho_Z(\z)}.
    \end{align*}
    Then, we have that
    \begin{align}
        \norm{\heta_\rho - \eta_\rho}_{\ltwo(Q_X)}^2 &\leq 3 \E{Q_X}{\p{\Ex_{\hrho_{Y, Z}}[\fstar(Y) \ipsmall{\varphi(X, Z), \Rhat - \Rsans}]^2}} \label{eq:complexity:rn_decomp1}\\
        &\quad + 3\int_{\X} \p{\int_{\Y \times \Z} \fstar(\y) \Rsans(\x, \z) \p{\d\hrho_{Y, Z}(\y, \z) - \d\rho_{Y, Z}(\y, \z)}}^2 \d Q_X(\x)  \label{eq:complexity:rn_decomp2}\\
        &\quad + 3\int_{\X} \p{\int_{\Z} g_\rho(\z)\Rsans(\x, \z) \p{\d Q_Z(\z) - \d \rho_Z(\z)}}^2 \d Q_X(\x).  \label{eq:complexity:rn_decomp3}
    \end{align}
    To control~\eqref{eq:complexity:rn_decomp1}, apply boundedness to achieve
    \begin{align*}
        \E{Q_X}{\p{\Ex_{\hrho_{Y, Z}}[\fstar(Y) \ipsmall{\varphi(X, Z), \Rhat - \Rsans}]^2}} \leq \fstarbound^2 \kappa_{\max{}}^2 \norm{\Rhat - \Rsans}_{\calS}^2.
    \end{align*}
    For~\eqref{eq:complexity:rn_decomp2}, the term is equal to $\norm{\tfrac{1}{M}\textstyle\sum_{j=1}^M W_j - \E{\rho_{Y, Z}}{W_1}}_{\ltwo(Q_X)}^2$ by definition of $W_1, \ldots, W_M$.
    For \eqref{eq:complexity:rn_decomp3}, we use that $\rho_Z \ll Q_Z$ and $\norm{g_\rho}_\infty \leq \fstarbound$ and apply the Cauchy-Schwarz inequality on $\ltwo(Q_Z)$ so that
    \begin{align*}
        &\p{\int_{\Z} g_\rho(\z)\Rsans(\x, \z) \p{\d Q_Z(\z) - \d \rho_Z(\z)}}^2 \\
        &= \p{\int_{\Z} g_\rho(\z)\Rsans(\x, \z) \p{1 - \frac{\d \rho_Z}{\d Q_Z}(\z)} \d Q_Z(\z) }^2\\
        &\leq \norm{\fstar}^2 \norm{\Rsans(\x, \cdot)}_{\ltwo(Q_Z)}^2 \underbrace{\int_Z \p{1 - \frac{\d \rho_Z}{\d Q_Z}(\z)}^2 \d Q_Z(\z)}_{\chi^2(\rho_Z \Vert Q_Z)}.
    \end{align*}
    Taking the expectation over $Q_X$ gives $\Ex_{Q_X}\norm{\Rsans(X, \cdot)}_{\ltwo(Q_Z)}^2 = \norm{\Rsans}_{\ltwo(Q_X \otimes Q_Z)}^2$ and completes the proof.
\end{proof}

Given the decomposition shown in \Cref{lem:est:rhat}, it remains to bound both the error term $\norm{\Rhat - \Rsans}_{\calS}^2$ regarding the estimated Radon-Nikodym derivative $\Rhat$, and the approximation term $\norm{\tfrac{1}{M}\textstyle\sum_{j=1}^M W_j - \E{\rho_{Y, Z}}{W_1}}_{\ltwo(Q_X)}^2$. We will employ \Cref{coro:background:rn} to this end. Unlike the arguments of \Cref{sec:a:complexity:conditional_mean}, there is only a single kernel regularized learning algorithm at play, that is, for the estimation of $\Rhat$.
We proceed to interpret the source condition \Cref{asm:main:estimation:rn:source}.

\subsubsection{Interpreting the Source Condition}\label{sec:a:complexity:rn_derivative:source}
To proceed, we introduce some notation related to $\ltwo(Q_X \otimes Q_Z)$ and the RKHS $\calS$. These objects are also introduced in \Cref{sec:a:background:rkhs}, so we review their properties briefly. Let $[h]_\sim$ index the equivalence class in $\ltwo(Q_X \otimes Q_Z)$ for a square-integrable function $h: \X \times \Z \rightarrow \R$. This indexing can also be identified with an \emph{embedding operator} $\I_{X, Z}: \calS \rightarrow \ltwo(Q_X \otimes Q_Z)$, which is Hilbert-Schmidt under the boundedness of the kernel $\kappa$ by $\kappa_{\max{}}$. Letting $\Sop_{X, Z} = \I_{X, Z}^*: \ltwo(Q_X \otimes Q_Z) \rightarrow \calS$ be its adjoint, we have that $\Iop_{X, Z} \Sop_{X,Z}: \ltwo(Q_X \otimes Q_Z) \rightarrow \ltwo(Q_X \otimes Q_Z)$ and $\Sop_{X,Z} \Iop_{X, Z}: \calS \rightarrow \calS$ are compact, trace class operators. These form the analogs of $(\Top_X, \Top_Z)$ and $(\Cop_X, \Cop_Z)$, respectively, from \Cref{sec:a:background:rkhs}.
Via \Cref{thm:eigen}, we write the eigendecomposition
\begin{align}
    \Iop_{X, Z} \Sop_{X,Z} = \sum_{i \in I} \mu_i \ip{\cdot, [e_i]_\sim}_{\ltwo(Q_X \otimes Q_Z)} [e_i]_\sim,\label{eq:main:estimation:rn:Top}
\end{align}
where we may take each representative $e_i$ as an element of $\calS$ \citep[Lemma 2.12]{steinwart2012mercers}. Then, we also have that
\begin{align}
    \Sop_{X,Z} \Iop_{X, Z} = \sum_{i \in I} \mu_i \ipsmall{\cdot, \mu_i^{1/2} e_i}_{\calS} \mu_i^{1/2} e_i,\label{eq:main:estimation:rn:Cop}
\end{align}
These constructions (along with \Cref{prop:lancaster}) give us the following relationship between the Hilbert-Schmidt norm of the conditional mean operator $\M_{Z|X}$ and the Radon-Nikodym derivative under the condition \Cref{asm:main:estimation:rn:source}. In fact, finiteness follows from the source condition itself and boundedness of the kernel.
\begin{lemma}\label{lem:main:estimation:source}
    Under \Cref{asm:main:estimation:rn:source}, it holds that
    \begin{align*}
         \norm{\M_{Z|X}}_{\HS(\ltwo(Q_Z), \ltwo(Q_X))}^2 = \norm{\I_{X, Z} \Rsans}_{\ltwo(Q_X \otimes Q_Z)}^2 = \sum_{i \in I} \mu_i^{2\beta+1} \ipsmall{\Ssans_{Q_{X, Z}}, \mu_i^{1/2} e_i}_{\calS}^2.
    \end{align*}
\end{lemma}
\begin{proof}
    Without loss of generality, assume that $\Rsans \in \Null(\I_{X, Z})^\top$ (as the component in $\Null(\I_{X, Z})$ will be excluded from the norm calculation anyway). We expand the expression for $\Rsans$ appearing in \Cref{asm:main:estimation:rn:source} on an ONB of $\Null(\I_{X, Z})^\top$. To do so, combine~\eqref{eq:main:estimation:rn:Top} and~\eqref{eq:main:estimation:rn:Cop} to introduce the singular value decomposition
    \begin{align*}
        \I_{X, Z} = \sum_{i \in I} \mu_i^{1/2} \ipsmall{\cdot, \mu_i^{1/2} e_i}_{\calS} [e_i]_\sim.
    \end{align*}
    Then, it holds under \Cref{asm:main:estimation:rn:source} that
    \begin{align*}
        \Rsans =  \sum_{i \in I} \mu_i^\beta \ipsmall{\Ssans_{Q_{X, Z}}, \mu_i^{1/2} e_i}_{\calS} \mu_i^{1/2} e_i \text{ and } \I_{X, Z} \Rsans = \sum_{i \in I} \mu_i^{\beta+1/2} \ipsmall{\Ssans_{Q_{X, Z}}, \mu_i^{1/2} e_i}_{\calS} [e_i]_\sim.
    \end{align*}
    Using that $([e_i]_\sim)_{i \in I}$ is an orthonormal system, we may use the second expression to perform the computation.
\end{proof}

To make use of \Cref{lem:main:estimation:source}, we now interpret $\beta$ in terms of eigendecay exponents of the operators in question. 
\begin{assumption}[Eigendecay and Singular Decay]\label{asm:rn:eigendecay}
    Let the eigenvalues of $\Iop_{X, Z}\Sop_{X, Z}$ and singular values of $\M_{Z|X}$ be given by $\br{\mu_{i}}_{i=1}^\infty$ and $\br{\sigma_i}_{i=1}^\infty$, respectively. There exist positive constants $c$, $C$, $\alpha > 1$, and $\gamma_{X, Z} > 1/2$ such that for all $i = 1, 2, \ldots$, we have the inclusions
    \begin{align*}
        \mu_{i} \leq \sbr{ci^{-\alpha}, Ci^{-\alpha}}. \text{ and } \sigma_{i} \in \sbr{ci^{-\gamma_{X, Z}}, Ci^{-\gamma_{X, Z}}}.
    \end{align*}
\end{assumption}
The following relationship holds over an interval in $\beta$. We explicitly account for the dependence of $\Ssans_{Q_{X, Z}}$ on $\beta$ when it comes to satisfying \Cref{asm:main:estimation:rn:source}.
\begin{proposition}\label{prop:complexity:rn:eigendecay}
    Let \Cref{asm:rn:eigendecay} be satisfied. Let \Cref{asm:main:estimation:rn:source} be satisfied for all $0 \leq \beta \leq \bar{\beta} < +\infty$, where $\Ssans_{Q_{X, Z}} \equiv \Ssans_{Q_{X, Z}}(\beta)$ is bounded in $\calS$-norm by $\bar{B}$ for all $\beta \in [0, \bar{\beta}]$. Then, we have that
    \begin{align*}
        \gamma_{X, Z} &\geq \frac{1}{2}\sbr{\frac{(\bar{B}^2C^{2\beta + 1} + c^2)\alpha(2\beta+1) - 1}{\bar{B}^2C^{2\beta + 1}\alpha(2\beta+1)}}.
    \end{align*}
\end{proposition}
\begin{proof}
    Write
    \begin{align}
        \norm{\M_{Z|X}}_{\HS(\ltwo(Q_Z), \ltwo(Q_X))}^2 = \sum_{i \in 1}^\infty \sigma_i^2 &= \sum_{i =1}^\infty \mu_i^{2\beta+1} \ipsmall{\Ssans_{Q_{X, Z}}, \mu_i^{1/2} e_i}_{\calS}^2 \notag \\
        &= \norm{\Ssans_{Q_{X, Z}}}_\calS^2 \sum_{i =1}^\infty \mu_i^{2\beta+1} \ipsmall{\Ssans_{Q_{X, Z}} / \norm{\Ssans_{Q_{X, Z}}}_\calS, \mu_i^{1/2} e_i}_{\calS}^2 \notag \\
        &\leq \bar{B}^2 \sum_{i =1}^\infty \mu_i^{2(\beta+1/2)}.\label{eq:complexity:rn:eigendecay1}
    \end{align}
    The right-hand side is finite, for all $\beta \geq 0$, as the $(\mu_i)_{i=1}^\infty$ sequence is associated to a trace class operator. Next, using that $\mu_i^{\beta+1/2} \leq C^{2\beta + 1}i^{-(\beta+1/2)\alpha}$, we use \Cref{lem:singular_decay} to upper bound~\eqref{eq:complexity:rn:eigendecay1} via
    \begin{align*}
        \sum_{i =1}^\infty \mu_i^{2(\beta+1/2)} \leq\frac{C^{2\beta + 1}(2\beta+1)\alpha}{(2\beta+1)\alpha - 1} = \frac{C^{2\beta + 1}}{1 - (2\beta+1)^{-1}\alpha^{-1}}.
    \end{align*}
    On the other hand, using \Cref{def:msc} and \Cref{lem:singular_decay}, the Hilbert-Schmidt norm is lower bounded via
    \begin{align*}
        \norm{\M_{Z|X}}_{\HS(\ltwo(Q_Z), \ltwo(Q_X))}^2 \geq \frac{c^2}{2\gamma_{X, Z} - 1}.
    \end{align*}
    Combining both bounds, we have
    \begin{align*}
        \frac{c^2}{2\gamma_{X, Z} - 1} \leq\frac{\bar{B}^2C^{2\beta + 1}}{1 - (2\beta+1)^{-1}\alpha^{-1}}.
    \end{align*}
    Inverting the bound gives the condition
    \begin{align*}
        \gamma_{X, Z} &\geq \frac{1}{2}\sbr{\frac{c^2}{\bar{B}^2C^{2\beta + 1}}\p{1 - \frac{1}{\alpha(2\beta+1)}} + 1}\\
        &= \frac{1}{2}\sbr{\frac{(\bar{B}^2C^{2\beta + 1} + c^2)\alpha(2\beta+1) - 1}{\bar{B}^2C^{2\beta + 1}\alpha(2\beta+1)}},
    \end{align*}
    the result as desired.
\end{proof}
From \Cref{prop:complexity:rn:eigendecay}, we consider the case in which $\alpha \rightarrow \infty$ (the data is finite-rank under independence), and derive the singular decay condition
\begin{align*}
    \gamma_{X, Z} &\geq \frac{1}{2}\p{\frac{\bar{B}^2C^{2\beta + 1} + c^2}{\bar{B}^2C^{2\beta + 1}}} > \frac{1}{2}
\end{align*}
for $c > 0$. While the relationship is not as direct as in the case of~\eqref{eq:complexity_q_alpha1}, we may still observe some regimes in which a ``maximally smooth'' target function boils down to an independence assumption. This holds intuitively as well, in the sense that $\Rsans \equiv 1$ holds $(Q_X \otimes Q_Z)$-almost surely if and only if $X$ and $Z$ are independent.

\subsubsection{Controlling the Prompting Term}
The term that relates $\hrho_{Y, Z}$ to $\rho_{Y, Z}$ is simply a measurement of the deviation of a sample mean from its population counterpart, within a Hilbert space. Thus, it is reasonable to assume an $O(1/M)$ scaling on this term. Below, we use the notation $(X_i', Z_i')$ to indicate a sample drawn from $Q_X \otimes Q_Z$, i.e., an unpaired example.
\begin{assumption}\label{asm:complexity:rn_derivative:ps}
    For constants $\delta \in (0, 1]$ and $M\geq 1$, there is an event $\mc{E}(\delta, M)$, which is independent of the pre-training data $\br{(X_i, Z_i)}_{i=1}^{N/2}, \br{(X_i', Z_i')}_{i=1}^{N/2}$, such that on $\mc{E}(\delta, M)$, 
    \begin{align}
        \bignorm{\tfrac{1}{M}\textstyle\sum_{j=1}^M W_j - \E{\rho_{Y, Z}}{W_1}}_{\ltwo(Q_X)}^2 \leq C_{\Rsans, \rho}(Q_X) \polylog(1/\delta) M^{-1},\label{eq:complexity:rn_derivative:ps_inequality}
    \end{align}
    where $C_{\Rsans, \rho}(Q_X)$ depends only on its arguments and $\fstar$, and is independent of $M$ and $\delta$. On $(\Omega, \msc{F}, \prob)$, the event $\mc{E}(\delta, M)$ occurs with probability at least $1 - \delta/2$.
\end{assumption}
The scaling shown in \Cref{asm:complexity:rn_derivative:ps} can be satisfied by placing a Bernstein-type condition on the random variable $W_1$ and applying, for instance, the Pinelis-Sahanenko inequality \citep{pinelis1986remarks}. Specifically, consider the case in which there are positive constants $\sigma, c > 0$ such that
\begin{align*}
    \sum_{j=1}^M \Ex_{\rho_{Y, Z}}\norm{\tfrac{1}{M} W_j - \tfrac{1}{M}\E{\rho_{Y, Z}}{W_1}}_{\ltwo(Q_X)}^{q} \leq \frac{q!}{2} \sigma^2 c^{q-2}
\end{align*}
for all $q \geq 2$. Then,~\eqref{eq:complexity:rn_derivative:ps_inequality} is satisfied, wherein the scalars  $\sigma$ and $c$ will scale as $1/M$, and have additional constants that depend on $\fstar$, $\Rsans$, $\rho_{Y, Z}$, and $Q_X$ (but not $Q_Z$ or $Q_{X, Z}$). This generates the constant $C_{\Rsans, \rho}(Q_X)$ above.

\subsubsection{Completing the Proof}\label{sec:a:complexity:rn_derivative:final}

\myparagraph{Well-Specified Case}
Because \Cref{prop:complexity:rn:eigendecay} yields an inexact relationship between the singular decay exponent $\gamma_{X, Z}$ and the source condition constant $\beta$, we maintain the statement of the result in terms of this constant. The following result comes as an immediate consequence of \Cref{lem:est:rhat} and \Cref{coro:background:rn}.
\begin{theorem}\label{thm:complexity:rn_derivative}
    Consider failure probability $\delta \in (0, 1]$. Assume that the conditions of \Cref{lem:est:rhat} are satisfied and that $N$ is large enough such that the conditions of \Cref{coro:background:rn} are satisfied, in addition to \Cref{asm:complexity:rn_derivative:ps}. Define
    \begin{align*}
        K_{\max{}} := 1 + (4\kappa_{\max}^2 + \kappa_{\max})^2.
    \end{align*}
    Then, with probability at least $1 - \delta$, it holds that
    \begin{align*}
        \norm{\heta_\rho - \eta_\rho}_{\ltwo(Q_X)}^2 &\lesssim \polylog(1/\delta) \sbr{K_{\max{}}^{\frac{\beta+2}{\beta+1}} N^{-\frac{\beta}{\beta + 1}}+ C_{\Rsans, \rho}(Q_X)M^{-1}} + \chi^2(\rho_Z \Vert Q_Z),
    \end{align*}
    where $C_{\Rsans, \rho}(Q_X)$ depends only on its arguments and $\fstar$, and not $M$ or $\delta$.
\end{theorem}
The constant $C_{\Rsans, \rho}(Q_X)$ appears directly from \Cref{asm:complexity:rn_derivative:ps}.

\myparagraph{Mis-Specified Case}
As mentioned in \Cref{sec:a:background:rkhs}, the mis-specified case ($\Rsans \notin \calS$) for Radon-Nikodym derivative estimation problems is less understood than the mis-specified case for real-valued and vector-valued nonparametric regression. We intend here to highlight the overall decomposition of error, for which such results could be plugged in as well.

\subsection{Distribution Shift}\label{sec:a:complexity:shift}
The results of the previous two subsections provided bounds in high probability on the term $\norm{\heta_\rho - \eta_\rho}_{\ltwo(Q_X)}^2$. Returning to the original error decomposition of~\eqref{eq:theory:decomp1}, we would like to relate this to a similar bound on $\norm{\heta_\rho - \eta_\rho}_{\ltwo(P_X)}^2$. We collect two general techniques for performing this change of measure, which lead to either a multiplicative or additive error depending on the assumptions the user is willing to make.
\begin{lemma}[Distribution Shift]\label{lem:distribution_shift}
    Assume that $P_X$ and $Q_X$ have densities $p_X$ and $q_X$ with respect to a common dominating measure $\nu_X$ on the measurable space $(\X, \calB(\X))$, and define the \emph{total variation} metric
    \begin{align*}
        \tv(P_X, Q_X) := \int_\X \abs{p_X(\x) - q_X(\x)} \d \nu_X(\x).
    \end{align*}
    Then, for any $\eta:\X \rightarrow \R$ such that $[\eta]_X \in \ltwo(P_X) \cap \ltwo(Q_X)$ (see \Cref{sec:a:background:l2}), the following holds.
    \begin{itemize}
        \item If the essential supremum $\norm{\eta}_\infty := \inf\br{\sup_{A \in \calB(\X)} \sup_{\x \in A}\abs{\eta(\x)}: \nu_X(A^c) = 0}$ is finite, then we have the additive relation
        \begin{align}
            \norm{\eta}_{\ltwo(P_X)}^2 \leq \norm{\eta}_{\ltwo(Q_X)}^2 +  \norm{\eta}_\infty^2 \tv(P_X, Q_X).\label{eq:main:class:add}
        \end{align}
        \item If $Q_X \ll P_X$, and $\frac{\d Q_X}{\d P_X}(\x, \z) \leq B_{P, Q}$ for $P_X$-almost all $\x \in \X$, then we have the multiplicative relation
        \begin{align}
            \norm{\eta}_{\ltwo(P_X)}^2 \leq B_{P, Q} \norm{\eta}_{\ltwo(Q_X)}^2.\label{eq:main:class:multi}
        \end{align}
    \end{itemize}
\end{lemma}
\begin{proof}
    In the case of~\eqref{eq:main:class:add}, we apply H\"older's inequality to achieve
    \begin{align*}
        \norm{\eta}_{\ltwo(P_X)}^2 = \E{P_X}{\eta^2(X)} &= \E{Q_X}{\eta^2(X)} + \int_\X \eta^2(\x) \p{p_X(\x) - q_X(\x)} \d \nu_X(\x)\\
        &\leq \norm{\eta}_{\ltwo(Q_X)}^2 + \norm{\eta}_\infty^2  \int_\X \abs{p_X(\x) - q_X(\x)} \d \nu_X(\x)\\
        &= \norm{\eta}_{\ltwo(Q_X)}^2 + \norm{\eta}_\infty^2 \tv(P_X, Q_X),
    \end{align*}
   which proves the first claim. For~\eqref{eq:main:class:multi}, on the other hand, write
    \begin{align*}
        \norm{\eta}_{\ltwo(P_X)}^2 = \E{P_X}{\eta^2(X)} = \E{Q_X}{\eta^2(X)\tfrac{\d Q_X}{\d P_X}(X)} \leq B_{P, Q}  \norm{\eta}_{\ltwo(Q_X)}^2,
    \end{align*}
    proving the second claim and completing the proof.
\end{proof}
From \Cref{lem:distribution_shift} and the boundedness assumption $\abs{\fstar(\cdot)} \leq \fstarbound$, we alter~\eqref{eq:theory:decomp1} slightly to read
\begin{align}
    \norm{\etastar - \heta_\rho}_{\ltwo(P_X)}^2 \leq 2\norm{\etastar - \eta_\rho}_{\ltwo(P_X)}^2 + 2\norm{\eta_\rho - \heta_\rho}_{\red{\ltwo(Q_X)}}^2 + \red{4\fstarbound^2 \tv(P_X, Q_X)},\label{eq:complexity:error2}
\end{align}
and plug in the previous bounds on the $\norm{\eta_\rho - \heta_\rho}_{\ltwo(Q_X)}^2$ term for an overall result.

\subsection{From Regression to Classification}\label{sec:a:complexity:classification}
Throughout this appendix, we evaluated the quality of an estimated map $\heta_\rho: \X \rightarrow \R$ via its $\ltwo(Q_X)$ distance to some target predictor $\eta_\rho$. This goal was based on the error decomposition~\eqref{eq:complexity:error2}, which feeds into ultimate upper bound for $\norm{\heta_\rho - \etastar}^2_{\ltwo(P_X)}$,
where each term was controlled using the techniques of \Cref{sec:a:dependence}, \Cref{sec:a:complexity:conditional_mean}, and \Cref{sec:a:complexity:rn_derivative}.
In the case that $\fstar: \Y \rightarrow \R$ represents a classification or structured prediction problem (e.g.~$\fstar(\y) = \ind\br{\y = c}$ for class $c \in \Y$), it is of clear interest whether the control of mean squared error translates to risk guarantees for classification error. Establishing these guarantees, using the notion of a \emph{structure encoding loss function (SELF)} described in \citet[Section 13.2]{bach2024learning}, is the subject of this section.

Assume that $\Y$ is discrete, or that $\abs{\Y} < \infty$. We consider a loss function $\ell: \Y \times \Y \rightarrow \R$ and a regular conditional distribution $P_{Y|X}(\cdot |\x)$ (see \Cref{def:rcd}), under which $\ell(\cdot, \y)$ is integrable for all $(\x, \y) \in \X \times \Y$. The corresponding risk of any map $h: \X \rightarrow \Y$ will be denoted
\begin{align}
    \risk(h) = \E{P_{X, Y}}{\ell(Y, h(X))}.\label{eq:main:misclass}
\end{align}
There are a number of assumptions that mark the SELF framework.
\begin{assumption}[SELF Loss for Structured Prediction]\label{asm:self}
     Consider the existence of a Hilbert space $\calF$, and two mappings $\chi: \Y \rightarrow \calF$ and $\xi: \Y \rightarrow \calF$ which act as embeddings of objects in $\Y$. Then, assume that $\ell$ satisfies the equality
    \begin{align*}
        \ell(\y, \y') = \ip{\chi(\y), \xi(\y')}_\calF.
    \end{align*}
\end{assumption}
As of yet, no assumptions (such as being an RKHS) have been placed on $\calF$. 
Under \Cref{asm:self}, the Bayes optimal predictor (with respect to~\eqref{eq:main:misclass}, and not mean squared error) is given by
\begin{align*}
    h_\star(\x) \in \argmin_{\y' \in \Y} \sum_{\y \in \Y} \ell(\y, \y') P_{Y|X}(\y|\x),
\end{align*}
where ties can be broken arbitrarily. In other words, $h_\star \in \argmin_h \risk(h)$. Additionally, because $\Y$ is finite, we may take the expectation
\begin{align*}
    \sum_{\y \in \Y} \ell(\y, \y') P_{Y|X}(\y|\x) &= \sum_\Y \ip{\chi(\y), \xi(\y')}_\calF P_{Y|X}(\y|\x)\\
    &= \ip{\E{P_{X, Y}}{\chi(Y)|X}(\x), \xi(\y')}_\calF,
\end{align*}
which is only based on finite sums of vectors in $\calF$.
Next, we define the notation of a surrogate loss. To construct a predictor (e.g.~classifier), we consider a function $s: \X \rightarrow \calF$ called the \emph{score function} and a map $\dec: \calF \rightarrow \Y$ known as a \emph{decoder}. We will then define an integrable surrogate loss $L: \Y \times \calF \rightarrow \R$, for which we can define the risk
\begin{align}
    \risk^L(s) = \E{P_{X, Y}}{L(Y, s(X))}.\label{eq:main:surrogate}
\end{align}
We can then define the Bayes surrogate risk\footnote{We assume the map $\x \mapsto \inf_{h \in \calF} \E{P_{X, Y}}{L(Y, h)|X}(\x)$ to be measurable as a technical consideration.} as 
\begin{align*}
    \risk_\star^L = \E{P_X}{\inf_{h \in \calF} \E{P_{X, Y}}{L(Y, h)|X}}.
\end{align*}
The relationship between the surrogate risk~\eqref{eq:main:surrogate} and the true risk~\eqref{eq:main:misclass} for squared surrogates is given in the following result.
\begin{proposition}{\citep[Section 13.4.2]{bach2024learning}}\label{prop:decoding}
    Consider the surrogate loss and decoder given by
    \begin{align*}
        L(\y, s(\x)) := \norm{\xi(\y) - s(\x)}_\calF^2 \text{ and } \dec(h) \in \argmin_{\y \in \Y} \ip{\chi(\y), h}_\calF.
    \end{align*}
    Then, for any score function $s: \X \rightarrow \calF$, it holds that
    \begin{align*}
        \risk(\dec \circ s) - \risk(h_\star) \leq 2\sup_{\y \in \Y}\norm{\chi(\y)}_\calF \cdot \sqrt{\risk^L(s) -\risk_\star^L}.
    \end{align*}
\end{proposition}

We stated \Cref{prop:decoding} generally; we now map it to classification, the prototypical task associated with zero-shot prediction. 
Let $\Y = \br{1, \ldots, C}$, where $C$ denotes the number of classes (in contrast to the absolute constants in \Cref{thm:complexity_vvkr} and \Cref{thm:complexity_rn_derivative}). Then, we have that $\chi(\y)$ is the one-hot encoding in $\R^C$, whereas $\xi(\y)$ is the complement, that is, $\xi_j(\y) = 1 - \chi_j(\y)$ for $c = 1, \ldots, C$. Thus, their inner product generates the 0-1 loss
\begin{align*}
    \ell(\y, \y') = \ind\br{\y \neq \y'} = \ip{\chi(\y), \xi(\y')}_{\R^C}.
\end{align*}
Then, we immediately have that $\sup_{\y \in \Y}\norm{\chi(\y)}_\calF = 1$. It remains to determine the score function $s: \X \rightarrow \R^C$. Note that we used a function $\fstar$ to define~\eqref{eq:est:bayes} and~\eqref{eq:est:ts}; we will now use $C$ such functions $\fstar\pow{1}, \ldots, \fstar\pow{C}$ each defined by
\begin{align}
    \fstar\pow{c}(\y) = \xi_j(\y) = \ind\br{\y = j}\label{eq:class:xi_encoding}
\end{align}
which in turn gives us the individual mean squared error minimizers
\begin{align*}
    \etastar\pow{c}(\x) = \E{P_{Y, X}}{\fstar\pow{c}(Y)|X}(\x) = \P{P_{Y, X}}{Y = j|X}(\x).
\end{align*}
Finally, we may use any of the estimation strategies developed in \Cref{sec:a:complexity:conditional_mean} or \Cref{sec:a:complexity:rn_derivative} to produce estimators $\heta_\rho\pow{1}, \ldots, \heta_\rho\pow{C}$ (i.e.~the predicted probability per class) to give the score function
\begin{align}
    s(\x) := \p{\heta_\rho\pow{1}(\x), \ldots, \heta_\rho\pow{C}(\x)} \in \R^C.\label{eq:main:class:score}
\end{align}
Each $\heta_\rho\pow{c}$ is then associated to the conditional mean given by the prompt distribution, which we denote $g_\rho\pow{c}$.
As a final step, we use the classical relationship between mean squared prediction error and mean squared integrated error, as seen below.
\begin{corollary}\label{coro:main:classification}
    For the score function given in~\eqref{eq:main:class:score} and decoder given in \Cref{prop:decoding}, it holds that
    \begin{align*}
        \risk(\dec \circ s) - \risk(h_\star) \leq 2 \sqrt{\textstyle\sum_{j=1}^C \norm{\heta_\rho\pow{c} - \etastar\pow{c}}_{\ltwo(P_X)}^2}.
    \end{align*}
\end{corollary}
\begin{proof}
    Given \Cref{prop:decoding}, we need only show that
    \begin{align}
        \risk^L(s) -\risk_\star^L = \sum_{j=1}^C \norm{\heta_\rho\pow{c} - \etastar\pow{c}}_{\ltwo(P_X)}^2.\label{eq:main:class:wts}
    \end{align}
    First, note that for the score function $s$ given in~\eqref{eq:main:class:score}, it holds by~\eqref{eq:class:xi_encoding} that
    \begin{align*}
        L(\y, s(\x)) := \norm{\xi(\y) - s(\x)}_{\R^C}^2 = \sum_{j=1}^C (\fstar\pow{c}(\y) - \heta_\rho\pow{c}(\x))^2,
    \end{align*}
    and after taking the expectation over $P_{X, Y}$,
    \begin{align*}
        \risk^L(s) = \E{P_{X, Y}}{L(Y, s(X))} = \sum_{j=1}^C \E{P_{X, Y}}{(\fstar\pow{c}(Y) - \heta_\rho\pow{c}(X))^2}.
    \end{align*}
    Then, by the bias-variance decomposition for each $c = 1, \ldots, C$, it holds that
    \begin{align*}
        \underbrace{\sum_{j=1}^C \E{P_{X, Y}}{(\fstar\pow{c}(Y) - \heta_\rho\pow{c}(X))^2}}_{\risk^L(s)} = \sum_{j=1}^C \norm{\heta_\rho\pow{c} - \etastar\pow{c}}_{\ltwo(P_X)}^2 + \underbrace{\sum_{j=1}^C \E{P_{X, Y}}{(\fstar\pow{c}(Y) - \etastar\pow{c}(X))^2}}_{\risk_\star^L}.
    \end{align*}
    Rearranging terms gives~\eqref{eq:main:class:wts} and completes the proof.
\end{proof}
In particular, when applying the bound above to results of \Cref{thm:res_dep1}, \Cref{thm:complexity_vvkr}, and \Cref{thm:complexity_rn_derivative}, we derive a bound of the form
\begin{align*}
    \risk(\dec \circ s) - \risk(h_\star) &\lesssim \sqrt{C\E{P_Z}{I(X; Y|Z)} + \textstyle\sum_{j=1}^C \norm{g_\rho\pow{c} - g_{P_{Y, Z}}\pow{c}}_{\ltwo(P_Z)}^2 + C\tv(P_X, Q_X)}\\
    &\quad + \begin{cases}
        \sqrt{C}\polylog\p{C/\delta}\p{N^{-\frac{q(t)}{2(q(t) + 1)}} + M^{-\frac{2\omega_\rho - 1}{4\omega_\rho + 2}}} &\text{ (conditional mean)}\\
        \sqrt{C}\polylog\p{C/\delta} \p{N^{-\frac{\beta}{2(\beta + 1)}} + M^{-1/2}} + \sqrt{D_{\chi^2}(\rho_Z \Vert Q_Z)} &\text{ (information density)}
    \end{cases},
\end{align*}
which holds with probability at least $1 - \delta$. While generalization bounds for classification and structured prediction can have sharper dependences on the number of examples and number of classes for supervised learning (e.g.,~via the techniques of~\citet{cabannes2021fast} and references therein), the conversion from regression to classification is a remarkably general way to account for the residual dependence, prompt bias, and multiple stages of estimation that mark our problem.

\subsection{Prompting Strategies}\label{sec:a:complexity:sampling}
\begin{figure*}[t!]
    \centering
    \includegraphics[width=\linewidth]{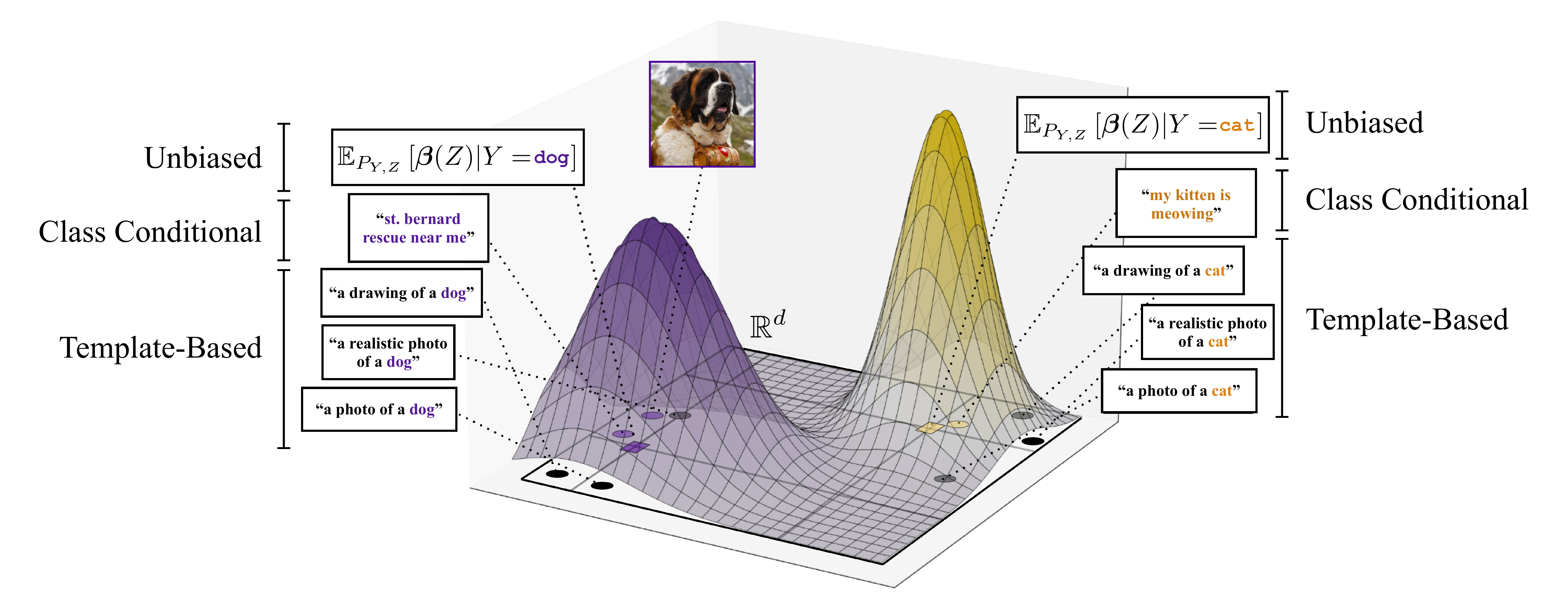}
    \caption{{\bf Illustration of Prompting Strategies.} A hypothetical distribution of embeddings $\bbeta(Z)$ parametrized by two classes (``cat'' and ``dog''). Three prompting strategies (template-based, class-conditional, and unbiased) are shown with example text and resulting embeddings in $\R^d$. Colors represent the probability of each class given the embedding.}
    \label{fig:prompt}
\end{figure*}

We have stated upper bounds on the statistical error in this section that depend on the size of the pre-training set $N$ and the number of prompts $M$. To state them more precisely, however, we must also specify the sampling schemes that lead to these examples/prompts. Sampling of the pre-training data falls into fixed and well-understood categories, boiling down to whether only paired examples or a combination of paired and unpaired examples are available. We describe these as part of the background (\Cref{sec:a:background:rkhs}), alongside the method to which they apply. However, the interpretation of prompting (the empirical procedure used in~\eqref{eq:proxy}) formally as a sampling scheme from a probability measure $\rho_{Y, Z}$ is itself a contribution of this paper. In the results of \Cref{sec:a:complexity:conditional_mean} and \Cref{sec:a:complexity:rn_derivative}, we considered simple random sampling $(Y_1, Z_1), \ldots, (Y_M, Z_M) \sim \rho_{Y, Z}$ i.i.d.~to provide examples of scenarios in which \Cref{asm:vvkrr:downstream} and \Cref{asm:complexity:rn_derivative:ps} can be satisfied. However, multiple practical and idealized strategies exist for prompting (such as the ones explored in \Cref{sec:experiments}). Below, we represent them in our framework below, as ways to define $\rho_{Y, Z}$ and approximate it with $\hrho_{Y, Z}$. 
\begin{itemize}
    \item {\bf Template-Based:} This technique reflects the earlier iterations of representing labels in natural language. Examples include ``photo of a \underline{\hspace{0.5cm}}'', ``realistic photo of a \underline{\hspace{0.5cm}}'', ``drawing of a \underline{\hspace{0.5cm}}'', etc. Notice that the prompt templates have no relationship with the class label. 
    One way this can be understood is by representing the caption via the structural equation $Z = f(Y, U)$, where $U$ represents the text of the caption with the label left blank (drawn according to a probability measure $\rho_U$), and $f$ represents the action of inserting the natural language label. Then, we have that under the template-based prompting distribution, $U \indep Y$. This does not imply that $Z \indep Y$, but instead that the dependence is governed fully by the function $f$. To sample, a fixed number of $m$ examples $\u_1, \ldots, \u_m$ are drawn directly from $\rho_U$. We then use the empirical measure $\hrho_{Y, Z}(\y, \z) = \frac{1}{m}\sum_{k=1}^m \rho_Y(\y) \ind\br{f(\y, \u_k) = \z}$, where $\rho_Y$ is fixed as the uniform distribution on the discrete set $\Y$. Here, $M = m\abs{\Y}$.
    \item {\bf Class Conditional:} This technique reflects the modern LLM-based techniques, such as CuPL \citep{pratt2023what}. We parameterize the joint distribution using the conditional distributions $\rho_{Y, Z} = \sum_{\y \in \Y} \rho_{Z|Y = \y}\cdot \rho_Y(\y)$ for each class $\y \in \Y$. Sampling from each $\rho_{Z|Y = \y}$ occurs by meta-prompting the LLM (such as the one we use in \Cref{sec:a:experiments}), which generates samples $\z^{\y}_1, \ldots, \z^{\y}_1$ and empirical measures $\hrho_{Z|Y = \y} = \frac{1}{m}\sum_{k=1}^m \delta_{\z_k^{\y}}$. Our final approximation is $\hrho_{Y, Z} = \sum_{\y \in \Y} \hrho_{Z|Y = \y}\cdot \rho_Y(\y)$, with $M = m \abs{\Y}$.
    \item {\bf Unbiased:} This techniques reflects the setting of \Cref{fig:ideal}, where the user may drawn samples from a joint distribution $P_{X, Y, Z}$, where the marginal $P_{X, Y}$ is in fact the data on which the zero-shot classifier will be evaluated. Then, the prompt distribution can be constructed, as we do, by drawing samples $(\y_1, \z_1), \ldots, (\y_M, \z_M)$ directly from $P_{Y, Z}$, and defining $\hrho_{Y, Z} = \frac{1}{M} \sum_{j=1}^M \delta_{(\y_j, \z_j))}$. We call this ``unbiased'', because the prompt bias term in \Cref{thm:res_dep1} is zero for this example. It is worth pointing out that even if $P_{Y|Z = \z}$ can be matched by the prompt distribution, the distribution mismatch term from \Cref{thm:complexity_rn_derivative} will be zero if and only if $\rho_Z = Q_Z$ (or the prompt captions match the pre-training captions in distribution). In this sense, $P_{Y, Z}$ may not be the ideal prompting distribution, but instead, $P_{Y|Z}Q_Z$.
\end{itemize}

\section{Self-Supervised Objectives and Cross Covariance Operators}\label{sec:a:objectives}

In \Cref{sec:theory}, we considered specific instances of both the conditional mean and information density approaches based on nonparametric regression in reproducing kernel Hilbert space (RKHS). This reflected the statistical goals of \Cref{thm:complexity_vvkr} and \Cref{thm:complexity_rn_derivative}. In this appendix, we aim to draw relationships with other approaches based on optimizing self-supervised learning (SSL) objectives, in order to align with practice. In particular, we focus on the relationship between such objectives and the mean square contingency $I(X;Z)$ introduced in \Cref{sec:framework}. To do so, we make explicit the intuition that SSL objectives (such as CLIP and VICReg) are implicit forms of dependence maximization between the representations $\balpha(X)$ and $\bbeta(Z)$. Some of the arguments below have previously appeared in the literature---we do not claim originality for them, but instead aim to consolidate them together in a single vignette.

When it comes to specific SSL objectives, we describe here their properties as functions acting on a batch of encoded data $(\balpha(\x_1), \bbeta(\z_1)), \ldots, (\balpha(\x_n), \beta(\z_n)$. This abstract description is agnostic to the function class used for the encoder. %
Reproducing kernel Hilbert space theory has been frequently used, in the recent literature, to define the function classes involved in contrastive and non-contrastive self-supervised foundation modeling~\cite{ li2021selfsupervised, balestriero2022constrastive, kiani2022jointembeddings, johnson2023contrastive, tan2024contrastive}. We also mention that the precise characterization of the function classes of various deep neural networks is an active area of research~\cite{schmidt-hieber2020rejoinder,scetbon2020harmonic,pahri2021banach,wu2022aspectral,bartolucci2023understanding, unser2023ridges, siegel2023characterization, schwarz-ziv2023aninformation,devore2025weighted}. However, these exciting yet still burgeoning theories of deep neural networks have not yet reached a maturity level comparable to the one of RKHS theory~\cite{wahba1990spline,cucker2007learning,christmann2008support,bach2024learning} needed for the theoretical analysis we develop in this paper. For more practical details on self-supervised learning, we point the reader to the recent survey~\cite{balestriero2023cookbook}.

\myparagraph{Covariance Operators}
To relate our theory (which centers around the mean square contingency measure of dependence) to SSL objectives, we first draw the relationship to covariance operators of $Q_{X, Z}$ on particular function spaces. Let $\calH$ be an RKHS of real-valued functions $\X$ and $\calG$ be an RKHS of real-valued functions on $\Z$. Then, define the cross-covariance operator $\C_{XZ}: \calG \rightarrow \calH$ by
\begin{align*}
    \ip{h, \C_{XZ} g}_{\calH} = \Cov_{Q_{X, Z}}(h(X), g(Z)),
\end{align*}
and the analogously defined auto-covariance operators $\C_{XX}: \calH \rightarrow \calH$ and $\C_{ZZ}: \calG \rightarrow \calG$. 
When $\C_{XX}$ and $\C_{ZZ}$ are compact, we define the powers $\C_{XX}^{1/2}$ and $\C_{ZZ}^{1/2}$ in the sense of~\eqref{eq:background:rkhs:power}. It then holds by \citet[Theorem 1]{baker1973joint} that there exists a unique bounded linear operator $\V_{XZ}: \calG \rightarrow \calH$, so that
\begin{align}
    \C_{XZ} = \C_{XX}^{1/2} \V_{XZ} \C_{ZZ}^{1/2}.\label{eq:baker}
\end{align}
The operator $\V_{XZ}$ is called the \emph{normalized cross-covariance operator}, or NOCCO for short \citep{fukumizu2005statistical}. As an abuse of notation, the NOCCO~\eqref{eq:baker} is sometimes communicated as $\V_{XZ} = \C_{XX}^{-1/2}\C_{XZ} \C_{ZZ}^{-1/2}$, though it is uniquely defined without necessarily constructing the square-root inverses. To rigorously use the formula $\C_{XX}^{-1/2}\C_{XZ} \C_{ZZ}^{-1/2}$ with a well-defined adjoint, we must make assumptions on the closure of the range of $\C_{XZ}$ and $\C_{ZX}$ being contained within the closure of the range of $\C_{XX}$ and $\C_{ZZ}$, respectively. The Hilbert-Schmidt norm of the population NOCCO, when finite, is equal to the mean square contingency
\begin{align}
    \norm{\V_{X, Z}}_{\HS(\calG, \calH)}^2 = I(X; Z), \label{eq:fukumizu_msc}
\end{align}
as shown in \citet[Theorem 4]{fukumizu2007kernel}. The relation~\eqref{eq:fukumizu_msc} requires a few additional technical conditions, such as $(\calH \otimes \calG) + \R$ being dense in $\ltwo(Q_X \otimes Q_Z)$ and $Q_{X, Z}$ having joint and marginal densities\footnote{Recall that our definition of $I(X; Z)$ does not include the square root that is usually used in the definition of mean square contingency.}.

\myparagraph{Variational Characterization of the Hilbert-Schmidt Norm}
This operator is an essential component of the kernel canonical correlations analysis (CCA) problem, which (with \eqref{eq:fukumizu_msc}) will be the common bridge that ties together SSL and the mean square contingency. From the nonparametric CCA perspective, the singular values $(\sigma_i)_{i=1}^\infty$ refer precisely to the canonical correlations and the singular functions $((\alpha_i, \beta_i))_{i=1}^\infty$ refer to the canonical variates \citep{lancaster1958thestructure, buja1990remarks, michaeli2016nonparametric}. Returning to~\eqref{eq:fukumizu_msc}, this operator is estimated with a regularization scheme, i.e.
\begin{align*}
    \hV_{X, Z} := (\hC_{XX} + \lambda \I)^{-1/2} \hC_{XZ} (\hC_{ZZ} + \lambda \I)^{-1/2},
\end{align*}
where $\hC_{XX}$, $\hC_{XZ}$, and $\hC_{ZZ}$ are the standard empirical covariance estimates (see \Cref{sec:a:background:rkhs}) and $\lambda > 0$ is a regularization parameter. Then, one solves the empirical CCA problem
\begin{align}
    \max_{\substack{h_1, \ldots, h_d \in \calH \text{ o.n.b}\\g_1, \ldots, g_d \in \calG \text{ o.n.b}}} \sum_{i=1}^d \ipsmall{h_i, \hV_{X, Z} g_i}_{\calH}.\label{eq:cca_empirical}
\end{align}
where o.n.b denotes an orthonormal basis. 
Setting aside matters of estimation, we consider how the norm quantity $\norm{\V_{X, Z}}_{\HS(\calG, \calH)}^2$ relates to the actual encoders returned by the CCA problem~\eqref{eq:cca_empirical} (assuming that $\hV_{X, Z} \approx \V_{X, Z}$).
Let $(s_i)_{i=1}^\infty$ be ordered singular values of the Hilbert-Schmidt operator $\V_{X, Z}$. In this case, denoting $h_1, \ldots, h_d \in \calH$ and $g_1, \ldots, g_d \in \calG$ the orthonormal bases of $\calH$ and $\calG$ resp. maximizing the criterion (\eqref{eq:cca_empirical}), we have
\begin{align}
    \sum_{i=1}^d \ipsmall{h_i, \V_{X, Z} g_i}_{\calH} &= \sum_{i=1}^d s_i \leq \sqrt{d \sum_{i=1}^d s_i^2}  \label{eq:cca_population}\\
    &= \sqrt{d\p{\norm{\V_{X, Z}}_{\HS(\calG, \calH)}^2 - \sum_{i=d+1}^\infty s_i^2}} \notag\\
    &= \sqrt{d \p{I(X; Z) - \sum_{i=d+1}^\infty s_i^2}}\notag\\
\end{align}
The two orthonormal bases maximizing the criterion~\eqref{eq:cca_empirical}
are actually the left and right singular functions of $\V_{XY}$ associated to the leading $d$ singular values (see \Cref{thm:svd}). The larger the truncation level $d$, the closer the quantity is to the mean-square contingency, up to the truncation level factor $d$. 

In either the population~\eqref{eq:cca_population} or empirical~\eqref{eq:cca_empirical} problems, the functions are maximizing an objective that is a measure of covariance with a constraint on variance. The constraint on variance is imposed by the norm condition on $h_1, \ldots, h_d \in \calH$ and $g_1, \ldots, g_d \in \calG$, respectively); see~\cite{fukumizu2007statistical}. This norm condition is relaxed into a penalization term in popular SSL objectives. 

Indeed, several SSL objectives can be written in an analogous \emph{variance-regularized covariance} form. This may offer one intuitive viewpoint as to why estimators based on these objectives might exhibit similar statistical properties to those analyzed in \Cref{sec:theory}. We first describe a format for these variance-regularized covariance objectives and show that a number of popular SSL objectives can be expressed in this form.

\myparagraph{Variance-Regularized Covariance Objectives}
Recall that $\balpha: \X \rightarrow \R^d$ and $\bbeta: \Z \rightarrow \R^d$ denote encoders for $\X$-valued and $\Z$-valued objects (often images and text, respectively). We denote the standard Euclidean inner product by $\ip{\u, \v} = \sum_{j=1}^d u_j v_j$ in $\R^d$. 
In either case, we consider a batch of data points $(\x_1, \z_1), \ldots, (\x_n, \z_n)$ which are thought to be $n$ independent and identically distributed realizations of $(X, Z)$ from the probability distribution $Q_{X, Z}$ over $\X \times \Z$. 
Let us then define the design matrices induced by the embeddings, written as
\begin{align*}
    \A := \begin{bmatrix}
    - \; \balpha(\x_1) \; -\\
    \vdots\\
    - \; \balpha(\x_n)\;  -
    \end{bmatrix} \in \R^{n \times d}
    \text{ and }
     \B := \begin{bmatrix}
    - \; \bbeta(\z_1)\; -\\
    \vdots\\
    - \; \bbeta(\z_n) \; -
    \end{bmatrix} \in \R^{n \times d}.
\end{align*}
Let $\J := \I - \frac{1}{n} \ones \ones^\top \in \R^{n \times n}$ be the centering matrix and construct the empirical auto-covariance and cross-covariance matrices 
\begin{align*}
    \hS_{AA} := (\J \A)^\top (\J \A), \quad \hS_{BB} := (\J \B)^\top (\J \B), \quad\text{and}\quad \hS_{AB} := (\J \A)^\top (\J \B).
\end{align*}
We aim to write the upcoming objectives in the form
\begin{align*}
    \mc{L}(\balpha, \bbeta) = -\Tr\p{\hS_{AB}} + \kappa \norm{\bS_{AB}}_{\mathrm{F}}^2 + V(\balpha, \bbeta),
\end{align*}
for hyperparameter $\kappa \geq 0$, matrix $\bS_{AB}$ (which is $\hS_{AB}$ with its diagonal components set to zero), and variance-regularization term $V(\balpha, \bbeta)$. The term $V(\balpha, \bbeta)$ may explicitly include the regularized inverses of $\hS_{XX}$ and $\hS_{ZZ}$, or may penalize variance or non-smoothness more implicitly.

\myparagraph{Example 1: Multimodal InfoNCE (CLIP)}
Consider the empirical objective for the contrastive language-image pre-training (CLIP) model \citep{radford2021learning} with batch size $n$,
\begin{align*}
    \hLclip(\balpha, \bbeta) := -\frac{1}{n} \sum_{i=1}^n \ip{\balpha(\x_i), \bbeta(\z_i)} + \frac{1}{2}\log \sum_{j=1}^n e^{\ip{\balpha(\x_i), \bbeta(\z_j)}} + \frac{1}{2}\log \sum_{j=1}^n e^{\ip{\balpha(\x_j), \bbeta(\z_i)}} + \log n,
\end{align*}
where the $\log n$ factor is appended to normalize the sums in the logarithmic terms and does not change the minimizer. Following arguments used (e.g.~by \citet{li2021selfsupervised}) for the SimCLR objective---the single-modality counterpart to CLIP---we analyze the logarithmic terms via Tayler expansion. To simplify the analysis, take the large-sample limit to define the population objective
\begin{align}
    \Lclip(\balpha, \bbeta) &:= -\Ex_{P}\ip{\balpha(X), \bbeta(Z)}\notag \\
    &+ \frac{1}{2}\Ex_{P_X}\sbr{\log \Ex_{P} \sbr{e^\ip{\balpha(X), \bbeta(Z)}\Big|X}} + \frac{1}{2}\Ex_{P_Z}\sbr{\log \Ex_{P} \sbr{e^\ip{\balpha(X), \bbeta(Z)}\Big|Z}}.\label{eq:clip:normalization}
\end{align}
Next, define the quantity $c(\x) := \E{P_{XZ}}{\ip{\balpha(X), \bbeta(Z)}|X}(\x)$ and apply a second-order Taylor expansion for every $\x \in \X$, the approximation
\begin{align*}
    e^\ip{\balpha(\x), \bbeta(Z)} &=  e^{c(\x)} e^{\ip{\balpha(\x), \bbeta(Z)} - c(\x)}\\
    &\approx e^{c(\x)} \p{1 + \ip{\balpha(\x), \bbeta(Z)} - c(\x) + \frac{1}{2}\p{\ip{\balpha(\x), \bbeta(Z)} - c(\x)}^2}.
\end{align*}
Plugging this approximation into the first term of~\eqref{eq:clip:normalization} yields
\begin{align*}
    \log \Ex_{P_Z} \sbr{e^\ip{\balpha(X), \bbeta(Z)}\Big | X}(\x) \approx c(\x) + \log \p{1 + \tfrac{1}{2}\Var\p{\ip{\balpha(\x), \bbeta(Z)}|X}(\x)}
\end{align*}
Using the Taylor expansion $\log(1 + y) = y + o(y)$ centered at $y = 0$, and evaluate the first-order approximation at $y = \tfrac{1}{2}\Var\p{\ip{\balpha(\x), \bbeta(Z)}|X}(\x)$, we finally have that
\begin{align*}
    \frac{1}{2}\E{P_X}{\log \Ex_{P_Z} \sbr{e^\ip{\balpha(X), \bbeta(Z)}\Big | X}} \approx \frac{1}{2} \E{P}{\ip{\balpha(X), \bbeta(Z)}} + \frac{1}{4} \E{P_X}{\Var\p{\ip{\balpha(X), \bbeta(Z)}|X}}.
\end{align*}
Applying an identical argument to the second term of~\eqref{eq:clip:normalization} gives
\begin{align*}
    \Lclip(\balpha, \bbeta) &\approx -\p{\Ex_{P}\ip{\balpha(X), \bbeta(Z)} - \ip{\E{P_X}{\balpha(X)}, \E{P_Z}{\bbeta(Z)}}} \\
    &+ \tfrac{1}{4} \E{P_X}{\Var\p{\ip{\balpha(X), \bbeta(Z)}|X}} + \tfrac{1}{4} \E{P_Z}{\Var\p{\ip{\balpha(X), \bbeta(Z)}|Z}}\\
    &= -\Tr\p{\Cov(\balpha(X), \bbeta(Z))} + \tfrac{1}{4} \E{P_X}{\Var\p{\ip{\balpha(X), \bbeta(Z)}|X}} + \tfrac{1}{4} \E{P_Z}{\Var\p{\ip{\balpha(X), \bbeta(Z)}|Z}}.
\end{align*}
which is the desired form for the population. Now, to rewrite the empirical version, we have
\begin{align*}
    \hLclip(\balpha, \bbeta) = \underbrace{-\Tr\p{\hS_{AB}}}_{\text{covariance}} + \underbrace{\tfrac{1}{4N} \sum_{i=1}^N \widehat{\Var}_N\p{\ip{\balpha(X), \bbeta(Z)}|X}(\x_i) + \tfrac{1}{4N} \sum_{i=1}^N \widehat{\Var}_N\p{\ip{\balpha(X), \bbeta(Z)}|Z}(\z_i)}_{\text{variance regularization}},
\end{align*}
where $\widehat{\Var}_N$ denotes the variance with respect to the empirical measure $\frac{1}{N} \sum_{i=1}^N \delta_{(\x_i, \z_i)}$.

\myparagraph{Example 2: BarlowTwins} The BarlowTwins objective \citep{zbontar2021barlow} has already been interpreted as an instance of kernel canonical correlations analysis (CCA) by previous work (e.g.~by \citet{balestriero2022constrastive}). This objective is usually defined in terms of the cross-correlation and auto-correlation matrices. To be consistent with other objectives in this section, we handle this by enforcing a constraint on the variance. Let $\iota_{S}: \R^{d \times d} \rightarrow \br{0, +\infty}$ denote the convex analytic indicator function such that $\iota(\bS) = 0$ if $\bS \in S$ and equals $+\infty$ otherwise. Let $\I$ be the identity matrix in $\R^{d \times d}$.
Given hyperparameter $\kappa > 0$, the objective can be written
\begin{align*}
    \hLBT(\balpha, \bbeta) &:= \frac{1}{2}\sum_{i=1}^d \p{(\hS_{AB})_{i, i} - 1}^2 + \kappa \norm{\bS_{AB}}_{\mathrm{F}}^2 + \iota_{\br{\I}}(\bS_{AA}) + \iota_{\br{\I}}(\bS_{BB})\\
    &= \underbrace{-\Tr(\hS_{AB}) + \kappa \norm{\bS_{AB}}_{\mathrm{F}}^2}_{\text{covariance}} + \underbrace{\sum_{i=1}^d (\hS_{AB})_{i, i}^2 + \frac{d}{2} +  \iota_{\br{\I}}(\bS_{AA}) + \iota_{\br{\I}}(\bS_{BB})}_{\text{variance regularization}}.
\end{align*}
Thus, this objective falls into the class as well, as the penalties enforce a particular variance structure akin to the regularizers above.

\myparagraph{Example 3: Spectral Contrastive Loss}
Finally, we consider the spectral contrastive loss from the pioneering work of \citet{haochen2021provable}. This relates to similar viewpoints of contrastive learning as spectral methods found in the literature, such as the Laplacian eigenmap viewpoint of VICReg \citep[Section 3]{balestriero2022constrastive}, the multidimensional scaling viewpoint of InfoNCE \citep[Section 4]{balestriero2022constrastive}, or the recent spectral clustering viewpoint of SimCLR/CLIP \citep[Sections 3 and 4]{tan2024contrastive}. Recall that $\bar{\balpha} := \frac{1}{n}\sum_{i=1}^n \balpha(\x_i)$ and $\bar{\bbeta} := \frac{1}{n}\sum_{i=1}^n \bbeta(\z_i)$. In the multimodal setting, this loss \citep[Eq.~(6)]{haochen2021provable} can be written as
\begin{align*}
    \hLSC(\balpha, \bbeta) &:= -\frac{1}{n} \sum_{i = 1}^n \ip{\balpha(\x_i), \bbeta(\z_i)} + \frac{1}{n(n-1)} \sum_{i \neq j} \p{\ip{\balpha(\x_i), \bbeta(\z_j)}}^2\\
    &= -\frac{1}{n} \sum_{i = 1}^n \ip{\balpha(\x_i) -\bar{\balpha}, \bbeta(\z_i) - \bar{\bbeta}} - \ip{\bar{\balpha}, \bar{\bbeta}} \\
    &\quad + \frac{1}{n(n-1)} \sum_{i \neq j} \p{\ip{\balpha(\x_i) - \bar{\balpha}, \bbeta(\z_j) - \bar{\bbeta}}}^2 + \p{\ip{\bar{\balpha}, \bar{\bbeta}}}^2\\
    &= -\Tr(\hS_{AB}) + \frac{1}{n-1}\norm{\bS_{AB}}_{\mathrm{F}}^2 - \ip{\bar{\balpha}, \bar{\bbeta}} + \p{\ip{\bar{\balpha}, \bar{\bbeta}}}^2\\
    &= \underbrace{-\Tr(\hS_{AB}) + \frac{1}{n-1}\norm{\bS_{AB}}_{\mathrm{F}}^2}_{\text{covariance}} + \underbrace{\p{\ip{\bar{\balpha}, \bar{\bbeta}} - \frac{1}{2}}^2 - \frac{1}{4}}_{\text{variance regularization}}
\end{align*}
where we set $\kappa := 1/(n-1)$ to complete the argument.

\myparagraph{Example 4: Multimodal VICReg}
We use a variant of the VICReg objective shown in \citet[Equation 1]{schwarz-ziv2023aninformation}. Note that this method is typically designed for one encoder being applied to two augmentations of the same object; however, it naturally generalizes to the multimodal case. The similarity graph simply connects paired observations, leading to the invariance term below. The multimodal VICReg objective has hyperparameters $(c_1, c_2, c_3, \kappa)$. To state it, define the real-valued function $r(x) := \max\br{0, c_1 - \sqrt{x + c_2}}$ for $x \in \R$. We will also apply $r$ to a matrix, which returns the matrix of element-wise applications of the function. The objective is written
\begin{align*}
    \hLvic(\balpha, \bbeta) &= \frac{c_3}{2d} \underbrace{\sbr{\Tr\p{r(\hS_{AA}}) + \Tr\p{r(\hS_{BB}})}}_{\text{variance}} +\frac{1}{2n}\underbrace{\norm{\A - \B}_{\mathrm{F}}^2}_{\text{invariance}} +\underbrace{\kappa\norm{\bS_{AB}}_{\mathrm{F}}^2}_{\text{covariance}}.
\end{align*}
While usually thought of as capturing a separate property, we will incorporate the invariance term into the other two terms, which crucially relies on having an extra degree of freedom via the second encoder (as opposed to the single-modality setting). Define $\bar{\balpha} := \frac{1}{n}\sum_{i=1}^n \balpha(\x_i)$ and $\bar{\bbeta} := \frac{1}{n}\sum_{i=1}^n \bbeta(\z_i)$, then write
\begin{align*}
    \frac{1}{2n}\norm{\A - \B}_{\mathrm{F}}^2 &= \frac{1}{n}\sum_{i=1}^n \norm{\balpha(\x_i) - \bbeta(\z_i)}_2^2\\
    &= \frac{1}{2n}\sum_{i=1}^n \norm{\balpha(\x_i)  - \bar{\balpha} - \bbeta(\z_i)  + \bar{\bbeta}}_2^2 + \frac{1}{2}\norm{\bar{\balpha} - \bar{\bbeta}}_2^2\\
    &= \frac{1}{2}\Tr\p{\hS_{AA}} + \frac{1}{2}\Tr\p{\hS_{BB}} - 2 \Tr\p{\hS_{AB}} + \frac{1}{2}\norm{\bar{\balpha} - \bar{\bbeta}}_2^2.
\end{align*}
The final term $\norm{\bar{\balpha} - \bar{\bbeta}}_2^2$ can harmlessly be dropped in the objective, as all other terms do not depend on the individual means. Thus, we can redefine our VICReg objective as
\begin{align*}
    &\hLvic(\balpha, \bbeta) = \underbrace{-\Tr\p{\hS_{AB}} + \kappa\norm{\bS_{AB}}_{\mathrm{F}}^2}_{\text{covariance}}\\
    &\quad + \underbrace{\frac{1}{2}\p{\Tr\p{\hS_{AA}} + \Tr\p{\hS_{BB}}} + \frac{c_3}{2d} \sbr{\Tr\p{r(\hS_{AA}}) + \Tr\p{r(\hS_{BB}})}}_{\text{variance regularization}},
\end{align*}
as intended.

\section{Experimental Details}\label{sec:a:experiments}
This appendix accompanies \Cref{sec:experiments} with further details of the study. Before describing the experiments, we comment one quantity appearing in the risk bounds that is not analyzed experimentally is the distribution shift error that passes $\ltwo(P_X)$-norm to the $\ltwo(Q_X)$-norm from \Cref{sec:theory}. For this, we refer the reader to the host of empirical work at the intersection of FSL, attribute-based and prompting-based ZSP, and distribution shift (see \citet{recht2019doimage, hendrycks2018benchmarking, goyal2023finetune} and references therein).

\subsection{Compute Environment}
Experiments were run on a CPU/GPU workstation with 12 virtual cores, 126G of memory, and four NVIDIA TITAN Xp GPUs with 12G memory each. The code was written in Python 3.10 with the environment given by the YAML file in the supplement. The \href{https://github.com/mlfoundations/open_clip}{OpenCLIP} and \href{https://github.com/LAION-AI/CLIP_benchmark}{CLIP Benchmark} repositories were either used directly or adapted in our codebase.

\subsection{Evaluation Datasets}
We use the following datasets as evaluation benchmarks for zero-shot image classification. Note that the following standard statistics describe their \emph{test} sets.
\begin{itemize}
    \item {\bf Describable Textures Dataset (DTD):} 1,880 examples labeled with 47 classes \citep{cimpoi14describing}.
    \item {\bf Flowers 102:} 6,149 examples labeled with 102 classes. \citep{nilsback2008automated}.
    \item {\bf FGVC Aircraft:} 3,333 examples labeled with 100 classes \citep{maji2013fine}.
    \item {\bf SUN397:} 21,750 examples labeled with 397 classes \citep{xiao2010sun}.
    \item {\bf ImageNet-1k:} 100,000 examples labeled with 998 classes. \citep{deng2009imagenet}.
\end{itemize}
The {\bf ImageNet-Captions} dataset \citep{fang2023data} is also used for evaluation using a subset of 134,593 examples, whereas a 40,000 held-out subset is used to estimate the conditional means of the text embeddings. For the subsets of ImageNet-Captions, the exact filenames of the ImageNet-1k subsets are provided along with their captions. Image preprocessing for evaluation was done using the transformations in the PyTorch \texttt{transforms} module that were associated with each OpenCLIP model.

For the experiment behind \Cref{fig:ideal}, we design three in-distribution sub-tasks by randomly selecting collections of 50 classes $(\Y_1, \Y_2, \Y_3)$ from each of 998 classes, reserving held-out prompting examples $(Z_1, Y_1), \ldots, (Z_{15,000}, Y_{15,000})$, 100 for each of 150 classes. Then, for task $i$, using $M$ examples $j_1(\y), \ldots, j_M(\y)$ selected randomly without replacement for $\y \in \Y_i$, we use the vector $\frac{1}{M} \sum_{m=1}^M \bbeta(Z_{j_m(\y)})$ as the class embedding (projected to unit norm). Using an evaluation set of approximately 25,000 examples from each sub-task, we compute the classification accuracy of this approach.

\subsection{Model Specification and Hyperparameters}

\paragraph{CLIP Architectures}
First, we specify which OpenCLIP models and pre-training sets were used. These models were chosen due to their range of top-1 zero-shot accuracies on the ImageNet-1k benchmark (as shown below). As opposed to already highly performant models ($\geq$50\% on ImageNet-1k), these models benefited more from optimized prompting techniques in our initial experiments. 
\begin{center}
\begin{tabular}{c|ccc}
     {\bf Model } & {\bf OpenCLIP Model Tag } &  {\bf Pre-Training Set Tag} &  {\bf ImageNet-1k Top-1 Acc.}\\
     \hline
     ResNet-50 & \texttt{RN50} & \texttt{yfcc15m} & 28.11\%\\
     NLLB-CLIP & \texttt{nllb-clip-base} & \texttt{v1} & 33.51\%\\
     ViT-B/32 & \texttt{ViT-B-32} & \texttt{datacomp\_m\_s128m\_b4k} & 32.81\%
\end{tabular}
\end{center}

\paragraph{Prompt-Generating Model}
We employed the \texttt{meta-llama/Llama-3.2-1B-Instruct} model publicly available on \href{https://huggingface.co/meta-llama/Llama-3.2-1B-Instruct}{HuggingFace}. For the purpose of generation, we used a {\bf top-$p$} hyperparameter of {\bf 0.9} and {\bf temperature} hyperparameter of {\bf 0.99} for more diverse responses. Meta-prompting was based on the following instructions per dataset, which are slight variations of those used in \citet{pratt2023what}:
\begin{itemize}
    \item {\bf Describable Textures Dataset (DTD):}
    \begin{itemize}
        \item {\it ``What does \underline{\hspace{0.5cm}} material look like?''},
        \item {\it ``What does a \underline{\hspace{0.5cm}} surface look like?''},
        \item {\it ``What does a \underline{\hspace{0.5cm}} texture look like?''},
        \item {\it ``What does a \underline{\hspace{0.5cm}} object look like?''},
        \item {\it ``What does a \underline{\hspace{0.5cm}} pattern look like?''}
    \end{itemize}
    \item {\bf Flowers 102:}
    \begin{itemize}
        \item {\it ``Describe how to identify a(n) \underline{\hspace{0.5cm}}, a type of flower.''},
        \item {\it ``What does a(n) \underline{\hspace{0.5cm}} flower looks like?''}
    \end{itemize}
    \item {\bf FGVC Aircraft:}
    \begin{itemize}
        \item {\it ``Describe a(n) \underline{\hspace{0.5cm}} aircraft.''},
        \item {\it ``Describe the \underline{\hspace{0.5cm}} aircraft.''}
    \end{itemize}
    \item {\bf SUN397:}
    \begin{itemize}
        \item  {\it ``Describe what a(n) \underline{\hspace{0.5cm}} looks like.''},
        \item {\it ``How can you identify a(n) \underline{\hspace{0.5cm}}?''},
        \item {\it ``Describe a photo of a(n) \underline{\hspace{0.5cm}}.''},
        \item {\it ``Describe the scene of a(n) \underline{\hspace{0.5cm}}.''}
    \end{itemize}
    \item {\bf ImageNet-1k:}
    \begin{itemize}
        \item {\it ``Describe what a(n) \underline{\hspace{0.5cm}} looks like.''},
        \item {\it ``How can you identify a(n) \underline{\hspace{0.5cm}}?''},
        \item {\it ``What does a(n) \underline{\hspace{0.5cm}} look like?''},
        \item {\it ``Describe an image from the Internet of a(n) \underline{\hspace{0.5cm}}.''},
        \item {\it ``Write a caption of an image of a(n) \underline{\hspace{0.5cm}}.''}
    \end{itemize}
\end{itemize}
The following additional instruction was appended for better-formatted responses: {\it ``Please format your response as one that contains only lower case letters and no special characters (including new lines, bold, and any markdown artifacts) other than a period (`.') or commas (`,'). The response should be a single sentence ending in a period that is directed toward the final instruction in this message. Your sentence should be a minimum of three words and a maximum of thirty.''}. 

Our reproducibility effort includes not only the full list of all 164,400 prompts generated from LlaMA 3, but the subset of prompts used for each class and each seed used to generate the figures in \Cref{sec:experiments}. 

\subsection{Derivation of Simulation Setting}
\label{sec:a:experiments:simulation}

The data-generating process for $(X, Z, Y)$ in the simulation from \Cref{sec:experiments} is as follows. Because we isolate the effect residual dependence in this simulation, we construct a joint distribution $P_{X, Y, Z}$ that satisfies \Cref{asm:rcd1}, and moreover, such that $Q_{X, Z} = P_{X, Z}$ and $\rho_{Y, Z} = P_{Y, Z}$.
Let $\Y = \br{0, 1}$, indicating binary classification. We consider $\X = \Z = \R^d$ and a pair of Gaussian distributions $(P_{X, Z|Y = 0}, P_{X, Z|Y = 1})$, where
\begin{align}
    \begin{bmatrix}
        X \\
        Z
    \end{bmatrix}
    \sim
    \mc{N}\p{
    \begin{bmatrix}
        \bmu_{X|\y} \\
        \bmu_{Z|\y}
    \end{bmatrix},
    \begin{bmatrix}
        \C_{XX|\y} & \C_{XZ|\y} \\
        \C_{ZX|\y} & \C_{ZZ|\y}
    \end{bmatrix}
    }
    \text{ given } Y = \y.\label{eq:mvn}
\end{align}
Then, the distribution is fully specified by mean vectors and covariance matrices along with the parameter $p = \P{}{Y = 1}$. Letting $\mc{N}(\cdot; \bmu,\C)$ indicate the density function of the $\mc{N}(\bmu, \C)$ distribution, the \emph{direct} predictor is equal to
\begin{align}
    p(\x) := \E{P_{X, Y}}{Y|X}(\x) = \frac{p \mc{N}(\x; \bmu_{X|1})}{p \mc{N}(\x; \bmu_{X|1}) + (1-p) \mc{N}(\x; \bmu_{X|0})}.\label{eq:re:bayes}
\end{align}
Similarly, the \emph{indirect predictor} is given by
\begin{align}
    \eta_\rho(\x) &= \E{P_{X,Z}}{\E{P_{Z, Y}}{Y|Z}|X}(\x) = \E{P_{X,Z}}{p(Z)|X}(\x), \label{eq:re:twostage}\\
    p(\z) &= \frac{p \mc{N}(Z; \bmu_{Z|1})}{p \mc{N}(Z; \bmu_{Z|1}) + (1-p) \mc{N}(Z; \bmu_{Z|0})}.
\end{align}
The expectation in~\eqref{eq:re:twostage} over $Z$ given $X = x$ can be evaluated via simulation based on the mixture model $P_{Z|X = \x} = (1-p(\x)) P_{Z|X = \x, Y = 0} + p(\x)P_{Z|X = \x, Y = 1}$ and the exact calculation
\begin{align*}
    Z &\sim \mc{N}\p{\bmu_{Z|\y} + \C_{ZX|\y} \C_{XX|\y}^{-1} (\x - \bmu_{X|\y}), \C_{ZZ|\y} - \C_{ZX|\y} \C_{XX|\y}^{-1} \C_{XZ|\y}} \text{ given } X = \x, Y = \y.
\end{align*}
Finally, the \emph{residual dependence} $\E{P_Z}{I(X;Y|Z)}$ can be computed by the following steps. 
First, notice that the conditional distribution of $X$ given $Z = \z$ and $Y = \y$ is given by
\begin{align*}
    X \sim \mc{N}\p{\bmu_{X|\y} + \C_{XZ|\y} \C_{ZZ|\y}^{-1} (\z - \bmu_{Z|\y}), \C_{XX|\y} - \C_{XZ|\y} \C_{ZZ|\y}^{-1} \C_{ZX|\y}}.
\end{align*}
The likelihood ratio $\Ssans_{\z}$ from~\eqref{eq:theory:rnd0} can be computed (where the evaluation at $\x$ refers to the density) via
\begin{align*}
    \Ssans_{\z}(\x, \y) &= \frac{P_{X|Y = \y, Z = \z}(\x)[\y p(\z) + (1-\y)(1 - p(\z))] }{P_{X|Z = \z}(\x)[\y p(\z) + (1-\y)(1 - p(\z))] } \\
    &= \frac{P_{X| Y = \y, Z = \z}(\x)}{P_{X|Z = \z}(\x)}\\
    &= \frac{P_{X| Y = \y, Z = \z}(\x)}{(1-p(\z)) P_{X|Y = 0, Z = \z}(\x) + p(\z)P_{X|Y = 1, Z = \z}(\x)}.
\end{align*}
To simulate from the marginal $P_Z$, we use the mixture $p P_{Z|Y=1} + (1-p)P_{Z|Y=0}$, after which~\eqref{eq:cond_msc} can be directly applied.

To interpolate between the setting in which $X \indep Z | Y$ (the indirect predictor performs at chance) and $X \indep Y | Z$ (the indirect predictor is equivalent to the direct one), we use the setting 
\begin{align*}
    \bmu_{X|0} = \tfrac{1}{2}\ones, \quad  \bmu_{X|1} = -\tfrac{1}{2}\ones.
\end{align*}
Let $a, b > 0$ be constants and let $\theta \in [0, 1]$ be a parameter. Then, we define
\begin{align*}
    \bmu_{Z|0} = 2\theta a \bmu_{X|0},& \quad  \bmu_{Z|1} = 2\theta b \bmu_{X|1}\\
    \C_{ZZ|0} = a\I, \quad \C_{ZX|0} = \tfrac{\theta a}{2} \I,& \quad \C_{ZZ|1} = b \I, \quad \C_{ZX|1} = \tfrac{\theta b}{2} \I
\end{align*}
and finally $\C_{XX|0} = (1 + \tfrac{a}{4})\I$ and $\C_{XX|1} = (1 + \tfrac{b}{4})\I$. Due to Gaussianity, it is clear that
\begin{align*}
    \theta = 0 &\implies \C_{ZX|0} = \C_{ZX|1} = \zeros \implies X \indep Z | Y =y \ \forall y.
\end{align*}
On the other hand, using the distribution of $X$ given $(Z, Y)$, that is,
\begin{align*}
    X \sim \mc{N}\p{\bmu_{X|y} + \C_{XZ|y} \C_{ZZ|y}^{-1} (\z - \bmu_{Z|y}), \C_{XX|y} - \C_{XZ|y} \C_{ZZ|y}^{-1} \C_{ZX|y}} \text{ given } Z = \z, Y = y,
\end{align*}
we have that
\begin{align}
    \theta = 1 &\implies \begin{cases}
    \bmu_{X|0} - \C_{XZ|0} \C_{ZZ|0}^{-1}\bmu_{Z|0} = \bmu_{X|1} - \C_{XZ|1} \C_{ZZ|1}^{-1} \bmu_{Z|1}\\
    \C_{XZ|0} \C_{ZZ|0}^{-1} = \C_{XZ|1} \C_{ZZ|1}^{-1}\\
    \C_{XX|0} - \C_{XZ|0} \C_{ZZ|0}^{-1} \C_{ZX|0} = \C_{XX|1} - \C_{XZ|1} \C_{ZZ|1}^{-1} \C_{ZX|1}\\
    \end{cases} \implies X \indep Y | Z = \z \ \forall \z, \label{eq:re:resid}
\end{align}
as the distribution of $X$ remains the same given either $Z = \z, Y = 0$ or $Z = \z, Y = 1$. Thus, in the simulation, we interpolate between $0$ and $1$ for the value of $\theta$. We set the parameters $a = 5$ and $b = 6$ simply to be different numbers for which $\E{P_Z}{I(X;Y|Z)}$ can be computed in a numerically stable manner. We set $p = \frac{1}{2}$ and $d = 2$. 

Finally, the lines labeled \emph{CLIP} and \emph{VICReg} in \Cref{fig:simulation} indicate the predictors generated by training two MLP encoders using the corresponding objective on observations $\br{(X_i, Z_i)}_{i=1}^N$ for $N = 10,000$ pre-training observations. The encoder had a single hidden layer of $16$ units and an output dimension of $d$. When performing classification, the prompting distribution used for the methods based on self-supervised learning is the true distribution of $Z|Y=\y$ with $M= 500$ samples, allowing us to isolate residual dependence while incurring no prompt bias and negligible prompt variance. Each model was trained for 30 epochs with the AdamW optimizer at a learning rate of $0.01$. In the case of the VICReg objective, we used the parameterization of the original paper \citep{bardes2022vicreg} with the settings $(\gamma, \lambda, \mu, \nu, \epsilon) = (1, 25, 25, 1, 0.0001)$ as per the authors' recommendations (see their Eq.~(6)).

\end{document}